\documentclass{article} 
\usepackage{graphicx} 
\usepackage{amsmath,amssymb,amsfonts,mathtools}
\usepackage{amsthm}
\usepackage{mathtools}
\usepackage{bbm}
\usepackage{enumitem}
\usepackage[margin=1in]{geometry}
\usepackage{xcolor}
\usepackage{hyperref}
\usepackage{natbib}
\usepackage{bbm}
\usepackage{url}
\usepackage{caption}
\usepackage{subcaption}
\usepackage{wrapfig}
\usepackage{booktabs}
\theoremstyle{definition}
\newtheorem{definition}{Definition}
\theoremstyle{plain}
\newtheorem{theorem}{Theorem}
\newtheorem{corollary}{Corollary}
\newtheorem{lemma}{Lemma}
\newtheorem{assumption}{Assumption}
\newtheorem{proposition}{Proposition}
\newtheorem{remark}{Remark}
\usepackage{algorithm}
\usepackage{algpseudocode}
\usepackage{nicematrix}
\usepackage{tikz}
\usepackage{natbib}
\usepackage{authblk}

\allowdisplaybreaks
\title{Adaptive Bregman Alternating Projections for \\ Feasible Gromov--Wasserstein Learning}

\author{Aoran Zhang}
\author{César A. Uribe}
\affil{Department of Electrical and Computer Engineering, Rice University}
\date{}

\begin{document}
\maketitle

\begin{abstract}
The Gromov-Wasserstein (GW) problem compares structured distributions without requiring a shared feature space or known correspondences, but its nonconvex objective and coupled marginal constraints make computation challenging. Bregman alternating projected gradient (BAPG) uses inexpensive alternating row and column updates, yet its fixed-penalty relaxation leaves a persistent feasibility gap. 
We propose Adaptive KL-BAPG (A-KL-BAPG), which combines a finite fixed-penalty burn-in with a guarded increasing-penalty phase. 
At each tail iteration, the method reuses BAPG’s alternating updates and backtracks a delayed-power step until a Sinkhorn-inspired projective-diameter safeguard is satisfied. We prove finite termination of the backtracking at each iteration and show that the feasibility gap vanishes asymptotically. 
We further establish a best-iterate $O(1/\log N)$ bound for the weighted squared corrected residual and, under a support regularity condition, the existence of a stationary accumulation point for the original GW problem.
This distinguishes A-KL-BAPG from fixed-penalty BAPG, whose stationarity guarantees are given for the relaxed problem. 
Experiments show that A-KL-BAPG achieves a favorable balance of accuracy, objective value, feasibility, and stationarity relative to BAPG variants, projection-based methods, and task-specific baselines. 
For synthetic and real graph alignment problems, it closely matches the accuracy and objective value of fixed-penalty KL-BAPG while reducing the marginal feasibility gap by $62$-$99\%$ and the projected stationarity residual by $28$-$98\%$. 
Heterogeneous domain adaptation experiments show a similar pattern: A-KL-BAPG maintains comparable target accuracy and objective values while achieving improved feasibility and stationarity than fixed-penalty KL-BAPG.
\end{abstract}

\section{Introduction}
The Gromov-Wasserstein (GW) problem compares structured distributions based on their internal geometry, without requiring a shared feature space or a known correspondence, making it well-suited to comparing heterogeneous structured data \citep{memoli2011gromov,peyre2016gromov}. Its development grew from geometric point-cloud \citep{memoli2004comparing} and shape comparison \citep{memoli2009spectral,memoli2011gromov} to practical discrete formulations \citep{peyre2016gromov}. Its extension from metric space to weighted networks further supports graph comparison beyond a shared node set \citep{chowdhury2019gromov}. More recently, GW has been widely used for graph alignment and representation learning \citep{xu2019gromov, li2023convergent} and for transfer across heterogeneous domains \citep{yan2018semi, demetci2024breaking, aritake2026domain}. The resulting optimization, however, is difficult since the squared-loss GW is a nonconvex, quadratic program over the transport polytope.

GW solvers range from Frank-Wolfe methods (CG) \citep{peyre2016gromov, titouan2019optimal}, entropic Bregman projected gradients (EPGD) \citep{solomon2016entropic}, and Bregamn proximal-point methods (PPA) \citep{xu2019gromov}, to recent entropic-GW (EGW) \citep{rioux2024entropic} and inexact projected gradient (iPG) \citep{liang2026provably}. 
Graph-specific alternatives include the recursive partitioning method ScalaGW \citep{xu2019scalable} and the heat-kernel spectral method SpecGW \citep{chowdhury2021generalized}. These methods serve different purposes, but generally they repeatedly perform a linear transport step or iterative Bregman or Euclidean projections to enforce both marginals. The recent iPG makes inexact projection verifiable, but still requires inner solvers at every iteration. Such routines can dominate the cost of large-scale computations \citep{li2023convergent, liang2026provably}.

Bregman alternating projected gradient (BAPG) offers a faster single-loop alternative. Starting from the one-step \mbox{BPG-S} heuristic \citep{xu2019gromov}, BAPG splits the row and column constraints and penalizes disagreement between two transport variables \citep{li2023convergent}. With KL geometry, each half-step is simply a multiplicative GW-gradient update followed by normalization. These matrix-scaling updates are GPU-friendly \citep{li2023convergent} and support parallel, differentiable FGW learning \citep{deng2026glncd} with a popular BAPG implementation in POT \citep{flamary2021pot}.

The speed of fixed-penalty BAPG comes with a cost: For a penalty $\rho$, the two split variables may disagree, resulting in marginal errors of order $O(1/\rho)$ that persist throughout iterations \citep{li2023convergent}. Increasing $\rho$ appears natural; however, if it increases too fast, the associated step $\alpha = 1/\rho$ shrinks too quickly and can lead to finite cumulative progress. If the step size remains too aggressive, the positive scaling maps can become ill-conditioned in a projective sense. 
Additionally, the existing KL-BAPG also assumes bounded accumulated asymmetric Bregman errors (AAE), which is difficult to verify, and only post-facto empirical evidence of it has been shown. Moreover, singular boundary geometry can yield a stationary point only on a limiting support face \citep{chen2024spurious}. Therefore, we need to jointly control the penalty growth, cumulative step progress, and projective geometry.

This difficulty motivates a connection between BAPG and Sinkhorn scaling. Sinkhorn alternates exact row and column KL projections \citep{sinkhorn1967concerning,cuturi2013sinkhorn,benamou2015iterative}, and its contraction is governed by the projective diameter \citep{franklin1989scaling,eveson1995elementary,idel2016review}. Modern variants retain this scaling structure under changing mirror steps~\citep{ballu2023mirror} or temperatures \citep{chizat2024annealed}. General BAPG also shares the alternating projection structure. Under KL geometry, these projections coincide with the row and column normalizations used by Sinkhorn. Unlike standard Sinkhorn, however, each normalization in KL-BAPG is preceded by a state-dependent multiplicative GW-gradient update. As the penalty increases, these updates become smaller, motivating a schedule that controls the projective geometry of the resulting iterates.

In this paper, we introduce \textbf{Adaptive KL-BAPG}(A-KL-BAPG). The method runs a finite, fixed-penalty burn-in phase and then switches to a guarded, increasing-penalty phase. At each tail iteration, it proposes a delayed-power step, performs the same alternating row- and column-marginal normalization update as KL-BAPG~\citep{li2023convergent}, and backtracks until the candidate coupling satisfies a logarithmically expanding projective-diameter bound. We prove that the backtracking terminates finitely at every iteration without requiring an auxiliary projection subproblem. Under this guarded schedule, iterates are asymptotically feasible, and the discrepancy between the alternating iterates vanishes.

Our main contributions are:
\begin{itemize}
    \item \textbf{Vanishing marginal feasibility with explicit BAPG updates.}     Existing BAPG methods leave a persistent $O(1/\rho)$ feasibility gap at fixed penalty, assume bounded accumulated asymmetric Bregman errors, and can certify stationarity only on a limiting support face. A-KL-BAPG removes the $O(1/\rho)$ feasibility floor while retaining closed-form updates through an increasing-penalty schedule without an AAE assumption.
    \item \textbf{Stationarity and a rate for the original GW constraints.} We prove finite termination of the guard, control the accumulated KL asymmetry, and derive a best-iterate $O(1/\log N)$ bound for the weighted squared corrected residual. Under an inactive-support nonnegative-residual assumption, we show that the generated iterates share a stationary accumulation point for the original GW problem, i.e., some subsequence converges to a stationary point.
    \item \textbf{Empirical validation on GW and FGW problems.} We evaluate the proposed method on synthetic and real graph alignment and heterogeneous domain adaptation problems. A-KL-BAPG retains matching accuracy and GW objective values comparable to other BAPG methods while significantly reducing the marginal gap and projected stationarity residual. Across the included methods, A-KL-BAPG preserves strong alignment and classification accuracy while improving feasibility and stationarity.
\end{itemize}

\section{Problem Formulation, Algorithm Design and Main Results}

Let $D_X\in\mathbb R^{n\times n}$ and $D_Y\in\mathbb R^{m\times m}$ encode within-domain geometry, and let $\mu\in\mathbb R^n_+$ and $\nu\in\mathbb R^m_+$ be the desired marginal probability vectors with strictly positive entries.  Define
$\mathcal C_1 = \{\pi \geq \mathbf 0:\pi \mathbf{1}_m = \mu\}$, 
$\mathcal C_2 = \{\pi \geq \mathbf 0:\pi^\top \mathbf{1}_n = \nu\}$, and
$\mathcal C =\mathcal C_1\cap\mathcal C_2$. Let $g_r=\mathbb I_{\mathcal C_r}$ be the indicator of the corresponding one-marginal set. Up to terms constant on $\mathcal C$, the squared-loss GW problem is 

\begin{align}\label{main_problem}
    \min_{\pi \in \mathbb R^{n\times m} } F(\pi)+ g_1(\pi)+g_2(\pi) = \min_{\pi \in \mathcal C} F(\pi), \qquad
F(\pi)=-\operatorname{Tr}(D_X\pi D_Y\pi^\top).
\end{align}

For clarity, we take $D_X,D_Y$ symmetric with $G(\pi)=D_X\pi D_Y$, $\nabla F(\pi)=-2G(\pi)$. Then a fused linear term can be included without changing the alternating construction. The stationary set of the original constrained problem is $\mathcal X := \{x \in \mathcal C: \mathbf 0 \in \nabla F(x) + \mathcal N_{\mathcal C}(x)\}$.
To decouple the two marginal constraints, we introduce $w$ to reformulate:

\[
\min_{\pi = w} f(\pi, w)+ g_1(\pi)+g_2(w) + \rho D_h(\pi, w),
\]

with $f(\pi, w) = -\mathrm{Tr}(D_X \pi D_Y w^\top)$ and for the Legendre function $h$, $D_h (x, y) := h(x)-h(y) - \langle \nabla h(y), x-y\rangle$ is the Bregman divergence. Here we use the Boltzmann-Shannon entropy kernel $h(x) := \sum_{i,j} x_{ij}(\log x_{ij} - 1)$. Then $\nabla h(x) = \log x$ and \[D_h(u, v) = \sum_{ij} u_{ij}\log \frac{u_{ij}}{v_{ij}} - u_{ij} + v_{ij}.\]

Our starting point is to introduce an increasing distance penalty $\{\rho_k\}$ with step size $\alpha_k = {1}/{\rho_k}$. Then, for $k$-th iteration, the BAPG update aims to solve:
\begin{align*}
    \pi^{k+1} & = \arg\min f(\pi, w^k)+\rho_k D_h(\pi, w^k) + g_1(\pi), \\
    w^{k+1}  &= \arg\min f(\pi^{k+1},w)+\rho_k D_h(w, \pi^{k+1}) + g_2(w), 
\end{align*}
where $\nabla_{\pi} f(\pi, w) = -G(w)$, and $\nabla_{w} f(\pi, w) = -G(\pi)$.
Since elements of $\mathcal C_1$ and $\mathcal C_2$ are nonnegative and have fixed total mass one, $\mathcal C_1$ and $\mathcal C_2$ are compact. Since $G(X)=D_X X D_Y$ is continuous for fixed finite matrices $D_X$ and $D_Y$, it is uniformly bounded on the compact set $\mathcal C_1\cup\mathcal C_2$. In particular, we use the explicit bound \[M_G:=|D_X|_{\max}|D_Y|_{\max} \geq \sup_{X\in\mathcal C_1\cup\mathcal C_2}|G(X)|\infty,\] which follows from the nonnegativity and unit total mass of $X$.

\subsection{Algorithm Design}

Let $\rho_0>0$ be a fixed penalty used during a finite prefix,  and $k_0<\infty$ be the length of the prefix. For $k \leq k_0$, the step size is $\alpha_k = 1/\rho_0$ and the two KL subproblems have the closed forms
\begin{subequations}\label{eq:bapg-cycle}
\begin{align}
z_1^t &=w^t\odot\exp\!\left(\alpha_tG(w^t)\right), &
\pi^{t+1}_{ij}&=\mu_i\frac{z^t_{1,ij}}{\sum_\ell z^t_{1,i\ell}},
\label{eq:row}\\
z_2^t &=\pi^{t+1}\odot \exp\!\left(\alpha_tG(\pi^{t+1})\right), &
w^{t+1}_{ij}&=\nu_j\frac{z^t_{2,ij}}{\sum_\ell z^t_{2,\ell j}}.
\label{eq:col}
\end{align}
\end{subequations}
By \cite{li2023convergent}, this often yields useful early progress on the objective, but its two split variables have a non-vanishing infeasibility gap. Starting immediately with diminishing steps removes that floor asymptotically, but can suppress early progress. We therefore run $k_0<\infty$ iterations with fixed penalty $\rho_0$ and then apply an adaptive diminishing-step tail. After the finite prefix, we reset the tail index to $t = 0$ and relabel the start of the post-burn state as $(\pi^0, w^0)$. 

\begin{algorithm}[t]
\caption{Adaptive KL-BAPG}
\label{alg:adaptive-delayed-power}
\small
\begin{algorithmic}[1]
\Require $w^0>0$, marginals $(\mu,\nu)$, burn-in $k_0$,
prefix penalty $\rho_0$, proposal scale $\alpha_0$, exponent $\beta$,
guard slope $a$, growth bound $B_\Delta$, backtracking factor $\tau$

\For{$k=0,\ldots,k_0-1$} \Comment{fixed-penalty burn-in}
    \State $\alpha\gets 1/\rho_0$
    \State $z_1\gets w^k\odot\exp\bigl(\alpha G(w^k)\bigr)$
    \State $\pi^{k+1}_{ij}\gets
        \mu_i z_{1,ij}/\sum_\ell z_{1,i\ell}$
    \State $z_2\gets \pi^{k+1}\odot
        \exp\bigl(\alpha G(\pi^{k+1})\bigr)$
    \State $w^{k+1}_{ij}\gets
        \nu_j z_{2,ij}/\sum_\ell z_{2,\ell j}$
\EndFor

\State $\Delta_0\gets\Delta(w^{k_0})$;
       $\alpha_{-1}\gets\infty$

\For{$t=0,1,2,\ldots$} \Comment{guarded increasing-penalty tail}
    \State $k\gets k_0+t$
    \State $\bar\alpha_t\gets
        \alpha_0\bigl((k_0+1)/(t+k_0+1)\bigr)^\beta$
    \State $h_t\gets
        (a/B_\Delta)\log\bigl((t+2)/(t+1)\bigr)$
    \State $\ell_t\gets\tau\min\{\bar\alpha_t,h_t\}$
    \State $\alpha\gets
        \max\{\ell_t,\min(\bar\alpha_t,\alpha_{t-1})\}$

    \Repeat
        \State $\widetilde z_1\gets
            w^k\odot\exp\bigl(\alpha G(w^k)\bigr)$
        \State $\widetilde\pi_{ij}\gets
            \mu_i\widetilde z_{1,ij}/
            \sum_\ell\widetilde z_{1,i\ell}$
        \State $\widetilde z_2\gets
            \widetilde\pi\odot
            \exp\bigl(\alpha G(\widetilde\pi)\bigr)$
        \State $\widetilde w_{ij}\gets
            \nu_j\widetilde z_{2,ij}/
            \sum_\ell\widetilde z_{2,\ell j}$

        \If{$\Delta(\widetilde w)
            \leq\Delta_0+a\log(t+2)$}
            \State $\alpha_t\gets\alpha$;
                   $\pi^{k+1}\gets\widetilde\pi$;
                   $w^{k+1}\gets\widetilde w$
        \Else
            \State $\alpha\gets\tau\alpha$
        \EndIf
    \Until{$\Delta(\widetilde w)
        \leq\Delta_0+a\log(t+2)$}
\EndFor
\end{algorithmic}
\end{algorithm}
To control the geometry of the alternating iterates, we introduce Hilbert's projective metric 
\[ d_H(u,v)=\log\frac{\max_i({u_i}/{v_i})}{\min_i({u_i}/{v_i})},\] and the associated projective diameter \[\Delta(X)=\max_{i,i',j,j'}\log \frac{X_{ij}X_{i'j'}}{X_{ij'}X_{i'j}}\]
where $u, v$ are positive vectors and $X$ is a positive matrix. Then for the tail iterates, we set $\Delta_0=\Delta(w^0)$. For tail index $t\geq0$, choose a guard slope $a\in(0,{2}/{3})$, delayed-power exponent $\beta\in(\frac{2}{3} + \frac{a}{2},1)$, and proposal scale $\alpha_0>0$. 
The permitted diameter and nominal step are:
\begin{align}
\hspace{-0.2cm} \Gamma_t&{=}\Delta_0+a\log(t+1), \qquad
\bar\alpha_t{=}\alpha_0\left(\frac{k_0+1}{t+k_0+1}\right)^\beta.
\label{eq:guard-proposal}
\end{align}
The logarithmic threshold grows slowly enough to retain polynomial projective contraction while $\bar\alpha_t$ avoids shrinking the proposed step abruptly at the switch. 
Let $B_\Delta:=8M_G$ denote the conservative projective diameter growth of one cycle, which we established in Lemma~\ref{lem:projective-growth-cycle} (Appendix~\ref{app:proof_backtrack}). The corresponding safe step size of the diameter increase is
\begin{align}
h_t:=\frac{\Gamma_{t+1}-\Gamma_t}{B_\Delta}
=\frac{a}{B_\Delta}\log\frac{t+2}{t+1}.
\label{eq:safe-scale}
\end{align}
To avoid the accepted step size from becoming arbitrarily small, we define the floor $\ell_t := \tau \min \{\bar \alpha_t, h_t\}$,
where $\tau \in (0,1)$ is the backtracking factor. At iteration $t$, the line search starts at
$\max\{\ell_t,\min\{\bar\alpha_t,\alpha_{t-1}\}\}$,
with $\alpha_{-1}=+\infty$.  The inner minimum guarantees the monotonicity of the accepted tail steps by preventing repeated tests above the previously accepted value, while the outer maximum preserves the lower safeguard. The method repeatedly shrinks the nominal step-size by $\tau$ and recomputes~\eqref{eq:row}--\eqref{eq:col} until $\Delta(w^{t+1})\leq\Gamma_{t+1}$.  The accepted value is then denoted by $\alpha_t$. 
The detailed algorithm is outlined in Algorithm~\ref{alg:adaptive-delayed-power}.

\begin{remark} \label{rem:finit-back}
    After Lemma \ref{lem:guarded-backtracking}, we show that whenever the current iterate satisfies $\Delta(w^t)\leq \Gamma_t$, every $\alpha\leq h_t$ is guaranteed to satisfy the next diameter bound $\Delta(w^{t+1})\leq \Gamma_{t+1}$. Thus, the backtracking line search terminates after finitely many steps at each iteration, and the floor $\ell_t=\tau\min\{\bar\alpha_t,h_t\}$ provides a controlled lower scale for the accepted step. Detailed proof is provided in Appendix~\ref{app:proof_backtrack}.
\end{remark}

A-KL-BAPG can be viewed as a variant of Sinkhorn scaling with guarded gradient perturbations. Each iteration alternates between row- and column-marginal normalizations as in Sinkhorn, while multiplicative gradient perturbations optimize the objective, and adaptive step sizes balance the optimization progress. The diameter guard controls the changing projective geometry without introducing additional projection subproblems. The formal connection is developed in Appendix~\ref{app:proof_sinkhorn}.

\subsection{Main Results}

We are now ready to introduce our main results on the subsequence convergence and stationarity of Algorithm~\ref{alg:adaptive-delayed-power}.  Our result is based on the idea of measuring performance and convergence by diminishing residual, specifically, we define the corrected residual as \[R_t:=\rho_t\bigl(\nabla h(w^{t})-\nabla h(w^{t+1})\bigr)+G(\pi^{t+1})-G(w^{t}).\]
The value of $R_t$ measures the imbalance between the GW gradient and the marginal constraint terms in the KL updates, with a correction
for the difference between $G(w^t)$ and $G(\pi^{t+1})$.
Specifically, the first result of Theorem~\ref{thm:unconditional} shows that its weighted norm vanishes along a subsequence of the generated iterates of the proposed algorithm.

\begin{theorem}[Feasibility and residual guarantees]
\label{thm:unconditional}
Consider the GW Problem~(\ref{main_problem}) with symmetric $D_X\in\mathbb{R}^{n\times n}$,
$D_Y\in\mathbb{R}^{m\times m}$, and marginals $\mu\in\mathbb{R}^{n}_{++}$,
$\nu\in\mathbb{R}^{m}_{++}$ satisfying $\mathbf{1}_n^{\top}\mu=\mathbf{1}_m^{\top}\nu=1$,
and let $M_G:=\sup_{X\in\mathcal{C}_1\cup\mathcal{C}_2}\|G(X)\|_\infty<\infty$.
Let the initialization be entrywise positive, let the burn-in length $k_0\in\mathbb{N}$ be finite, and let the prefix penalty $\rho_{0}>0$, the proposal scale $\alpha_0>0$, the backtracking factor $\tau\in(0,1)$, the growth bound $B_\Delta\ge 8M_G$, and the guard slope and delayed-power exponent $ a\in (0,{2}/{3})$, and $\beta\in ({2}/{3}+{a}/{2},\,1)$. Then, the tail iterates $\{(\pi^{t},w^{t})\}_{t\ge 0}$ generated by Algorithm~\ref{alg:adaptive-delayed-power}, indexed from the switch at $k=k_0$, have the following properties:
\begin{enumerate}
  \item $\displaystyle\liminf_{t\to\infty}\bigl\|R_t\bigr\|_{M_t}=0$,
        with $\rho_t:=1/\alpha_t$ and $\|U\|_{M_t}^{2}:=\sum_{i,j}m^{t}_{ij}U_{ij}^{2}$, where
        $m^{t}_{ij}:=\min\{w^{t}_{ij},w^{t+1}_{ij}\}$;
  \item $\displaystyle\lim_{t\to\infty}\bigl\|x^{t+1}-\pi^{t+1}\bigr\|_F+\bigl\|\pi^{t+1}-w^{t+1}\bigr\|_F = 0$, where $x^{t}:=\mathcal{P}_{\mathcal{C}}(\pi^{t})$ denotes the
        Euclidean projection of $\pi^{t}$ onto $\mathcal{C}=\mathcal{C}_1\cap\mathcal{C}_2$;
        in particular, $\{x^{t}\}$, $\{\pi^{t}\}$ and $\{w^{t}\}$ have the same set of
        accumulation points.
\end{enumerate}
\end{theorem}

Theorem~\ref{thm:unconditional} shows that A-KL-BAPG recovers feasibility for the original GW problem and admits a subsequence along which the weighted corrected residual vanishes unconditionally. Later in Theorem~\ref{thm:stationarity}, with the addition of Assumption \ref{assum:support}, we show that the limit of any such convergent subsequence is stationary. Unlike KL-BAPG~\citep{li2023convergent}, which assumes bounded accumulated asymmetric Bregman error (AAE) throughout, A-KL-BAPG guarantees absolute summability of the asymmetric errors over the entire iterate sequence, as shown in Lemma~\ref{lem:aae}.
\begin{lemma}
\label{lem:aae}
Let ${(\pi^t,w^t)}_{t\geq0}$ be the iterates generated by Algorithm~\ref{alg:adaptive-delayed-power}. Under the setting of Theorem \ref{thm:unconditional}, 
$\sum_{t=0}^{\infty}\bigl|D_h(\pi^{t+1},w^t)-D_h(w^t,\pi^{t+1})\bigr|
+
\bigl|D_h(w^{t+1},\pi^{t+1})
      -D_h(\pi^{t+1},w^{t+1})\bigr|
=O(\eta_t^3)<\infty$, i.e., AAE holds.
\end{lemma}
 While Theorem \ref{thm:unconditional} shows vanishing residuals and asymptotic convergence, to show the existence of a stationary accumulation point, we require the following assumption.
\begin{assumption}\label{assum:support}
For every $\{t_\ell\}$ such that $\|R^{t_\ell}\|_{M_{t_\ell}}\to0$ and $x^{t_\ell+1}=\mathcal P_{\mathcal C}(\pi^{t_\ell+1})\to\bar x$, define the active support of $\bar x$ as $S(\bar x):=\{(i,j):\bar x_{ij}>0\}$. Then for every $(i,j)\notin S(\bar x)$,
 $\liminf_{\ell\to\infty}R^{t_\ell}_{ij}\ge 0$.
\end{assumption}

\begin{wrapfigure}[13]{H}{0.4\textwidth}
    \centering
    \vspace{-5pt}
    \includegraphics[width=0.93\linewidth]{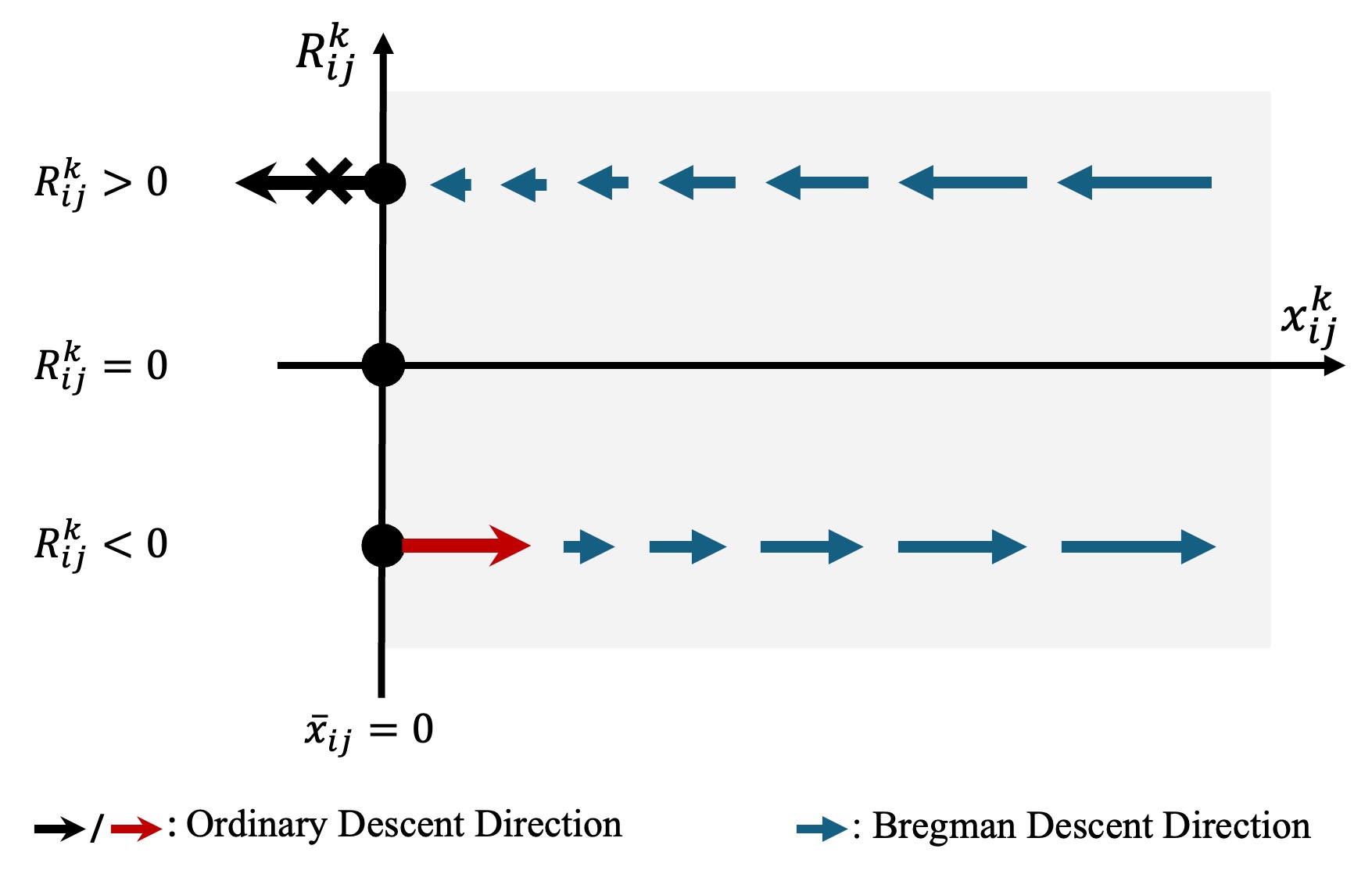}
    \vspace{-10pt}
    \caption{\small Boundary effect of the Descent. At an inactive entry, an outward descent direction is blocked (black), while an inward direction might remain feasible (red).}
    \label{fig:assum_1}
\end{wrapfigure}
%\begin{remark}
Assumption~\ref{assum:support} is needed only to pass from the weighted corrected residual to stationarity at a boundary point.  
As illustrated in Figure~\ref{fig:assum_1}, on active entries, the weights in $M_t$ stay positive and force $R^t_{ij}\to0$. On an inactive entry, the weight may vanish: $R^t_{ij}>0$ corresponds to an outward descent direction blocked by nonnegativity, while $R^t_{ij}<0$ reveals an inward feasible descent direction (similar situations discussed in \citet{chen2024spurious}).  
Related boundary regularity conditions are standard in constrained optimization, including nondegeneracy and strict complementarity in active-set analyses \citep{burke1988identification,liang2017activity}, and transversality conditions used to control limiting multipliers \citep{lewis2011identifying}. Our assumption is milder, requiring only that the limiting residuals be nonnegative at inactive entries. In particular, Assumption~\ref{assum:support} holds if entries vanishing at the limit do not increase to first order in relative terms along the subsequence, as shown in Proposition~\ref{cor:mono-suff},  which is the main building block for our second main result on the stationarity of the accumulation points.
%\end{remark}
\begin{proposition}[A posteriori sufficient condition]
\label{cor:mono-suff} 
Let $\{t_\ell\}$ be subsequence with $x^{t_\ell+1}\to\bar x$.  Under the setting of Theorem \ref{thm:unconditional}, Assumption~\ref{assum:support} holds along this subsequence if for every $(i,j)\notin S(\bar x)$,
\[\liminf_{\ell\to\infty}
\rho_{t_\ell}
\log \frac{w^{t_\ell}_{ij}}{w^{t_\ell+1}_{ij}}
\geq 0.\]
In particular, it suffices that
$w^{t_\ell+1}_{ij}\leq w^{t\ell}_{ij}(1+\varepsilon_\ell^{ij})$
with $\varepsilon_\ell^{ij}=o(\alpha_{t_\ell})$. Thus, eventual nonincrease of every inactive entry along the subsequence is sufficient.
\end{proposition}

\begin{theorem}[Stationary accumulation point]\label{thm:stationarity}
Consider the GW Problem \ref{main_problem} and the A-KL-BAPG iterates ${(\pi^t,w^t)}_{t\geq0}$ generated by Algorithm~\ref{alg:adaptive-delayed-power} under the conditions of Theorem~\ref{thm:unconditional}. Let Assumption~\ref{assum:support} hold. Then every residual-vanishing subsequence has a limit $\bar x\in\mathcal C$ satisfying $0\in\nabla F(\bar x)+N_{\mathcal C}(\bar x)$. In particular, the iterates admit at least one stationary accumulation point.
\end{theorem}
Theorem~\ref{thm:stationarity} shows that projected subsequences and the two alternating sequences have the same accumulation points. Compared with KL-BAPG \citep{li2023convergent}, our result strengthens the stationarity conclusion. Under bounded accumulated asymmetric Bregman error, KL-BAPG shows that every accumulation point $(\pi^\star,w^\star)$ belongs to the relaxed BAPG fixed-point set, with $\|\pi^\star-w^\star\|=O(\rho^{-1})$ and \[\operatorname{dist}\left(\frac{\pi^\star+w^\star}{2},\mathcal X\right) =O(\rho^{-1}).\] Thus, KL-BAPG characterizes every accumulation point of the relaxed problem, whereas our result establishes exact subsequence convergence to a stationary point of the original GW problem.

The proof in Section~\ref{sec:conv-analysis} uses the accepted-step lower bound generated by the projective-diameter guard. One might hope to replace it by a lower bound with the same power as the delayed proposal, i.e., $\alpha_t\geq c(t+1)^{-\beta}$ for some $c>0$ and $\beta<1$. However, the next proposition shows that this stronger lower bound cannot be guaranteed from the logarithmic guard alone.
\begin{proposition}[Impossibility of a Universal Power Lower Bound]
\label{prop:no-universal-power-lower-bound}
Under the conditions of Theorem \ref{thm:unconditional}, there exists a valid $2\times2$ instance, constructed in Appendix \ref{app:proof_lower_bound}, such that
any guarded tail of A-KL-BAPG on this instance satisfies $\sum_{s=0}^{T-1}\alpha_s=O(\log T)$.
Consequently, no theorem using only the projective-diameter guard can guarantee an eventual lower bound $\alpha_t\geq c(t+1)^{-\beta}$ for any $c>0$ and any $\beta<1$.
\end{proposition}

\section{Convergence Analysis} \label{sec:conv-analysis}
We now establish Theorem~\ref{thm:stationarity}. We first introduce some notation and then present auxiliary propositions that will help us complete the proof.
First, define  $\eta_t:=(t+1)^{-(\beta-a/2)}$.
We note that the initial point in the following proof is $(\pi^{k_0},w^{k_0})$, not the original initialization. 
For $W\in\mathcal{C}_2$, let $r(W)=\frac{1}{\mu}(W\mathbf 1_m)$ so $r(W)=\mathbf 1_n$ when $W\in\mathcal C$. 

\begin{proposition}\label{prop:automatic-marginal-feasibility}
Let $\{(\pi^t,w^t)\}_{t\geq0}$ be the tail iterates generated by Algorithm~\ref{alg:adaptive-delayed-power}. Under the conditions of Theorem \ref{thm:unconditional},  $\|w^t\mathbf 1_m-\mu\|_1 = O(\eta_t)$ and $\|(\pi^{t+1})^\top\mathbf 1_n-\nu\|_1=O(\eta_t)$.
Moreover, for all sufficiently large~$t$,
$D_h(\pi^{t+1},w^t) = O(\eta_t^2)$
and $D_h(\pi^{t+1},w^{t+1}) = O(\eta_t^2)$.
\end{proposition}

Since $D_h$ is not symmetric, the descent analysis generates the asymmetry terms
$A_t^{\rm row} =D_h(\pi^{t+1},w^t)-D_h(w^t,\pi^{t+1})$ and
$A_t^{\rm col}=D_h(w^{t+1},\pi^{t+1})-D_h(\pi^{t+1},w^{t+1})$ with the total KL asymmetry $A_t=A_t^{\rm row}+A_t^{\rm col}$. 

\begin{proposition}\label{prop:cubic-main}
Let $\{(\pi^t,w^t)\}_{t\geq0}$ be the tail iterates generated by Algorithm~\ref{alg:adaptive-delayed-power}.  Under the conditions of Theorem \ref{thm:unconditional}, for all sufficiently large $t$,
$|A_t^{\rm row}|+|A_t^{\rm col}|=O(\eta_t^3)$
and $|A_t|=O(\eta_t^3)$.
\end{proposition}

\begin{proposition}\label{prop:lyap-main}
For $\{(\pi^t,w^t)\}_{t\geq0}$ generated by Algorithm~\ref{alg:adaptive-delayed-power}, under the conditions of Theorem \ref{thm:unconditional},
\[\sum_{t=1}^{\infty}\rho_t\left[D_h(\pi^t,\pi^{t+1})+D_h(w^t,w^{t+1})\right]<\infty.\]
\end{proposition}

\begin{proposition}\label{prop:residual}
Let $\{R_t\}_{t\geq0}$ be the corrected residuals associated with the tail iterates generated by Algorithm~\ref{alg:adaptive-delayed-power} and suppose the conditions of Theorem \ref{thm:unconditional} hold. Then there exists an integer $K_R$ such that
$\sum_{t = K_R}^\infty \alpha_t \|R_t\|_{M_t}^2 < \infty$.
Consequently, for every $N \geq K_R + 1$,
\[\min_{0\leq i \leq N-1}\|R_i\|_{M_i}^2 = O\left(\frac{1}{{\log N}}\right).\]
\end{proposition}

\textbf{Proof Sketch for Theorems \ref{thm:unconditional} and \ref{thm:stationarity} .}
Proposition~\ref{prop:automatic-marginal-feasibility} establishes vanishing marginal infeasibility through projective contraction and converts this estimate into quadratic KL cross-divergence, while Proposition~\ref{prop:cubic-main} shows that the resulting KL asymmetry is cubic. Proposition~\ref{prop:lyap-main} then incorporates these bounds into the Lyapunov analysis, which in turn yields the residual summability in Proposition~\ref{prop:residual}. Together with the guarded step-size bounds, this summability yields the vanishing residual, while Proposition~\ref{prop:automatic-marginal-feasibility}, Pinsker's inequality, and linear regularity show that the two iterates converge to their feasible projections, so all three sequences share the same accumulation points, as established in Theorem~\ref{thm:unconditional}. Moreover,  Assumption~\ref{assum:support} guarantees that every residual-vanishing subsequence limit satisfies the KKT conditions and is therefore stationary, as stated in Theorem~\ref{thm:stationarity}. Complete proofs are provided in Appendix~\ref{app:proofs}.

\subsection{Discussion on Step-Size Schedules}

\textbf{$1/k$ schedule:} A natural alternative to the proposed algorithm is the deterministic harmonic schedule $1/k$, i.e.,
$\alpha_k^\prime = (k + 1)^{-1}$ and $\rho_k^\prime =(\alpha_k^\prime)^{-1}$.
 Here, we show that our method has a less restrictive convergence guarantee.
\begin{corollary}\label{coro:1/k}
For the $1/k$ schedule initialized with $w^0>0$ and strictly
positive marginals, if $M_G<1/12$, the feasibility and residual
conclusions of Section~\ref{sec:conv-analysis} also hold.
\end{corollary}

The proof is deferred in Appendix \ref{app:1/k}. $M_G < {1}/{12}$ is a conservative sufficient condition
and there is no independent guard-slope or power parameter that can be adjusted when this sufficient condition is violated.
For the scaled harmonic schedule
$\alpha_k'=c/(k+1)$, the same argument gives the sufficient
condition $cM_G<1/12$ where a smaller $c$ reduces the initial step size.
On the other hand, in our proposed method, the projective-diameter slope $a$ and the delayed-power exponent $\beta$ are design parameters linked by the admissibility condition $\beta> 2/3+a/2$, which jointly balances the geometric control and step-size decay while providing a convergence guarantee.

\textbf{Constant Penalty Schedule:} \label{sec:dis-const-rho}
The proposed method can be viewed as an extension of the fixed-penalty
BAPG method introduced by \citet{li2023convergent} where the penalty  $\rho$ and the corresponding step size $\alpha$ are held constant across iterations. Under their convergence assumption, every accumulation point
$(\pi^\star,w^\star)$ belongs to the fixed-point set of the relaxed BAPG problem. By \citep[ Proposition~3.3]{li2023convergent}, 
    \[\operatorname{dist}\left(
        \frac{\pi^\star+w^\star}{2}, \mathcal{X}
    \right)=
    O({\rho}^{-1}).\]
For every fixed $\rho$, it converges to the stationary set of the relaxed split problem rather than the original problem. In particular, the infeasibility gap is controlled by $O(\rho^{-1})$, which only vanishes when $\rho \to \infty$. 
Our method strengthens the asymptotic guarantee by introducing an increasing penalty in the tail phase, thereby removing the permanent approximation floor associated with the fixed-penalty $\rho$ schedule. The adaptive safeguard does introduce a backtracking procedure within each tail iteration. 
Nevertheless, as shown in Remark~\ref{rem:finit-back}, this backtracking loop is well defined and terminates after finitely many reductions at each tail iteration.

\textbf{Pure Adaptive Power Schedule:} A pure adaptive power schedule with a nominal step $\alpha_k^{\prime\prime}=\alpha_0 (k+1)^{-\beta}$, $\rho_k^{\prime\prime}={\alpha_k^{\prime \prime}}^{-1}$ and  the same guard as A-KL-BAPG,
still converges since the same convergence analysis in Section~\ref{sec:conv-analysis} applies when $k_0=0$. However, it can enter a small-step regime before the coarse optimization phase is complete. 
For every $(i,j)$, 
\[\exp(\alpha_k^{\prime\prime}G_{ij}(w^k))
=1+\alpha_k^{\prime\prime}G_{ij}(w^k)
+O((\alpha_k^{\prime\prime})^2),\]
i.e., the objective-driven perturbation is in $O(\alpha_k^{\prime\prime})$. 
When the iterates are still far from this region, the primary task is to obtain substantial objective reduction and approach the stationary set of the relaxed split problem. 
If the guard is imposed at initialization, these transient changes in the projective diameter may cause repeated backtracking, thereby forcing the accepted step toward the conservative safe scale $h_t$, despite the need for larger objective-driven steps. 
Therefore, the finite fixed-penalty prefix helps handle the initial transient and avoid premature step-size decay. 

%%%%%%%%%%%%%%%%%%%%%%%%%%%%%%%%%%%%%%%%%%%%%%
%%%%%%%%%%%%%%%%%%%%%%%%%%%%%%%%%%%%%%%%%%%%%%
%%%%%%%%%%%%%%%%%%%%%%%%%%%%%%%%%%%%%%%%%%%%%%
%%%%%%%%%%%%%%%%%%%%%%%%%%%%%%%%%%%%%%%%%%%%%%
\section{Numerical Analysis}
%%%%%%%%%%%%%%%%%%%%%%%%%%%%%%%%%%%%%%%%%%%%%%
%%%%%%%%%%%%%%%%%%%%%%%%%%%%%%%%%%%%%%%%%%%%%%
%%%%%%%%%%%%%%%%%%%%%%%%%%%%%%%%%%%%%%%%%%%%%%
%%%%%%%%%%%%%%%%%%%%%%%%%%%%%%%%%%%%%%%%%%%%%%

We evaluate A-KL-BAPG on two important applications of GW: graph alignment \citep{xu2019gromov,li2023convergent} and heterogeneous domain adaptation \citep{yan2018semi,aritake2026domain}. We compare against major GW/FGW solver families, including conditional-gradient methods \citep{peyre2016gromov,titouan2019optimal}, entropic projected-gradient methods \citep{solomon2016entropic}, proximal and Bregman methods \citep{xu2019gromov,li2023convergent}, and recent entropic-GW and inexact projected-gradient algorithms \citep{rioux2024entropic,liang2026provably}. We also include graph-specific and domain adaptation baselines \citep{vogelstein2015fast,xu2019scalable,chowdhury2021generalized,courty2016optimal,courty2017joint,aritake2022unsupervised}. 

\subsection{GW Objective, Feasibility, and Stationarity Residual}
Given the source and target structure matrices $D_X$ and $D_Y$, the GW objective evaluated at a transport plan $X$ is $F(X) = -\operatorname{Tr}\left(D_X X D_Y X^\top\right)$. For the fused Gromov-Wasserstein problem involving an additional cross-domain feature cost $M$, we report the fused objective
    \[F_{\mathrm{FGW}}(X) = -2\lambda\langle D_X X D_Y,X\rangle +(1-\lambda)\langle M,X\rangle, \qquad \lambda \in (0,1),\]
which is the nonconstant part of squared-loss FGW. Lower values indicate better minimization of the corresponding GW or FGW criterion. We report prediction accuracy separately to measure downstream performance.

Denote $W^{t}$ as the second-block iterate returned after iteration $k$. We report its marginal feasibility gap
\[r_{\mathrm{feas}}(W^{t+1}) := \left\|W^{t+1}\mathbf{1}-\mu\right\|_1 + \left\|(W^{t+1})^\top\mathbf{1}-\nu\right\|_1.\] We evaluate whether the output matrix is close to a stationary point of the constrained GW problem. 
Since the split iterates produced by BAPG may not exactly lie in $\mathcal C$, we first construct a feasible shadow coupling $\widehat W_t\in \mathcal C$ by rescaling the rows and columns of the reported matrix. We then measure stationarity at $\widehat W_t$ by defining the projected-gradient stationarity residual
\[
r_{\mathrm{stat}}(\widehat W_t) = \frac{1}{\eta}
\big\| \widehat W_t - \mathcal{P}_{\mathcal C}
\big(\widehat W_t-\eta \nabla F(\widehat W_t)
\big) \big\|_F.
\]
Note that for $F(W) = - \mathrm{Tr}(D_X WD_Y W^\top)$, $\nabla F(W)
= -\left(D_X^\top W D_Y^\top + D_X W D_Y\right)$. Then
\begin{align*}
\|\nabla F(W_1)-\nabla F(W_2)\|_F
& \le \|D_X^\top (W_1 - W_2) D_Y^\top\|_F + \|D_X(W_1 - W_2) D_Y\|_F  \\
& \le \|D_X^\top\|_2\|D_Y^\top\|_2\|W_1 - W_2\|_F + \|D_X\|_2\|D_Y\|_2\|W_1 - W_2\|_F \\
& = 2\|D_X\|_2\|D_Y\|_2\|W_1-W_2\|_F .
\end{align*}
Thus $2\|D_X\|_2\|D_Y\|_2$ is a valid Lipschitz upper bound for the gradient, and stepsize $\eta$ is defined as
$\eta := (2\|D_X\|_2\|D_Y\|_2)^{-1}$.
\begin{remark}
Note that $r_{\rm stat}(\widehat W_t) = 0$ is equivalent to $\widehat W_t = \mathcal P_{\mathcal C} \big( \widehat W_t-\eta \nabla F(\widehat W_t) \big)$, i.e., $0\in\nabla F(\widehat W_t)+\mathcal N_{\mathcal C}(\widehat W_t)$, which is the KKT first-order stationarity condition for the GW problem over $\mathcal C$. A similar quantity $\big\| \pi-\operatorname{Proj}_{\mathcal C} \big(\pi+D_X\pi D_Y\big) \big\|_F$ is used by \citet[Proposition~3.1]{li2023convergent} to establish a Luo--Tseng error bound for GW.  Their quantity is the projected-gradient displacement with $\eta=1/2$, whereas ours uses a Lipschitz-based step size. 
When the projected-gradient residual is sufficiently small, the Luo--Tseng error bound \citep{li2023convergent} bounds the distance from $\widehat W_t$ to the GW stationary set by a constant multiple of $r_{\mathrm{stat}}(\widehat W_t)$.
Thus, a smaller residual indicates less first-order improvement remained and a closer proximity to the stationary set. 
\end{remark}
\subsection{Graph Alignment} \label{sec:syn_ga}
\paragraph{Synthetic Graph Alignment}
Following \citet{li2023convergent}, we draw a 250-node Gaussian random partition graph and corrupt it by adding random missing edges and 75 noisy target nodes. We use adjacency matrices as $D_X$ and $D_Y$, uniform marginals, and average over 30 seeds.
 We compare A-KL-BAPG with fixed KL-BAPG \citep{li2023convergent}, KL-BAPG with post-hoc Sinkhorn rounding,  $1/k$ schedule, pure-power schedule, and several GW solvers, including conditional-gradient GW (POT CG) \citep{titouan2019optimal}, proximal-point GW (POT PPA) \citep{xu2019gromov}, entropic projected-gradient GW (POT EPGD) \citep{peyre2016gromov}, and POT's implementation of BAPG.
We additionally include the recent inexact projected-gradient (iPG) method of \citet{liang2026provably}. For A-KL-BAPG, we use $K=4000$ total iterations, $k_0=1000$, $\rho_0=0.1$, and tail parameters $\alpha_0=1$, $\beta=0.75$, $a=0.1$, and $\tau=0.5$. For the pure power schedule, $\beta = 0.75$. The hyperparameters for each baseline are set according to the settings in their respective papers, except that POT BAPG uses $\varepsilon=0.4$ to match the penalty  $\rho=0.1$ of KL-BAPG.  
Following \citet{li2023convergent}, we use top-1 matching accuracy to evaluate the quality of graph alignment. 
We also report the mean GW objective, feasibility gap, stationarity residual, and runtime, where higher top-1 accuracy and lower objective, feasibility gap, stationarity residual, and runtime are better. Specific graph construction details, baseline information, and accuracy definition are provided in Appendix \ref{app:syn_ga}.

Table~\ref{tab:syn_ga} and Figure~\ref{fig:syn_ga_iteration} confirm the strength of BAPG for graph alignment: KL-BAPG and POT BAPG obtain high accuracy and low GW objectives, but the returned couplings are far from feasible and stationary. Post-hoc Sinkhorn rounding preserves KL-BAPG's objective and accuracy while reducing the feasibility gap, yet still yields a higher stationarity residual, indicating that marginal correction alone cannot restore stationarity. A-KL-BAPG retains the same accuracy and objective while reducing the feasibility gap by more than $99\%$ and the stationarity residual by $98\%$ relative to KL-BAPG. $1/k$ and pure-power schedules are nearly feasible but fall short on accuracy and stationarity. POT CG attains a smaller feasibility gap and stationarity residual, yet with substantially worse accuracy and GW objective. POT PPA and iPG achieve high performance across all metrics but require significantly longer runtimes due to inner iterations.

\begin{table*}[ht]
\centering
\small
\begin{tabular}{lccccc}
\toprule
Method & Accuracy & GW objective & Feasibility & Stationarity Residual & Runtime (s) \\
\midrule
KL-BAPG & $0.966$ & $-0.034$ & $0.098$ & $0.470$ & $278.981$ \\
KL-BAPG + Sinkhorn & $0.966$ & $-0.034$ & $0.005$ & $0.140$ & $279.987$ \\
\textbf{A-KL-BAPG} & $\mathbf{0.966}$ & $\mathbf{-0.034}$ & $\mathbf{5.524\times10^{-6}}$ & $\mathbf{0.011}$ & $\mathbf{225.242}$ \\
$1/k$ schedule & $0.022$ & $-0.003$ & $1.430\times10^{-7}$  & $0.204$ & $54.171$ \\
Pure power & $0.022$ & $-0.003$ & $1.110\times10^{-6}$  & $0.205$ & $51.926$ \\
POT CG & $0.015$ & $-0.017$ & $2.390\times10^{-17}$ & $7.730\times10^{-4}$ & $5.390$ \\
POT PPA & $0.966$ & $-0.034$ & $1.450\times10^{-5}$  & $0.002$ & $2112.932$ \\
POT EPGD & $0.065$ & $-0.003$ & $4.780\times10^{-12}$ & $0.233$ & $0.871$ \\
POT BAPG           & $0.966$ & $-0.034$ & $0.098$ & $0.470$ & $298.960$ \\
Degree Hungarian   & $0.003$ & $-0.003$ & $1.310\times10^{-15}$ & $0.750$ & $0.009$ \\
SciPy FAQ & $0.059$ & $-0.015$ & $1.230\times10^{-15}$ & $0.333$ & $11.250$ \\
iPG & $1.000$ & $-0.034$ & $3.100\times10^{-10}$ & $3.090\times10^{-4}$ & $1927.503$ \\
\bottomrule
\end{tabular}
\caption{Synthetic Graph Alignment Results.}\label{tab:syn_ga}
\end{table*}

\begin{figure}[ht]
\centering
\includegraphics[width=\linewidth]{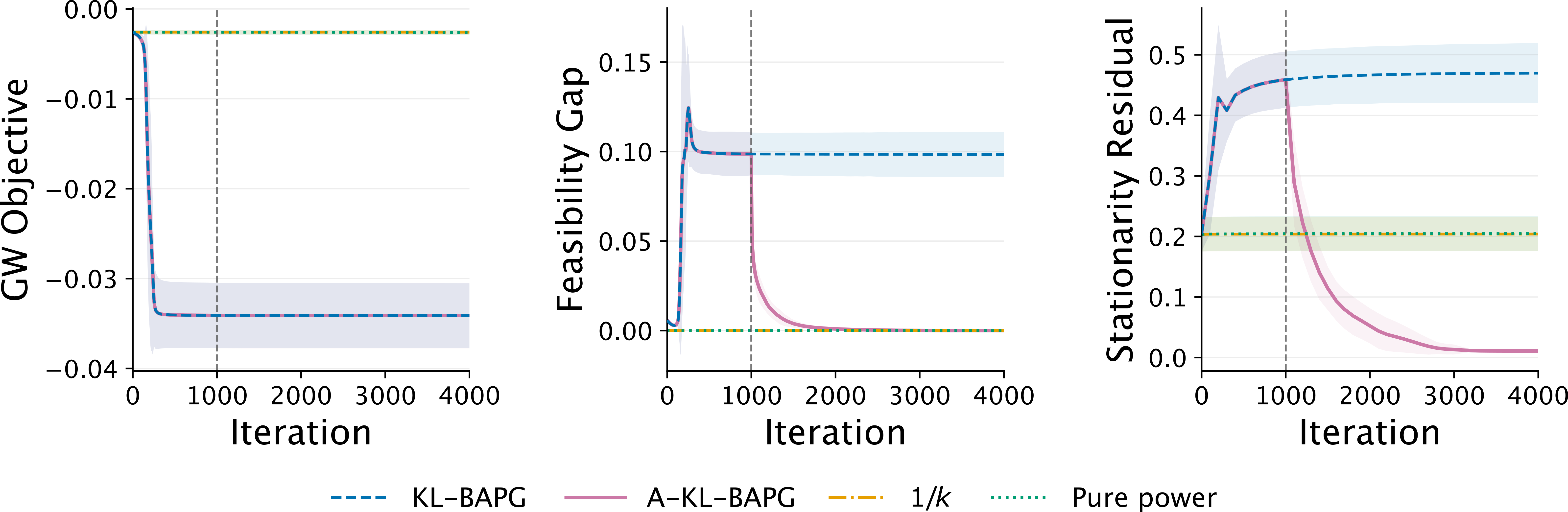}
\vspace{-15pt}
\captionof{figure}{Iteration convergence comparison.}
\label{fig:syn_ga_iteration}
\end{figure}

\paragraph{Real Graph Alignment.}
We evaluate A-KL-BAPG on the PROTEINS and ENZYMES datasets from TUDataset \citep{Borgwardt2005Protein,dobson2003distinguishing,morris2020tudataset}.
Following \citet{li2023convergent}, we construct each target graph by adding $10\%$ missing edges and $10\%$ noisy nodes to its source graph. We use unweighted adjacency matrices as the GW structure matrices and uniform node measures. Results are averaged over 30 independent corruptions per graph. Dataset and corruption details are provided in Appendix~\ref{app:real_ga}.
We use the same baselines as in Section \ref{sec:syn_ga} and additionally include Degree Hungarian \citep{kuhn1955hungarian}, FAQ with ten random restarts \citep{vogelstein2015fast}, ScalaGW \citep{xu2019scalable}, specGW \citep{chowdhury2021generalized}, and BPG and BPG-S \citep{xu2019gromov}. Baseline hyperparameters follow their respective papers, and BAPG schedules use the same hyperparameters as in Section \ref{sec:syn_ga}. In Section \ref{sec:syn_ga}, we report top-1 accuracy for direct comparison. Here, however, several baselines, such as FAQ, directly return discrete assignments, whereas top-1 decoding can map multiple nodes to the same target node. Therefore, we apply the Hungarian decoding for a global one-to-one alignment and report the resulting matching accuracy
\citep{kuhn1955hungarian}. GW objective, feasibility gap, and stationarity residual are also reported if applicable. The decoding and metric definitions are provided in Appendix~\ref{app:real_ga}.
\begin{table*}[t]
\vspace{-0.17cm}
\centering
\scriptsize
\setlength{\tabcolsep}{1pt}
\caption{ENZYMES and PROTEINS graph alignment performance.}
\label{tab:real_ga}

\begin{NiceTabular}{l ccccc ccccc}
\toprule
& \multicolumn{5}{c}{ENZYMES}
& \multicolumn{5}{c}{PROTEINS} \\
\cmidrule(lr){2-6}
\cmidrule(lr){7-11}
Method
& Accuracy & Objective & Feasibility & Stationarity & Time (s)
& Accuracy & Objective & Feasibility & Stationarity & Time (s) \\
\midrule

Degree Hungarian
& $0.282$ & $-0.041$ & $9.91\times10^{-17}$ & $0.406$ & $0.000$
& $0.303$ & $-0.077$ & $1.01\times10^{-16}$ & $0.390$ & $0.000$ \\

SciPy FAQ
& $0.503$ & $-0.093$ & $1.12\times10^{-16}$ & $0.117$ & $0.040$
& $0.469$ & $-0.131$ & $1.15\times10^{-16}$ & $0.100$ & $0.117$ \\

ScalaGW
& $0.137$ & $-0.032$ & $1.38\times10^{-16}$ & -- & $0.283$
& $0.185$ & $-0.058$ & $1.51\times10^{-16}$ & -- & $0.249$ \\

SpecGW
& $0.199$ & $-0.067$ & $2.82\times10^{-16}$ & -- & $0.052$
& $0.204$ & $-0.104$ & $2.86\times10^{-16}$ & -- & $0.029$ \\

BPG
& $0.706$ & $-0.091$ & $6.89\times10^{-6}$ & $0.020$ & $0.175$
& $0.645$ & $-0.119$ & $5.59\times10^{-6}$ & $0.034$ & $0.151$ \\

BPG-S
& $0.709$ & $-0.091$ & $5.37\times10^{-4}$ & $0.020$ & $0.032$
& $0.649$ & $-0.119$ & $6.59\times10^{-4}$ & $0.035$ & $0.038$ \\

iPG
& $0.694$ & $-0.091$ & $2.02\times10^{-5}$ & $0.002$ & $2.681$
& $0.673$ & $-0.122$ & $1.95\times10^{-5}$ & $0.002$ & $4.060$ \\

POT CG
& $0.204$ & $-0.085$ & $3.40\times10^{-16}$ & $0.004$ & $0.074$
& $0.211$ & $-0.122$ & $3.26\times10^{-16}$ & $0.008$ & $0.087$ \\

POT PPA
& $0.731$ & $-0.092$ & $2.21\times10^{-4}$ & $0.006$ & $52.194$
& $0.693$ & $-0.121$ & $2.20\times10^{-4}$ & $0.007$ & $69.666$ \\

POT EPGD
& $0.684$ & $-0.087$ & $7.17\times10^{-6}$ & $0.103$ & $77.281$
& $0.632$ & $-0.115$ & $1.03\times10^{-5}$ & $0.103$ & $69.272$ \\

POT BAPG
& $0.200$ & $-0.058$ & $0.782$ & $0.440$ & $0.131$
& $0.263$ & $-0.081$ & $0.754$ & $0.436$ & $0.473$ \\

\midrule

$1/k$ schedule
& $0.326$ & $-0.018$ & $1.07\times10^{-6}$ & $0.128$ & $0.380$
& $0.336$ & $-0.040$ & $1.63\times10^{-6}$ & $0.147$ & $0.915$ \\

Pure power
& $0.358$ & $-0.021$ & $7.08\times10^{-6}$ & $0.148$ & $0.389$
& $0.369$ & $-0.052$ & $9.86\times10^{-6}$ & $0.160$ & $0.917$ \\

KL-BAPG
& $0.748$ & $-0.094$ & $0.113$ & $0.178$ & $1.483$
& $0.716$ & $-0.125$ & $0.108$ & $0.184$ & $3.551$ \\

KL-BAPG + Sinkhorn
& $0.752$ & $-0.093$ & $0.092$ & $0.417$ & $1.689$
& $0.719$ & $-0.123$ & $0.079$ & $0.414$ & $3.845$ \\

\textbf{A-KL-BAPG}
& $\mathbf{0.747}$ & $\mathbf{-0.091}$ & $\mathbf{0.043}$ & $\mathbf{0.130}$ & $\mathbf{2.445}$
& $\mathbf{0.718}$ & $\mathbf{-0.120}$ & $\mathbf{0.041}$ & $\mathbf{0.132}$ & $\mathbf{4.056}$ \\

\bottomrule

\CodeAfter
\tikz \draw[line width=0.5pt] (1-|7) -- (19-|7);
\end{NiceTabular}
\end{table*}

Table~\ref{tab:real_ga} confirms the strength of A-KL-BAPG on real graph alignment. It achieves comparable accuracy with fixed KL-BAPG on both ENZYMES and PROTEINS. At the same time, A-KL-BAPG reduces the feasibility gap by approximately $62\%$ on both datasets and the stationarity residual by $27\%$ on ENZYMES and $28\%$ on PROTEINS. Consistent with the synthetic results, Sinkhorn rounding improves KL--BAPG's feasibility but worsens its stationarity, and is worse than A-KL-BAPG on both metrics. $1/k$ and pure-power schedules approach feasibility but yield substantially lower matching accuracy. Projection-based methods attain smaller feasibility and stationarity residuals, but generally at the cost of lower matching accuracy.  Thus, A-KL-BAPG achieves the most favorable balance among the BAPG schedules in terms of alignment accuracy, feasibility, and stationarity. 
 
\subsection{Domain Adaptation} \label{sec:syn_da}
\paragraph{Synthetic Domain Adaptation}
We test FGW in synthetic domain adaptation with four binary tasks (Blobs, Two Circles, Two Moons, and Two Spirals), where the source has one common feature, and the target additionally has a target-only feature. The source geometry is label-aware, and the target geometry is a $4$-nearest-neighbor (4-nn) graph, following \citet{aritake2026domain}. We use 400 source and 200 target observations, and repeat each experiment with the same 30 random seeds. Following \citet{aritake2026domain}, we set $\lambda = 0.01$. Specific experiment settings are included in Appendix \ref{app:syn_da}.
We compare KL--BAPG ($\rho_k=0.5$), KL BAPG + Sinkhorn rounding, $1/k$, pure power ($\beta = 0.75$), and our A-KL-BAPG, with $k_0 = 500$, $\alpha_0=1$, $\beta=0.75$, $a=0.1$, and $\tau=0.5$. Other baselines include OT-based domain adaptation \citep{courty2016optimal}, naive OT cluster \citep{ng2001spectral}, GW \citep{peyre2016gromov}, FGW \citep{titouan2019optimal}, and the LA-FGW, GB-FGW, and LAGB-FGW \citep{aritake2026domain}. OT, GW, and FGW baselines are implemented using POT \citep{flamary2021pot}. Hyperparameters for baselines are set according to their respective papers, with details provided in Appendix \ref{app:syn_da}. 
Specific accuracy calculation is provided in Appendix \ref{app:syn_da}. 

\begin{figure}[ht]
    \centering

    \begin{subfigure}{0.22\linewidth}
        \centering
        \includegraphics[width=\linewidth]{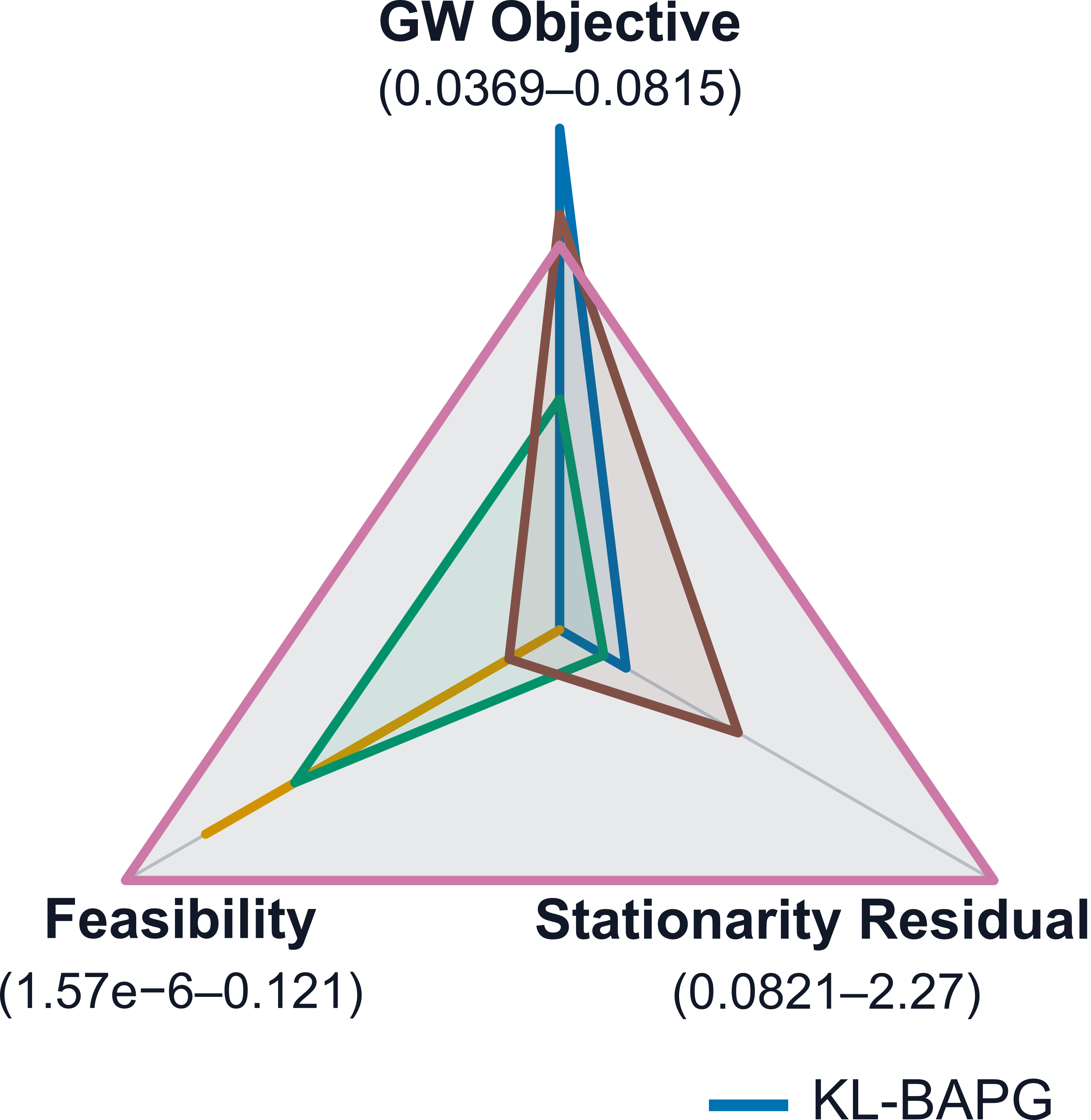}
        \caption{Blobs}
        \label{fig:blobs}
    \end{subfigure}
    \begin{subfigure}{0.22\linewidth}
        \centering
    \includegraphics[width=\linewidth]{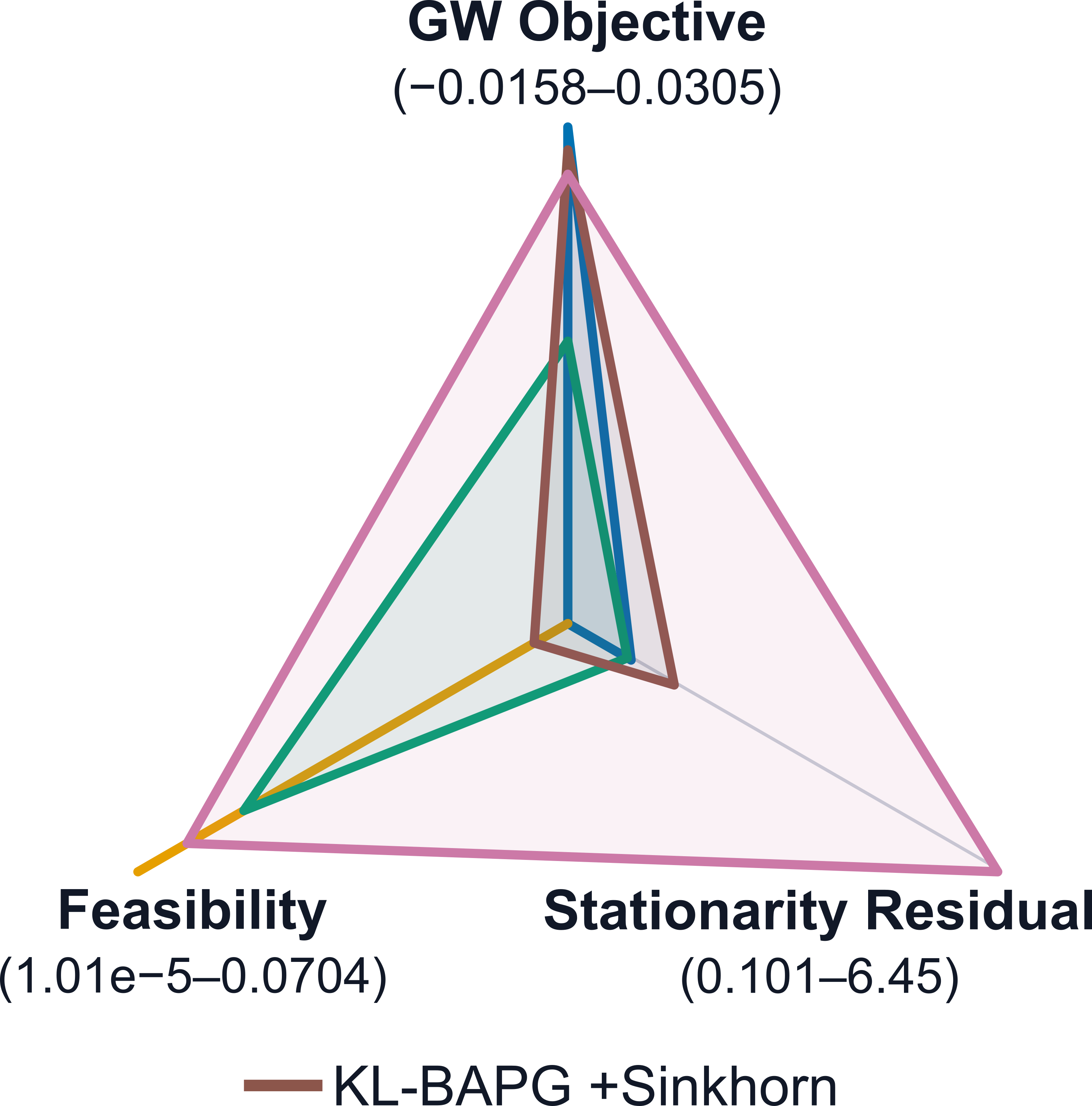}
        \caption{Two Circles}
        \label{fig:two_circles}
    \end{subfigure}
    \begin{subfigure}{0.22\linewidth}
        \centering
    \includegraphics[width=\linewidth]{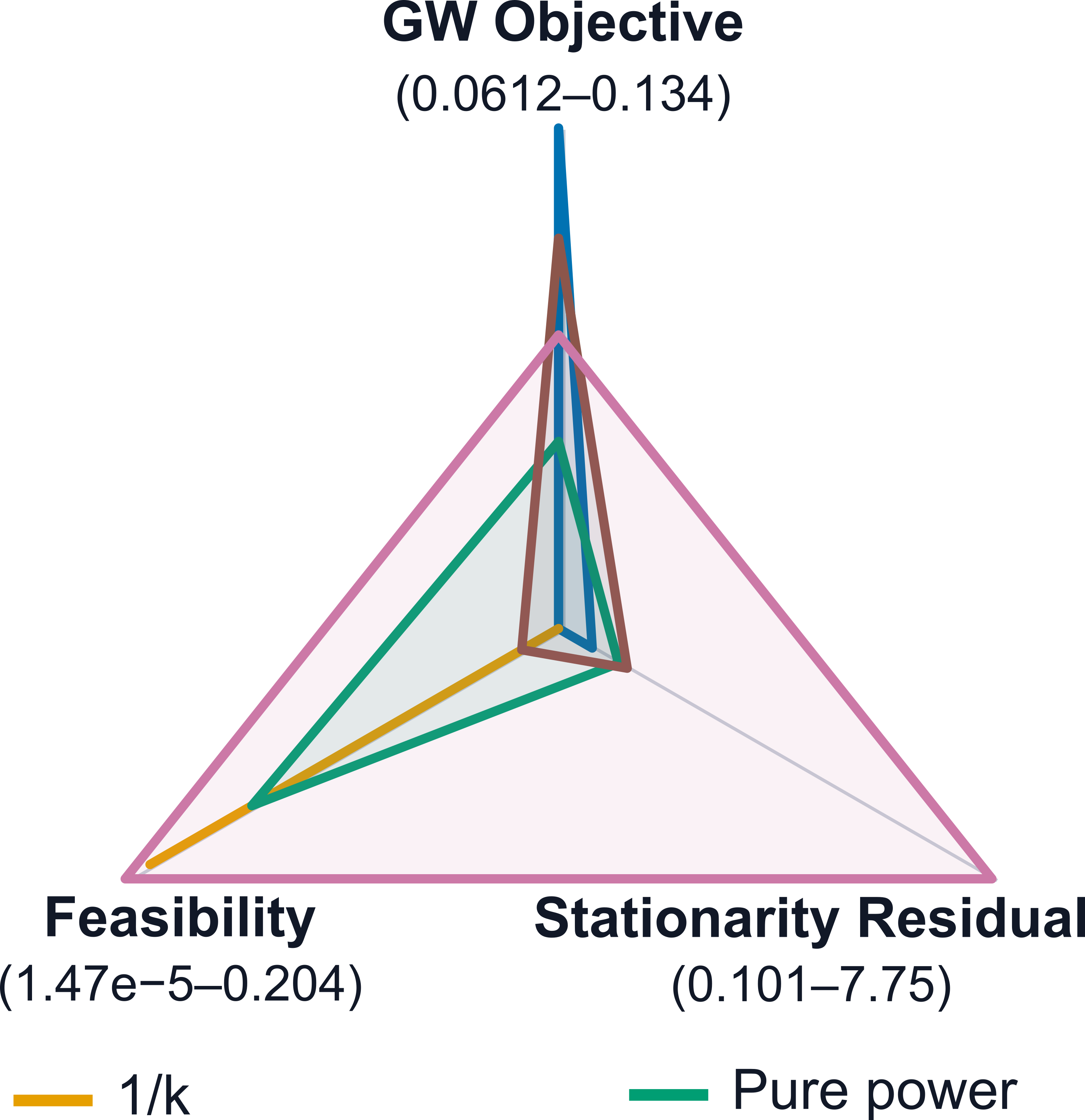}
        \caption{Two Moons}
        \label{fig:two_moons}
    \end{subfigure}
    \begin{subfigure}{0.22\linewidth}
        \centering
    \includegraphics[width=\linewidth]{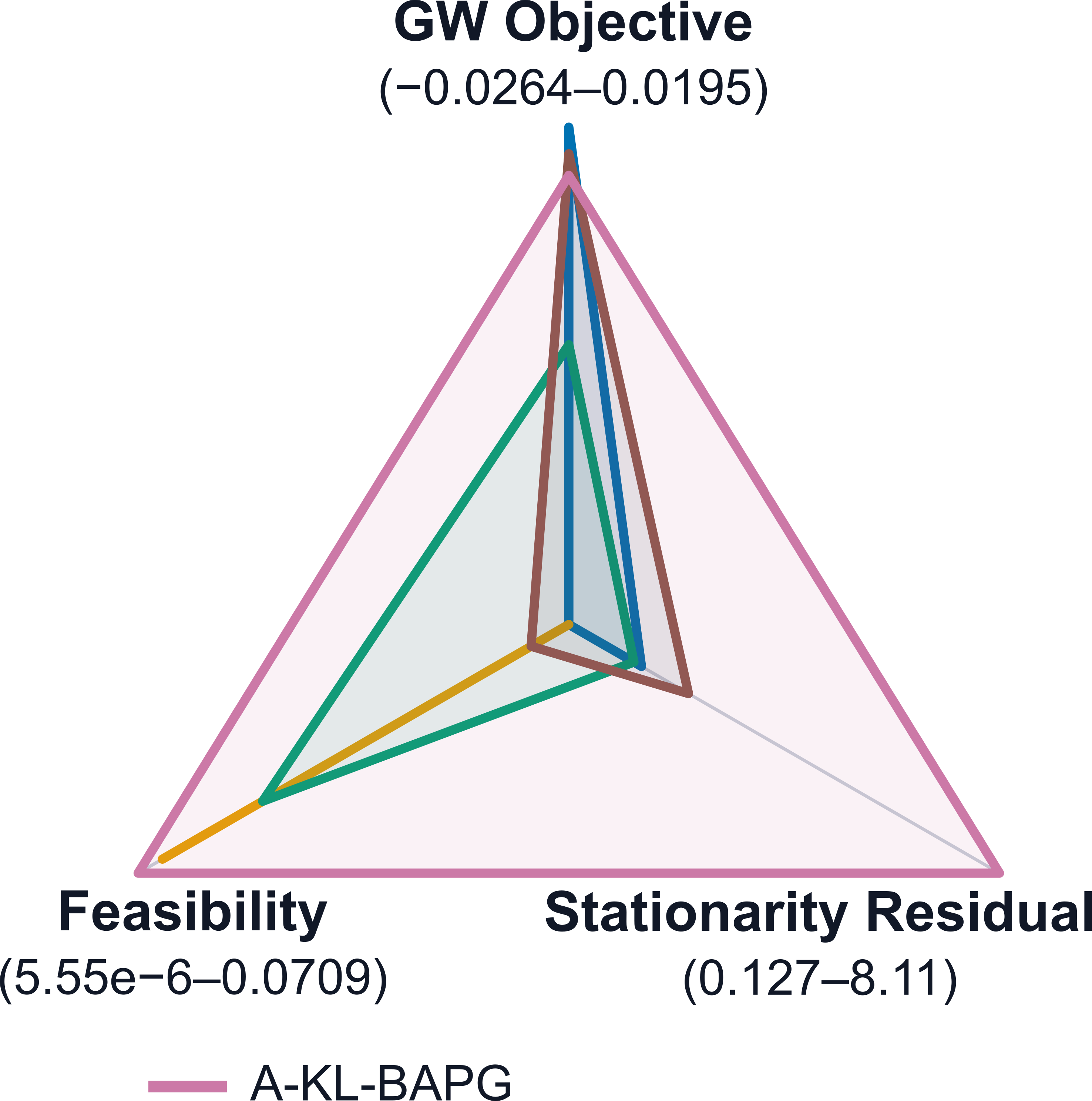}
        \caption{Two Spirals}
        \label{fig:two_spirals}
    \end{subfigure}
    \caption{Optimization diagnostics of synthetic domain adaptation experiment. Outward is better, i.e., lower in values. \vspace{-0.3cm}}
    \label{fig:da-diag}
\end{figure}

Figure \ref{fig:da-diag} shows that fixed-penalty KL--BAPG makes the greatest raw objective progress at the cost of infeasibility split, while Sinkhorn rounding improves its feasibility but remains substantially less stationary than A-KL-BAPG.
A-KL-BAPG retains most of the objective gain and achieves the smallest stationarity residuals, whereas the $1/k$ and pure-power schedules improve feasibility at the cost of objective progress.
Table \ref{tab:da-results}(a) further confirms the strength of A-KL-BAPG, where it exhibits the highest mean accuracy across the four data sets among all methods.  It is therefore the best overall method, while POT LAGB--FGW and naive OT clustering remain competitive on individual data sets. Specific reconstruction visualizations are provided in Figure \ref{fig:da_syn_vis} in Appendix \ref{app:syn_da}.

\begin{table*}[ht]
\vspace{-0.1cm}
\centering
\scriptsize
\setlength{\tabcolsep}{2.5pt}
\renewcommand{\arraystretch}{1.12}
\caption{Synthetic target accuracy and HAR70+ results.
$(^*)$ Included for completeness and not directly comparable to the
other methods because they use additional task-specific knowledge and
assumptions \citep{aritake2022unsupervised}. }
\label{tab:da-results}

\begin{NiceTabular*}{\textwidth}
{@{\extracolsep{\fill}}lcccc lcccc@{}}
\toprule
\multicolumn{5}{c}{\textbf{(a) Synthetic target accuracy}}
&
\multicolumn{5}{c}{\textbf{(b) HAR70+ performance}} \\
\cmidrule(r){1-5}
\cmidrule(l){6-10}

Method & Blobs & Circles & Moons & Spirals
& Method & Accuracy & \shortstack{Objective}
& Feasibility & \shortstack{Stationarity} \\
\midrule

KL--BAPG & $0.955$ & $0.822$ & $0.917$ & $0.905$
& KL--BAPG & $0.880$ & $-0.012$ & $0.034$ & $1.559$ \\

KL--BAPG + Sinkhorn & $0.960$ & $0.821$ & $0.934$ & $0.907$
& KL--BAPG + Sinkhorn & $0.880$ & $-0.011$ & $0.010$ & $1.067$ \\

$1/k$ schedule & $0.981$ & $0.667$ & $0.801$ & $0.529$
& $1/k$ schedule & $0.662$ & $0.017$ & $5.04\times 10^{-6}$ & $4.029$ \\

Pure power & $0.981$ & $0.696$ & $0.848$ & $0.570$
& Pure power & $0.741$ & $2.97\times 10^{-4}$ & $4.25\times 10^{-5}$ & $2.311$ \\

\textbf{A-KL-BAPG} & $\mathbf{0.984}$ & $\mathbf{0.833}$ & $\mathbf{0.963}$ & $\mathbf{0.906}$
& \textbf{A-KL-BAPG} & $\mathbf{0.880}$ & $\mathbf{-0.010}$
& $\mathbf{5.07\times10^{-4}}$ & $\mathbf{0.207}$ \\

\cmidrule(r){1-5}
\cmidrule(l){6-10}

POT OT & $0.913$ & $0.685$ & $0.786$ & $0.597$
& POT OT & $0.827$ & -- & -- & $0.113$ \\

POT FGW & $0.912$ & $0.689$ & $0.786$ & $0.600$
& POT FGW & $0.827$ & -- & -- & $0.112$ \\

POT LA--FGW & $0.932$ & $0.708$ & $0.636$ & $0.729$
& POT LA--FGW & $0.829$ & -- & -- & $0.112$ \\

POT GB--FGW & $0.913$ & $0.684$ & $0.782$ & $0.596$
& POT GB--FGW & $0.854$ & -- & -- & $0.088$ \\

POT LAGB--FGW & $0.940$ & $0.826$ & $0.975$ & $0.859$
& POT LAGB--FGW$^*$ & $0.879$ & -- & -- & $0.086$ \\

\cmidrule(r){1-5}
\cmidrule(l){6-10}

POT GW & $0.502$ & $0.696$ & $0.530$ & $0.489$
& JDOT & $0.805$ & -- & -- & $1.857$ \\

Naive OT & $0.965$ & $0.605$ & $1.000$ & $0.947$
& DAEVS$^*$ & $0.895$ & -- & -- & -- \\

\bottomrule

\CodeAfter
\tikz \draw[line width=0.5pt]
    (1-|6) -- (15-|6);

\end{NiceTabular*}
\end{table*}

\paragraph{Real-data Domain Adaptation} We evaluate on HAR70+, a human-activity-recognition dataset containing motion-sensor recordings from 18 older adults performing four activities: walking, standing, sitting, and lying with accelerometers placed on the back and thigh to record motion along three spatial axes at 50\,Hz \citep{har70, aritake2026domain}.
The target graph is constructed by treating windows as nodes and connecting similar windows through a $4$-nn graph.
We evaluate all 18 subjects across 30 random splits, yielding 540 trials. Preprocessing, feature construction, and split details are provided in Appendix~\ref{app:real_da}.
Hyperparameters are set according to Section \ref{sec:syn_da}. 
We additionally include DAEVS \citep{aritake2022unsupervised} and JDOT \citep{courty2017joint}, with hyperparameters set according to the corresponding papers.
We report target classification accuracy, i.e., the proportion of target samples whose predicted labels match ground-truth labels, for all methods and stationarity for all methods solving an FGW problem. We also include raw GW objective and marginal feasibility gap for KL--BAPG (+ Sinkhorn rounding), $1/k$, pure power, and A-KL-BAPG. Other baselines are not optimizing over the same objective.

Table~\ref{tab:da-results} shows that A-KL-BAPG achieves the strongest overall performance among the BAPG variants. It matches KL--BAPG in target accuracy with a comparable GW objective and reduces the feasibility gap by $99\%$ and the stationarity residual by $85\%$. Sinkhorn rounding improves KL--BAPG's feasibility at the cost of stationarity, while A-KL-BAPG performs better on both metrics.
Among all baselines, A-KL-BAPG achieves one of the highest accuracies, second only to DAEVS, which is tailored to domain adaptation. DAEVS alternates between OT and hard pseudo-labeling and trains an RBF-SVM on the full set of target features, thereby denoising transported labels and exploiting additional sensor information \citep{aritake2022unsupervised}.
This advantage relies on informative initial pseudo-labels, a suitable target classifier, and conditional invariance of the additional features. In contrast, A-KL-BAPG requires neither classifier training nor pseudo-label feedback, yet achieves comparable accuracy. POT FGW baselines achieve competitive accuracy at much lower runtime.
With $\lambda=0.01$, the objective places much greater weight on the linear-feature cost, which may favor methods that linearize the structural term.
However, this speed--accuracy tradeoff is task-dependent: POT CG, SpecGW, and ScalaGW also use linearized structural updates, yet achieve substantially lower accuracy in graph alignment, as shown in Table~\ref{tab:syn_ga}-\ref{tab:real_ga}. A-KL-BAPG remains competitive across both applications, though sometimes at a higher computational cost.
The full comparison against all baselines is reported in  Table~\ref{tab:har70-all-methods} in Appendix \ref{app:real_da}. 

\section{Conclusion}
We proposed Adaptive KL-BAPG to remove the persistent feasibility gap of fixed-penalty BAPG while retaining its inexpensive alternating updates. Combining a finite fixed-penalty burn-in with an increasing-penalty phase and a projective-diameter safeguard to control the evolving coupling geometry, we achieve vanishing marginal infeasibility and absolute summability of the accumulated asymmetric Bregman errors.
Under a support regularity condition, the iterates admit a stationary accumulation point of the original constrained GW problem. Synthetic and real experiments in graph alignment and domain adaptation show improved feasibility and stationarity compared to KL-BAPG while maintaining competitive objective values and predictive performance. 
Future work will focus on strengthening the current subsequence guarantee to full-sequence convergence, weakening the support regularity condition, and reducing the computational cost of the safeguards.

\bibliography{refs}

@inproceedings{
li2023convergent,
title={A Convergent Single-Loop Algorithm for Relaxation of {G}romov-{W}asserstein in Graph Data },
author={Jiajin Li and Jianheng Tang and Lemin Kong and Huikang Liu and Jia Li and Anthony Man-Cho So and Jose Blanchet},
booktitle={The Eleventh International Conference on Learning Representations },
year={2023},
url={https://openreview.net/forum?id=0jxPyVWmiiF}
}

@article{chen2024spurious,
  title={Spurious Stationarity and Hardness Results for Bregman Proximal-Type Algorithms},
  author={Chen, He and Li, Jiajin and So, Anthony Man-Cho},
  journal={arXiv preprint arXiv:2404.08073},
  year={2024}
}

@article{aritake2026domain,
  author  = {Aritake, Toshimitsu and Hino, Hideitsu},
  title   = {Domain Adaptation With Additional Features via Label-Aware and Graph-Based Fused {G}romov-{W}asserstein Optimal Transport},
  journal = {Neural Computation},
  year    = {2026},
  volume  = {38},
  number  = {7},
  pages   = {1203--1232},
  doi     = {10.1162/NECO.a.1525}
}

@inproceedings{ballu2023mirror,
  title={Mirror Sinkhorn: fast online optimization on transport polytopes},
  author={Ballu, Marin and Berthet, Quentin},
  booktitle={International Conference on Machine Learning},
  pages={1595--1613},
  year={2023},
  organization={PMLR}
}

@article{benamou2015iterative,
  title={Iterative Bregman projections for regularized transportation problems},
  author={Benamou, Jean-David and Carlier, Guillaume and Cuturi, Marco and Nenna, Luca and Peyr{\'e}, Gabriel},
  journal={SIAM Journal on Scientific Computing},
  volume={37},
  number={2},
  pages={A1111--A1138},
  year={2015},
  publisher={SIAM}
}

@article{chizat2024annealed,
  title={Annealed sinkhorn for optimal transport: convergence, regularization path and debiasing},
  author={Chizat, L{\'e}na{\"\i}c},
  journal={arXiv preprint arXiv:2408.11620},
  year={2024}
}

@article{cuturi2013sinkhorn,
  title={Sinkhorn distances: Lightspeed computation of optimal transport},
  author={Cuturi, Marco},
  journal={Advances in neural information processing systems},
  volume={26},
  year={2013}
}

@article{courty2016optimal,
  title={Optimal transport for domain adaptation},
  author={Courty, Nicolas and Flamary, R{\'e}mi and Tuia, Devis and Rakotomamonjy, Alain},
  journal={IEEE transactions on pattern analysis and machine intelligence},
  volume={39},
  number={9},
  pages={1853-1865},
  year={2016},
  publisher={IEEE}
}

@inproceedings{demetci2024breaking,
  title={Breaking isometric ties and introducing priors in {G}romov-{W}asserstein distances},
  author={Demetci, Pinar and Tran, Quang Huy and Redko, Ievgen and Singh, Ritambhara},
  booktitle={International Conference on Artificial Intelligence and Statistics},
  pages={298-306},
  year={2024},
  organization={PMLR}
}

@article{deng2026glncd,
  title={{GLNCD}: Graph-Level Novel Category Discovery},
  author={Deng, Bowen and Fu, Lele and Huang, Sheng and Liao, Tianchi and Chen, Jialong and Tao, Zhang and Chen, Chuan},
  journal={Advances in Neural Information Processing Systems},
  volume={38},
  pages={171058-171087},
  year={2026}
}

@inproceedings{eveson1995elementary,
  title={An elementary proof of the Birkhoff-Hopf theorem},
  author={Eveson, Simon P and Nussbaum, Roger D},
  booktitle={Mathematical Proceedings of the Cambridge Philosophical Society},
  volume={117},
  pages={31-55},
  year={1995},
  organization={Cambridge University Press}
}

@article{franklin1989scaling,
  title={On the scaling of multidimensional matrices},
  author={Franklin, Joel and Lorenz, Jens},
  journal={Linear Algebra and its applications},
  volume={114},
  pages={717-735},
  year={1989},
  publisher={Elsevier}
}

@article{idel2016review,
  title={A review of matrix scaling and Sinkhorn's normal form for matrices and positive maps},
  author={Idel, Martin},
  journal={arXiv preprint arXiv:1609.06349},
  year={2016}
}

@article{liang2026provably,
  title={A Provably Convergent and Practical Algorithm for {G}romov--{W}asserstein Optimal Transport},
  author={Liang, Ling and Yang, Lei},
  journal={arXiv preprint arXiv:2605.04175},
  year={2026}
}

@article{memoli2011gromov,
  title={{G}romov-{W}asserstein distances and the metric approach to object matching},
  author={M{\'e}moli, Facundo},
  journal={Foundations of computational mathematics},
  volume={11},
  number={4},
  pages={417-487},
  year={2011},
  publisher={Springer}
}

@inproceedings{peyre2016gromov,
  title={{G}romov-{W}asserstein averaging of kernel and distance matrices},
  author={Peyr{\'e}, Gabriel and Cuturi, Marco and Solomon, Justin},
  booktitle={International conference on machine learning},
  pages={2664-2672},
  year={2016},
  organization={PMLR}
}

@article{rioux2024entropic,
  title={Entropic {G}romov-{W}asserstein distances: Stability and algorithms},
  author={Rioux, Gabriel and Goldfeld, Ziv and Kato, Kengo},
  journal={Journal of Machine Learning Research},
  volume={25},
  number={363},
  pages={1-52},
  year={2024}
}

@article{sinkhorn1967concerning,
  title={Concerning nonnegative matrices and doubly stochastic matrices},
  author={Sinkhorn, Richard and Knopp, Paul},
  journal={Pacific Journal of Mathematics},
  volume={21},
  number={2},
  pages={343-348},
  year={1967},
  publisher={Mathematical Sciences Publishers}
}

@article{solomon2016entropic,
  title={Entropic metric alignment for correspondence problems},
  author={Solomon, Justin and Peyr{\'e}, Gabriel and Kim, Vladimir G and Sra, Suvrit},
  journal={ACM Transactions on Graphics (ToG)},
  volume={35},
  number={4},
  pages={1-13},
  year={2016},
  publisher={ACM New York, NY, USA}
}

@inproceedings{xu2019gromov,
  title={{G}romov-{W}asserstein learning for graph matching and node embedding},
  author={Xu, Hongteng and Luo, Dixin and Zha, Hongyuan and Duke, Lawrence Carin},
  booktitle={International conference on machine learning},
  pages={6932-6941},
  year={2019},
  organization={PMLR}
}

@inproceedings{yan2018semi,
  title={Semi-Supervised Optimal Transport for Heterogeneous Domain Adaptation.},
  author={Yan, Yuguang and Li, Wen and Wu, Hanrui and Min, Huaqing and Tan, Mingkui and Wu, Qingyao},
  booktitle={IJCAI},
  volume={7},
  pages={2969-2975},
  year={2018}
}

@article{aritake2022unsupervised,
  title={Unsupervised domain adaptation for extra features in the target domain using optimal transport},
  author={Aritake, Toshimitsu and Hino, Hideitsu},
  journal={Neural Computation},
  volume={34},
  number={12},
  pages={2432-2466},
  year={2022},
  publisher={MIT Press One Rogers Street, Cambridge, MA 02142-1209, USA journals-info~…}
}

@inproceedings{titouan2019optimal,
  title={Optimal transport for structured data with application on graphs},
  author={Vayer, Titouan and Courty, Nicolas and Tavenard, Romain and Flamary, R{\'e}mi},
  booktitle={International Conference on Machine Learning},
  pages={6275-6284},
  year={2019},
  organization={PMLR}
}

@article{flamary2021pot,
  title={{POT}: Python optimal transport},
  author={Flamary, R{\'e}mi and Courty, Nicolas and Gramfort, Alexandre and Alaya, Mokhtar Z and Boisbunon, Aur{\'e}lie and Chambon, Stanislas and Chapel, Laetitia and Corenflos, Adrien and Fatras, Kilian and Fournier, Nemo and others},
  journal={Journal of Machine Learning Research},
  volume={22},
  number={78},
  pages={1-8},
  year={2021}
}

@article{courty2017joint,
  title={Joint distribution optimal transportation for domain adaptation},
  author={Courty, Nicolas and Flamary, R{\'e}mi and Habrard, Amaury and Rakotomamonjy, Alain},
  journal={Advances in neural information processing systems},
  volume={30},
  year={2017}
}

@article{ng2001spectral,
  title={On spectral clustering: Analysis and an algorithm},
  author={Ng, Andrew and Jordan, Michael and Weiss, Yair},
  journal={Advances in neural information processing systems},
  volume={14},
  year={2001}
}

@article{morris2020tudataset,
  title={{TUD}ataset: A collection of benchmark datasets for learning with graphs},
  author={Morris, Christopher and Kriege, Nils M and Bause, Franka and Kersting, Kristian and Mutzel, Petra and Neumann, Marion},
  journal={arXiv preprint arXiv:2007.08663},
  year={2020}
}

@article{Borgwardt2005Protein,
  author    = {Borgwardt, Karsten M. and Ong, Cheng Soon and Sch{\"o}nauer, Stefan and Vishwanathan, S. V. N. and Smola, Alex J. and Kriegel, Hans-Peter},
  title     = {Protein Function Prediction via Graph Kernels},
  journal   = {Bioinformatics},
  volume    = {21},
  number    = {Suppl. 1},
  pages     = {i47-i56},
  year      = {2005},
  publisher = {Oxford University Press},
  doi       = {10.1093/bioinformatics/bti1007}
}

@article{dobson2003distinguishing,
  title={Distinguishing enzyme structures from non-enzymes without alignments},
  author={Dobson, Paul D and Doig, Andrew J},
  journal={Journal of molecular biology},
  volume={330},
  number={4},
  pages={771-783},
  year={2003},
  publisher={Elsevier}
}

@article{xu2019scalable,
  title={Scalable {G}romov-{W}asserstein learning for graph partitioning and matching},
  author={Xu, Hongteng and Luo, Dixin and Carin, Lawrence},
  journal={Advances in neural information processing systems},
  volume={32},
  year={2019}
}

@inproceedings{chowdhury2021generalized,
  title={Generalized spectral clustering via {G}romov-{W}asserstein learning},
  author={Chowdhury, Samir and Needham, Tom},
  booktitle={International Conference on Artificial Intelligence and Statistics},
  pages={712-720},
  year={2021},
  organization={PMLR}
}

@article{vogelstein2015fast,
  title={Fast approximate quadratic programming for graph matching},
  author={Vogelstein, Joshua T and Conroy, John M and Lyzinski, Vince and Podrazik, Louis J and Kratzer, Steven G and Harley, Eric T and Fishkind, Donniell E and Vogelstein, R Jacob and Priebe, Carey E},
  journal={PLOS one},
  volume={10},
  number={4},
  pages={e0121002},
  year={2015},
  publisher={Public Library of Science San Francisco, CA USA}
}

@article{kuhn1955hungarian,
  title={The Hungarian method for the assignment problem},
  author={Kuhn, Harold W},
  journal={Naval research logistics quarterly},
  volume={2},
  number={1-2},
  pages={83-97},
  year={1955},
  publisher={Wiley Online Library}
}

@inproceedings{memoli2009spectral,
  title={Spectral {G}romov-{W}asserstein distances for shape matching},
  author={M{\'e}moli, Facundo},
  booktitle={2009 IEEE 12th International Conference on Computer Vision Workshops, ICCV Workshops},
  pages={256-263},
  year={2009},
  organization={IEEE}
}

@inproceedings{memoli2004comparing,
  title={Comparing point clouds},
  author={M{\'e}moli, Facundo and Sapiro, Guillermo},
  booktitle={Proceedings of the 2004 Eurographics/ACM SIGGRAPH symposium on Geometry processing},
  pages={32-40},
  year={2004}
}

@article{chowdhury2019gromov,
  title={The {G}romov-{W}asserstein distance between networks and stable network invariants},
  author={Chowdhury, Samir and M{\'e}moli, Facundo},
  journal={Information and Inference: A Journal of the IMA},
  volume={8},
  number={4},
  pages={757-787},
  year={2019},
  publisher={Oxford University Press}
}

@article{har70,
  title={Validation of an activity type recognition model classifying daily physical behavior in older adults: the HAR70+ model},
  author={Ustad, Astrid and Logacjov, Aleksej and Trolleb{\o}, Stine {\O}verengen and Thingstad, Pernille and Vereijken, Beatrix and Bach, Kerstin and Maroni, Nina Skj{\ae}ret},
  journal={Sensors},
  volume={23},
  number={5},
  pages={2368},
  year={2023},
  publisher={MDPI}
}

@article{burke1988identification,
  title={On the identification of active constraints},
  author={Burke, James V and Mor{\'e}, Jorge J},
  journal={SIAM Journal on Numerical Analysis},
  volume={25},
  number={5},
  pages={1197--1211},
  year={1988},
  publisher={SIAM}
}

@article{lewis2011identifying,
  title={Identifying activity},
  author={Lewis, Adrian S and Wright, Stephen J},
  journal={SIAM Journal on Optimization},
  volume={21},
  number={2},
  pages={597--614},
  year={2011},
  publisher={SIAM}
}

@article{liang2017activity,
  title={Activity identification and local linear convergence of forward--backward-type methods},
  author={Liang, Jingwei and Fadili, Jalal and Peyr{\'e}, Gabriel},
  journal={SIAM Journal on Optimization},
  volume={27},
  number={1},
  pages={408--437},
  year={2017},
  publisher={SIAM}
}
\bibliographystyle{plainnat}

\appendix
\section{Proofs}
\label{app:proofs}

This appendix records the main technical steps omitted from the main text to establish Theorem~\ref{thm:stationarity}. The argument first relates the guarded BAPG cycle to exact Sinkhorn contraction, then derives feasibility, Lyapunov, residual, and stationarity statements. All asymptotic statements focus on the tail iteration and a finite prefix changes only constants.

\subsection{Projective contraction and Exact Sinkhorn}\label{app:proof_sinkhorn}
\begin{definition}[Hilbert's projective metric]
For $u,v\in \mathbb{R}_{++}^r$, define
\[d_H(u,v) := \log \frac{\max_i (u_i/v_i)}{\min_i (u_i/v_i)}.\]
\end{definition}

\begin{definition}[Projective diameter]
For $K\in \mathbb{R}_{++}^{r\times s}$, define
\[
\Delta(K) := \max_{i,i',j,j'} \log \frac{K_{ij}K_{i'j'}}{K_{ij'}K_{i'j}}.
\]
\end{definition}

\begin{theorem}[Birkhoff contraction theorem \cite{eveson1995elementary}] \label{theo:birk-contract}
For $K\in \mathbb{R}_{++}^{r\times s}$ and for all $u,v\in \mathbb{R}_{++}^s$,
\[
d_H(Ku,Kv) \leq \tanh\left(\frac{\Delta(K)}{4}\right)d_H(u,v).
\]
\end{theorem}

\begin{lemma}\label{lem:inv-hilb-proj}
For positive vectors $u,v\in \mathbb{R}_{++}^r$, $d_H(u^{-1},v^{-1})=d_H(u,v)$, where inversion is entrywise. 
Moreover, if $D_1$ and $D_2$ are positive diagonal matrices and $X$ is strictly positive, then $\Delta(D_1XD_2)=\Delta(X)$.
\end{lemma}
\begin{proof}
    We know that 
    \[\max_{i} \frac{u_i^{-1}}{v_i^{-1}} = \max_i \frac{v_i}{u_i} = \frac{1}{\min_i \frac{u_i}{v_i}},\]
    and similarly \[\min_i \frac{u_i^{-1}}{v_i^{-1}} = \frac{1}{\max_i \frac{u_i}{v_i}}.\] Thus, the quotient inside the logarithm defining $d_H$ remains unchanged. 
    Denote $(D_1XD_2)_{ij} = a_iX_{ij}b_j$ with $a_i, b_j > 0$. For every $(i, i^\prime, j, j^\prime)$,
    \[
    \frac{(D_1XD_2)_{ij}(D_1XD_2)_{i^\prime j^\prime}}{(D_1XD_2)_{ij^\prime}(D_1XD_2)_{i^\prime j}} 
    = \frac{a_iX_{ij}b_j a_{i^\prime}X_{i^\prime j^\prime}b_{j^\prime}}{a_iX_{ij^\prime}b_{j^\prime }a_{i^\prime}X_{i^\prime j}b_j} 
    = \frac{X_{ij}X_{i^\prime j^\prime}}{X_{ij^\prime}X_{i^\prime j}}.
    \]
    Thus, taking the maxima of logarithms gives $\Delta (D_1 X D_2) = \Delta(X)$.
\end{proof}

For a strictly positive matrix $X$, we define the exact KL projections onto the row and column marginal constraints by
$R(X) := \mathcal P^{\mathrm{KL}}_{\mathcal C_1}(X)$, 
$C(X) := \mathcal P^{\mathrm{KL}}_{\mathcal C_2}(X)$, and
$S(X) := C(R(X))$.
Thus $S$ denotes one exact Sinkhorn row-column scaling cycle. When $X\in \mathcal C_2$, define its row-ratio vector
$r(X) := \frac{1}{\mu}X\mathbf 1_m \in \mathbb R^n_{++}$ and $r_i(X)=\frac{1}{\mu_i}\sum_j X_{ij}$.
Then $r(X)=\mathbf 1_n$ if and only if $X\in \mathcal C_1$. We also define two row-stochastic matrices associated with $X$:
\[
(A_X)_{ij}:=\frac{X_{ij}}{\sum_{\ell}X_{i\ell}},
\qquad
(B_X)_{ji}:=\frac{X_{ij}}{\nu_j}.
\]
Since $X\in \mathcal C_2$, each column of $X$ sums to $\nu_j$ and $B_X$ is
row-stochastic.

\begin{lemma}
\label{lem:exact-row-ratio-dynamics}
Let $X\in \mathcal C_2$ be strictly positive. Then
$r(S(X)) = A_X\left(B_X(r(X)^{-1})\right)^{-1}$,
where inversion is entrywise.
\end{lemma}
\begin{proof}
    Recall that \[R(X)_{ij} = \frac{\mu_i}{\sum_{l}X_{il}}X_{ij} = X_{ij}/r_i.\] Denote $t_j = \sum_i X_{ij}/r_i$. \[\frac{t_j}{\nu_j} = \sum_{i} \frac{X_{ij}}{\nu_j} r_i^{-1} = \left(B_X(r^{-1})\right)_j.\] Then $S(X)_{ij} = \frac{\nu_j}{t_j} \frac{X_{ij}}{r_{i}}$. Since $X_{ij} = \mu_i r_i(A_X)_{ij}$,  
    \[
    \frac{(S(X)\mathbf 1_m)_i}{\mu_i} = \frac{1}{\mu_i r_i} \sum_j X_{ij}\frac{\nu_j}{t_j} = \sum_j (A_X)_{ij}\frac{\nu_j}{t_j} = \sum_j (A_X)_{ij}\left(B_X(r^{-1})\right)^{-1}_{j}.
    \]
    
\end{proof}

 From Lemma \ref{lem:exact-row-ratio-dynamics}, we can view one exact Sinkhorn cycle as two positive linear maps applied to the row-ratio vector. Hence, Theorem \ref{theo:birk-contract} gives the following estimate.

\begin{lemma}[Projective contraction of exact Sinkhorn]
\label{lem:sinkhorn-projective-contraction}
Let $X\in \mathcal C_2$ be strictly positive. Then
\[
d_H(r(S(X)),\mathbf 1_n)
\leq
\tanh^2\!\left(\frac{\Delta(X)}{4}\right)
d_H(r(X),\mathbf 1_n).
\]
\end{lemma}
\begin{proof}
    Since $A_X \mathbf 1_m = \mathbf 1_n$, $d_H(r(S(X)), \mathbf 1_n) = d_H\left(A_X\left(B_X(r^{-1})^{-1}\right), A_X(\mathbf 1_m)\right)$. By Lemma \ref{lem:inv-hilb-proj}, $d_H(B_X(r^{-1})^{-1}, \mathbf{1}_m) = d_H(B_X(r^{-1}), \mathbf{1}_m)$.
    By Theorem \ref{theo:birk-contract} and Lemma \ref{lem:inv-hilb-proj}, 
    \[
    d_H(r(S(X)), \mathbf 1_n) \leq \tanh \left(\frac{\Delta (A_X)}{4}\right)d_H(B_X(r^{-1}), \mathbf{1}_m) = \tanh \left(\frac{\Delta (X)}{4}\right)d_H(B_X(r^{-1}), \mathbf{1}_m) .
    \]
    Apply Theorem \ref{theo:birk-contract} and Lemma \ref{lem:inv-hilb-proj} again, 
    \[
    d_H(B_X(r^{-1}), \mathbf{1}_m) \leq \tanh \left(\frac{\Delta(B_X)}{4}\right)d_H(r^{-1}, \mathbf 1_n) = \tanh \left(\frac{\Delta(X)}{4}\right)d_H(r, \mathbf 1_n).
    \]
    Therefore, combining all together, we obtain
    \[
d_H(r(S(X)),\mathbf 1_n)
\leq
\tanh^2\!\left(\frac{\Delta(X)}{4}\right)
d_H(r(X),\mathbf 1_n).
\]
\end{proof}

\subsection{Controlled Projective Growth and Finite Backtracking}\label{app:proof_backtrack}
BAPG differs from the exact Sinkhorn only by the exponential perturbation applied before each normalization. In the following Lemma \ref{lem:projective-growth-cycle}, we control the projective diameter growth of one cycle.

\begin{lemma}
\label{lem:projective-growth-cycle}
For any accepted A-KL-BAPG step with
step size $\alpha>0$,
    $\Delta(w^{t+1})\leq \Delta(w)+B_\Delta\alpha$,
where $B_\Delta := 8M_G$ is a conservative projective-growth constant.
\end{lemma}
\begin{proof}
For one exponential perturbation $Z=X\odot\exp(\alpha H)$ with $\|H\|_\infty\leq M_G$, every log cross-ratio changes by $\alpha\bigl(H_{ij}+H_{i'j'}-H_{ij'}-H_{i'j}\bigr)$, whose absolute value is at most $4M_G\alpha$. Thus $\Delta(Z)\leq \Delta(X)+4M_G\alpha$. The row and column KL projections are positive diagonal scalings and therefore do not change cross-ratios.  A full A-KL-BAPG cycle contains two exponential perturbations, so
    \[\Delta(w^{t+1})\leq \Delta(w)+8M_G\alpha \leq \Delta(w)+B_\Delta\alpha.\]
\end{proof}

We record two elementary properties of the guarded tail. The first guarantees positivity throughout all cycles, and the second summarizes the resulting behavior of the backtracking rule.

\begin{lemma}
    If $\pi^0, w^0 > 0$ entriwise, then for every iterate, $\pi^t, w^t$ is entrywise positive.
\end{lemma}
\begin{proof}
    Assume $w^t > 0$ entrywise. Then $z_1^t = w^t \odot \exp(\alpha_t G(w^t))$ is positive entrywise. Since $\mu_i > 0$, \[\pi_{ij}^{t+1} = \frac{\mu_i z_{1, ij}^t}{\sum_{l}z_{1,il}^t}\] is entrywise positive, i.e., $\pi^{t+1} > \mathbf 0$. Similarly, $w^{t+1}_{ij} > 0$ and hence $w^{t+1} > \mathbf 0$. By induction, entrywise positivity holds for all iterates.
\end{proof}

\begin{lemma}
\label{lem:guarded-backtracking}
Assume that the current tail iterate satisfies $\Delta(w^t)\leq\Gamma_t$.  Then the backtracking loop in Algorithm~\ref{alg:adaptive-delayed-power} terminates finitely for each iteration.  Moreover, the accepted step satisfies $\ell_t\leq \alpha_t \leq \bar\alpha_t$ and $\alpha_t \leq \alpha_{t-1}$.
Since $\beta<1$, there exist constants $c_-,c_+>0$ and $T_0$ such that, for all $t\geq T_0$,
    \[\frac{c_-}{t+2}\leq \alpha_t
    \leq c_+(t+1)^{-\beta}.\]
\end{lemma}
\begin{proof}
    Define $s_t:= \max\{\ell_t, \min\{\bar{\alpha}_t, \alpha_{t-1}\}\}$.
    By the construction of $h_t$ and $\bar \alpha_t$, the sequences $h_t$ and $\bar\alpha_t$ are nonincreasing. Then the sequence of $\ell_t$ is also nonincreasing and so is the sequence of $s_t$. Since $\alpha_{-1}=+\infty$, $\alpha_{-1}\geq\ell_0$ for $t = 0$. 
    
    For arbitrary $t \geq 1$, assume $\alpha_{t-1}\geq \ell_{t-1}\geq \ell_t$. Then $\ell_{t} \leq \alpha_{t-1}$ and $\min\{\bar \alpha_{t}, \alpha_{t-1}\} \leq \alpha_{t-1}$, i.e., $s_t \leq \alpha_{t-1}$. We also know that $\ell_t \leq s_t \leq \bar \alpha_t$.
    If $s_t\leq h_t$, we have shown that at this time the safeguard is satisfied, and we can directly accept $a_t = s_t$. Otherwise, when $s_t > h_t$, we backtrack until we find the first candidate $\alpha_t$ such that $\alpha_t \leq h_t$. Then by construction, $\alpha_t > \tau h_t$. Since $\ell_t \leq \tau h_t$, it holds that $\alpha_t \geq \ell_t$. 
    Moreover, given that the sequence of $s_t$ decreases, we obtain
    \[
    \ell_t\leq\alpha_t\leq s_t
    \leq \bar\alpha_t,
    \qquad
    \alpha_t\leq s_t\leq \alpha_{t-1}.
    \]
    
    For the final bound, we know that $\bar\alpha_t\asymp (t+1)^{-\beta}$ and
    \[
        h_t=\frac{a}{B_\Delta}\log\left(1+\frac{1}{t+1}\right).
    \]
    Since $\beta<1$, eventually $\bar\alpha_t\geq h_t$, so $\ell_t=\tau h_t$ for all sufficiently large $t$.  Since \[\log\left(1+\frac{1}{t+1}\right)\geq \frac{1}{t+2}\] and $\alpha_t\leq\bar\alpha_t$, there exists $c_{-} = \tau a/B_{\Delta} > 0$,  $c_+> 0$ and $T_0$ such that for all $t \geq T_0$, 
    \[
    \frac{c_{-}}{t+2} \leq \alpha_t \leq c_+(t+1)^{-\beta}.
    \]
\end{proof}

\subsection{Marginal Feasibility Estimates and Cross-Divergence Bounds}
Fix $t$ and define the exact Sinkhorn cycle launched from $w^t$ by $\widehat \pi^{t+1}:=R(w^t)$ and
$\widehat w^{t+1}:=S(w^t)=C(\widehat \pi^{t+1})$. 
We compare the actual BAPG iterates $(\pi^{t+1},w^{t+1})$ with $(\widehat \pi^{t+1},\widehat w^{t+1})$.

\begin{lemma}
\label{lem:full-cycle-comparison}
For every $t$ and every $i,j$,
\[
\exp(-2M_G\alpha_t) \leq \frac{\pi^{t+1}_{ij}}{\widehat \pi^{t+1}_{ij}} \leq \exp(2M_G\alpha_t), 
\qquad \exp(-6M_G\alpha_t) \leq
\frac{w^{t+1}_{ij}}{\widehat w^{t+1}_{ij}} \leq \exp(6M_G\alpha_t).
\]
Consequently,
$d_H\!\left(r(w^{t+1}),r(\widehat w^{t+1})\right) \leq 12M_G\alpha_t$.
\end{lemma}
\begin{proof} We know that 
    \[
    \frac{\pi^{t+1}_{ij}}{\widehat \pi^{t+1}_{ij}}  = \exp(\alpha_t G(w^t)_{ij})\frac{\sum_{l}w_{il}^t}{\sum_{\ell} w_{il}^t \exp(\alpha_t G(w^t)_{il})},
    \]
    with \[\exp(-M_G \alpha_t) \sum_{\ell} w_{il}^t \leq \sum_{\ell} w_{il}^t \exp(\alpha_t G(w^t)_{il}) \leq \exp(M_G \alpha_t)\sum_{\ell} w_{il}^t.\] 
    Thus \[\exp(-2M_G\alpha_t) \leq \frac{\pi^{t+1}_{ij}}{\widehat \pi^{t+1}_{ij}} \leq \exp(2M_G\alpha_t).\] 
    Recall that $z_2^t = \pi^{t+1}\odot \exp(\alpha_t G(\pi^{t+1}))$ and $\widehat w^{t+1} = \mathcal P_{\mathcal C_2}^{\mathrm{KL}}(\widehat \pi^{t+1})$, $w^{t+1} = \mathcal P_{\mathcal C_2}^{\mathrm{KL}}(z_2^t)$. Hence
    \[
    \frac{w^{t+1}_{ij}}{\widehat w^{t+1}_{ij}} = \frac{z_{2,ij}^t / \sum_\ell z_{2,\ell j}^t}{\widehat \pi^{t+1}_{ij} / \sum_\ell \widehat \pi_{\ell j}^{t+1}} = \frac{z_{2,ij}^t}{\widehat \pi^{t+1}_{ij}} \cdot \frac{\sum_\ell \widehat \pi_{\ell j}^{t+1}}{\sum_\ell z_{2,\ell j}^t},
    \quad \text{and} \quad
    \frac{z_{2,ij}^t}{\widehat \pi^{t+1}_{ij}} = \frac{\pi^{t+1}_{ij} \exp(\alpha_t G(\pi^{t+1})_{ij})}{\widehat \pi^{t+1}_{ij}} \in [e^{-3M_G\alpha_t}, e^{3M_G\alpha_t}].
    \]
    Summing over each column, we obtain
\[
\frac{\sum_\ell \pi^{t+1}_{\ell j} \exp(\alpha_t G(\pi^{t+1})_{\ell j})}{\sum_\ell \widehat \pi^{t+1}_{\ell j}} \in [e^{-3M_G\alpha_t}, e^{3M_G\alpha_t}] \quad \text{and} \quad \frac{\sum_\ell \widehat \pi^{t+1}_{\ell j}}{\sum_\ell \pi^{t+1}_{\ell j} \exp(\alpha_t G(\pi^{t+1})_{\ell j})} \in [e^{-3M_G\alpha_t}, e^{3M_G\alpha_t}].
\]
Putting these together, we obtain
\[\frac{w^{t+1}_{ij}}{\widehat w^{t+1}_{ij}} \in [e^{-3M_G\alpha_t} \cdot e^{-3M_G\alpha_t}, e^{3M_G\alpha_t} \cdot e^{3M_G\alpha_t}] = [e^{-6M_G\alpha_t}, e^{6M_G\alpha_t}].\]
\end{proof}
With Lemma \ref{lem:full-cycle-comparison}, we can derive a polynomial bound for the projective row error. Define $p:= \frac{1}{2}a$.

\begin{lemma} \label{lem:polynomial-row-error}
Assume $\beta>p$. Then there exist constants $\widetilde C_\delta>0$ and $K_\delta\geq0$ such that, for every $t\geq K_\delta$,
    $d_H(r(w^t),\mathbf 1_n)\leq \widetilde C_\delta \eta_t$.
\end{lemma}
\begin{proof}
 By Lemma \ref{lem:full-cycle-comparison},
    $d_H(r(w^{t+1}),r(\widehat w^{t+1}))
    \leq 12M_G\alpha_t
    \leq C_\alpha(t+1)^{-\beta}$
for some constant $C_\alpha>0$ and all sufficiently large $t$.
By the accepted diameter guard,
$\Delta(w^t)\leq\Delta_0+a\log(t+1)$ and hence
\[
    1-\tanh^2\!\left(\frac{\Delta(w^t)}{4}\right)
    =\operatorname{sech}^2\!\left(\frac{\Delta(w^t)}{4}\right) \geq \exp(-\frac{\Delta(w^t)}{2})
    \geq c_\Delta(t+1)^{-p}
\]
for some $c_\Delta>0$.
Then \[\tanh^2\left(\frac{\Delta(w^t)}{4}\right) \leq 1 - c_\Delta(t+1)^{-p}.\] By Lemma \ref{lem:sinkhorn-projective-contraction}, 
\[d_H(r(\widehat w^{t+1}), \mathbf 1_n) =  d_H(r(S(w^{t})), \mathbf 1_n)\leq \tanh \left(\frac{\Delta(w^t)}{4}\right)^2 d_H(r(w^t), \mathbf 1_n) \leq \left(1 - c_\Delta(t+1)^{-p}\right) d_H(r(w^t), \mathbf 1_n) .\] 
By triangle inequality, $d_H(r(w^{t+1}), \mathbf 1_n) \leq d_H(r(w^{t+1}), r(\widehat w^{t+1})) + d_H(r(\widehat w^{t+1}), \mathbf 1_n)$. Then 
\[d_H(r(w^{t+1}), \mathbf 1_n) \leq \left(1 - c_\Delta(t+1)^{-p}\right)d_H(r(w^t), \mathbf 1_n) + C_\alpha (t+1)^{-\beta}.\] 

Let $s:=\beta-p>0$. If $d_H(r(w^t), \mathbf 1_n)\leq K(t+1)^{-s}$, then
    $\left(1-c_\Delta(t+1)^{-p}\right)d_H(r(w^t), \mathbf 1_n)
    \leq
    K(t+1)^{-s}-c_\Delta K(t+1)^{-(p+s)}$.
Since $p+s=\beta$, the recurrence gives
\[
    d_H(r(w^{t+1}), \mathbf 1_n)
    \leq K(t+1)^{-s}-(c_\Delta K-C_\alpha)(t+1)^{-\beta}.
\]
Choose $K$ sufficiently large such that $c_\Delta K>C_\alpha$. Since $p = a/2 < 1$, $\beta=p+s<s+1$ and the correction term
$(t+1)^{-\beta}$ decays more slowly than $(t+2)^{-s-1}$, which is the size of $(t+1)^{-s} - (t+2)^{-s}$.
Thus the implication
    \[d_H(r(w^{t}), \mathbf 1_n)\leq K(t+1)^{-s}
    \Rightarrow
    d_H(r(w^{t+1}), \mathbf 1_n)\leq K(t+2)^{-s},\]
holds for all sufficiently large $t$. Enlarging $K$ to absorb the finitely many earlier indices gives the claim with $\widetilde C_\delta=K$.
\end{proof}

Define the row sums of $z_1^t$ by
$r_i^t:=(z_1^t\mathbf 1_m)_i =\sum_{j=1}^m z_{1,ij}^t$.
Since $\pi^{t+1}$ is the KL projection of $z_1^t$ onto the row-marginal constraint set $\mathcal C_1$, we have the explicit row-normalization 
$\pi_{ij}^{t+1} = (\mu_i/r_i^t)z_{1,ij}^t$. 

The second half-step is analogous except that now the projection is onto the column marginal constraint set $\mathcal C_2$.
Define the column sums of $z_2^t$ by
$c_j^t:=((z_2^t)^\top\mathbf 1_n)_j = \sum_{i=1}^n z_{2,ij}^t$.
Equivalently, $c_j^t=\sum_i z_{2,ij}^t$.
Since $w^{t+1}$ is the KL projection of $z_2^t$ onto $\mathcal C_2$,  $w_{ij}^{t+1} = (\nu_j/c_j^t)z_{2,ij}^t$.

\begin{lemma}[Exact Row/Column Projection Identities]
\label{lem:exact-projection-identities}
For every $t$, for row projection, the following identities hold: $\|\pi^{t+1}-z_1^t\|_1 = \|z_1^t\mathbf 1_m-\mu\|_1$,
\[
    D_h(\pi^{t+1},z_1^t) =
    \sum_{i=1}^n \left[
        \mu_i\log\frac{\mu_i}{r_i^t}-
        \mu_i+r_i^t \right],
\quad \text{and} \quad
    D_h(z_1^t,\pi^{t+1})=
    \sum_{i=1}^n\left[
        r_i^t\log\frac{r_i^t}{\mu_i}- r_i^t+\mu_i
    \right].
\]
Similarly, for column projection,
$\|w^{t+1}-z_2^t\|_1 = \|(z_2^t)^\top\mathbf 1_n-\nu\|_1$,
\[
    D_h(w^{t+1},z_2^t)  = \sum_{j=1}^m
    \left[
        \nu_j\log\frac{\nu_j}{c_j^t} - \nu_j + c_j^t
    \right],
    \quad    \text{and}  \quad 
    D_h(z_2^t,w^{t+1})  =
    \sum_{j=1}^m \left[
        c_j^t\log\frac{c_j^t}{\nu_j} - c_j^t + \nu_j
    \right].
\]
\end{lemma}
\begin{proof}
    We know that $\sum_{j=1}^m |\pi^{t+1}_{ij} - z_{1, ij}^{t}| = \left|\frac{\mu_i}{r_i^t}-1\right|\sum_{j = 1}^m z_{1, ij}^{t} = |\mu_i - r_i^t|$. Summing over $i$, we obtain $\|\pi^{t+1}-z_1^t\|_1 = \|z_1^t\mathbf 1_m-\mu\|_1$. 
    \begin{multline*}
        D_h(\pi^{t+1},z_1^t) = \sum_{i,j} \left(\pi^{t+1}_{ij} \log \frac{\pi^{t+1}_{ij}}{z_{1,ij}^t} - \pi^{t+1}_{ij} + z_{1,ij}^t\right) \\
    = \sum_{ij} z_{1,ij}^t\left(\frac{\mu_i}{r_i^t}\log \frac{\mu_i}{r_i^t} - \frac{\mu_i}{r_i^t} + 1\right)
    = \sum_{i = 1}^n \left[\mu_i \log \frac{\mu_i}{r_i^t} - \mu_i + r_i^t\right].
    \end{multline*}
The derivation of formula for $D_h(z_1^t , \pi^{t+1})$ is identical. The proof for column projection identities is analogous to the above proof for row projection identities.
\end{proof}

\paragraph{Proof for Proposition \ref{prop:automatic-marginal-feasibility}}
\begin{proof}
 For each $t$, define the row ratio vector \[\theta_i^t := \frac{(w^t \mathbf 1_m)_i}{\mu_i}.\]
    Since $(w^t)^\top \mathbf 1_n = \nu$, \[\sum_{i} (w^t \mathbf 1_m)_i = \sum_{i,j}w^{t}_{ij} = \sum_j \nu_j = 1\] and $\sum_i \theta_i^t \mu_i = \sum_{i} (w^t \mathbf 1_m)_i = 1$. Since $\mu_i > 0$ and $\sum_i \mu_i = 1$, $\min_i (\theta_i^t) \leq 1 \leq \max_i (\theta_i^t)$. Then for every $i \in [n]$, $\exp(-d_H(\theta^t, \mathbf 1_n)) \leq \theta_i^t \leq \exp(d_H(\theta^t, \mathbf 1_n))$ and \[\|w^t \mathbf 1_m - \mu\|_1 = \sum_i \mu_i|\theta_i^t - 1| \leq \exp(d_H(\theta^t, \mathbf 1_n)) - 1.\] 
    Define
   $K_{\mathrm{row}}:=
    \min\left\{ t\geq K_\delta:\widetilde C_\delta\eta_t\leq 1
    \right\}$.
    Then for every $t \geq K_{\operatorname{row}}$, $0 \leq d_H(r(w^t), \mathbf 1_n) \leq \widetilde C_\delta\eta_t \leq 1$.  Define $C_{\mathrm{row}}:=e\widetilde C_\delta$. By the mean value theorem, $\exp(d_H(r(w^t), \mathbf 1_n)) - 1\leq ed_H(r(w^t), \mathbf 1_n)$ and thus
    $\|w^t \mathbf 1_m - \mu\|_1 \leq e \widetilde C_\delta \eta_t = C_{\operatorname{row}}\eta_t$.
    Define $C_z:=\exp({M_G c_+})M_Gc_+$. Since $|G_{ij}(w^t)| \leq M_G$ and $\alpha_t\leq c_+(t+1)^{-\beta}\leq c_+$ from Lemma \ref{lem:guarded-backtracking}, the mean value theorem gives 
    \[|\exp(\alpha_t G_{ij}(w^t)) - 1| \leq e^{M_Gc_+} M_G\alpha_t \leq C_z(t+1)^{-\beta}.\] 
    
    Since $z_1^t - w^t = w^t \odot (\exp(\alpha_t G(w^t)) - 1)$, $\sum_{i,j} w_{ij}^t = 1$ and $p\geq0$, $\|z_1^t - w^t\|_1 \leq C_z(t+1)^{-\beta} \leq C_z\eta_t$. 
    Define $ C_{\mathrm{fea}}:=\exp({M_G c_+})M_Gc_++e\widetilde C_\delta$. Then 
    \[\|z_1^t\mathbf 1_m - \mu\|_1 \leq \|z_1^t - w^t\|_1 + \|w^t \mathbf 1_m - \mu\|_1 \leq (C_{\operatorname{row}} + C_z)\eta_t = C_{\operatorname{fea}}\eta_t.\]
    Since $w^t \in \mathcal C_2$, by Lemma \ref{lem:exact-projection-identities}, $\|\pi^{t+1} - z^{t}_1\|_1 = \|z_1^t \mathbf 1_m - \mu\|_1$ and
    \begin{multline*}
        \|(\pi^{t+1})^\top \mathbf 1_n - \nu\|_1 \leq \|\pi^{t+1} - w^t\|_1 
        \leq \|\pi^{t+1} - z_1^t\|_1 + \|z_1^t - w^t\|_1 
        \leq \|z_1^t\mathbf 1_m - \mu\|_1 + \|z_1^t - w^t\|_1 
        \leq 2C_{\operatorname{fea}}\eta_t.
    \end{multline*}

We hence show that $d_H(r(w^t), \mathbf 1_n) = O\left(\eta_t\right)$, $\|z_1^t\mathbf 1_m - \mu\|_1 = O\left(\eta_t\right)$, and $\|(\pi^{t+1})^\top \mathbf 1_n - \nu\|_1 = O\left(\eta_t\right)$.

Now define
    $K_{\mathrm{cross},1} :=
    \min\left\{
        t\geq K_{\mathrm{row}}:
        C_{\mathrm{fea}}\eta_t \leq \frac{\mu_{\min}}{2}
    \right\}$.
Then for every $t\geq K_{\mathrm{cross},1}$,
    \[|r_i^t-\mu_i|\leq
    C_{\mathrm{fea}}\eta_t \leq
    \frac{\mu_{\min}}{2},\] where $\mu_{\min}:=\min_i \mu_i$. Hence for every $i$,
   $ r_i^t\in\left[\mu_i/2,3\mu_i/2\right]$.
    Since $z_1^t = w^t\odot \exp(\alpha_t G(w^t))$, $\log z_1^t - \log w^t = \alpha_t G(w^t)$. Since $h$ is strictly convex, by the three-point identity of Bregman divergence,  
    \begin{multline*}
    D_h(\pi^{t+1}, w^t) = D_h(\pi^{t+1}, z_1^t) + D_h(z_1^t, w^t) + \langle \nabla h(z_1^t) - \nabla h(w^t), \pi^{t+1} - z_1^t\rangle\\
    = D_h(\pi^{t+1}, z_1^t) + D_h(z_1^t, w^t) + \alpha_t \langle G(w^t), \pi^{t+1} - z_1^t\rangle.
    \end{multline*}

    By Lemma \ref{lem:exact-projection-identities}, for $t \geq K_{\operatorname{cross},1}$,
    \[|\mu_i - r_i^t| \leq \|\pi^{t+1} - z_1^t\|_1 = \|z_1^t \mathbf 1_m - \mu\|_1 \leq C_{\operatorname{fea}}\eta_t \leq \frac{\mu_{\min}}{2}.\] To satisfy this, we need $r_i^t \in \left[\frac{\mu_i}{2}, \frac{3\mu_i}{2}\right]$ for every $i \in [n]$. Define $\phi_i(s):= \mu_i\log\frac{\mu_i}{s}-\mu_i+s$. By Lemma \ref{lem:exact-projection-identities}, $D_h(\pi^{t+1}, z_1^t) =  \sum_i \phi_i(r_i^t)$. Since $\phi_i(\mu_i)=0$, $\phi_i^\prime(\mu_i)=0$, and $\phi_i^{\prime \prime}(s)=\frac{\mu_i}{s^2}$, for sufficiently large $t$, 
    \[\phi_i^{\prime \prime}(s) = \frac{\mu_i}{s^2} \le \frac{4}{\mu_i} \le \frac{4}{\mu_{\min}},\] 
    for all $s \in \left[\frac{\mu_i}{2}, \frac{3\mu_i}{2}\right]$. 
    By Taylor’s theorem, around $s = \mu_i$, $\phi_i(r_i^t) \le \frac{2}{\mu_{\min}}(r_i^t-\mu_i)^2$. Summing over $i \in [n]$, we obtain
    \[
    D_h(\pi^{t+1},z_1^t) 
    \le \frac{2}{\mu_{\min}}\|z_1^t\mathbf 1_m-\mu\|_2^2
    \le \frac{2}{\mu_{\min}} \|z_1^t\mathbf 1_m-\mu\|_1^2
    \le \frac{2C_{\rm fea}^2}{\mu_{\min}}\eta_t^2.
    \]
    Define $\xi_{ij}^t:=\alpha_t G_{ij}(w^t)$. Then 
    $z_{1,ij}^t = w_{ij}^t \exp({\xi_{ij}^t})$ and 
    \[D_h(z_1^t,w^t) = \sum_{i,j} w_{ij}^t
    \left[\exp({\xi_{ij}^t})\xi_{ij}^t - \exp({\xi_{ij}^t}) + 1 \right]
    = \sum_{i,j}w_{ij}^t\psi(\xi_{ij}^t),\] where $\psi(s):=\exp(s)(s-1)+1$. Since $\psi(0)=0$, $\psi^\prime(0)=0$, and $\psi^{\prime \prime}(s)=\exp(s)(1+s)$. By Taylor's theorem, for $|s| \leq M_Gc_+$, $\psi(s) \leq \frac{1}{2}\exp({M_Gc_+})(1+M_Gc_+)s^2$. Then 
    \[
    D_h(z_1^t, w^t) \leq \frac{1}{2}\exp(M_Gc_+)(1+M_G c_+) \sum_{ij}w_{ij}^t \left(\alpha_t G_{ij}(w^t)\right)^2 \leq \frac{1}{2}\exp(M_G c_+)(1+M_Gc_+)\alpha_t^2 M_G^2\sum_{ij} w_{ij}^t. 
    \]

    Since $w^t \in \mathcal C_2$, $\sum_{ij} w_{ij}^t = \sum_j \nu_j = 1$. Then $D_h(z_1^t, w^t) = O(\alpha_t^2)=O(\eta_t^2)$.

    Finally, by Lemma \ref{lem:exact-projection-identities},  $\|\pi^{t+1} - z_1^t\|_1 = \|z_1^t \mathbf 1_m - \mu\|_1$. Then 
    \[|\alpha_t\langle G(w^t), \pi^{t+1} - z_1^t\rangle| 
    \leq \alpha_t M_G \|z_1^t \mathbf 1_m - \mu\|_1 \leq M_G C_{\operatorname{fea}}\alpha_t\eta_t = O(\eta_t^2).\]
    Combining all three terms together, we have 
    \[
    D_h(\pi^{t+1},w^t)
    \le \frac{2C_{\rm fea}^2}{\mu_{\min}}\eta_t^2
        + O(\alpha_t^2)
        + O(\alpha_t\eta_t).
    \]
    Since $\alpha_t=O((t+1)^{-\beta})=O(\eta_t)$, there exists a constant $C_{1, \operatorname{cross}}$ such that the whole right-hand side is bounded by $C_{1, \operatorname{cross}}\eta_t^2$, i.e., $D_{h}(\pi^{t+1}, w^t) \leq C_{1, \operatorname{cross}}\eta_t^2$.

    The estimate for $D_h(\pi^{t+1},w^{t+1})$ is obtained analogously from the second half-step. Define $u_j^t:=((z_2^t)^\top \mathbf 1_n)_j$. Since $w^{t+1}=\mathcal P^{\rm KL}_{\mathcal C_2}(z_2^t)$,  $w_{ij}^{t+1} = \frac{\nu_j}{u_j^t}z_{2,ij}^t$ and $\log z_{2,ij}^t - \log w_{ij}^{t+1} = \log \frac{u_j^t}{\nu_j}$. Again by the three-point identity for Bregman divergences,
    \[D_h(\pi^{t+1},w^{t+1}) = D_h(\pi^{t+1}, z_2^{t}) + D_h(z_2^t, w^{t+1}) + \langle \log z_2^t - \log w^{t+1}, \pi^{t+1} - z_2^t \rangle.\]
    Since $z_2^t=\pi^{t+1}\odot \exp(\alpha_tG(\pi^{t+1}))$ and $|G(\pi^{t+1})|_\infty\le M_G$, the same Taylor argument used for $D_h(z_1^t,w^t)$ gives $D_h(\pi^{t+1}, z_2^{t}) = O(\alpha_t^2)=O(\eta_t^2)$. We know that \[|u^t-\nu|_1=|(z_2^t)^\top\mathbf 1_n-\nu|_1\le |(\pi^{t+1})^\top\mathbf 1_n-\nu|_1+|z_2^t-\pi^{t+1}|_1.\] Since $|(\pi^{t+1})^\top\mathbf 1_n-\nu|_1\le 2C_{\operatorname{fea}}\eta_t$, by the mean-value theorem,
   $|z_2^t-\pi^{t+1}|_1=O(\alpha_t)=O(\eta_t).$
   Therefore, there exists a constant $C_{\operatorname{col}}>0$ such that $|u^t-\nu|_1\le C_{\operatorname{col}}\eta_t$. For sufficiently large $t$, this implies $u_j^t\in [\nu_j/2,3\nu_j/2]$ for every $j$. By Lemma \ref{lem:exact-projection-identities} and Taylor's theorem, we obtain
   \[D_h(z_2^t,w^{t+1})\le \frac{1}{\nu_{\min}}|u^t-\nu|_1^2\le \frac{C_{\operatorname{col}}^2}{\nu_{\min}}\eta_t^2.\]
   Since $u_j^t\in [\nu_j/2,3\nu_j/2]$ for sufficiently large $t$, the mean-value theorem gives 
   \[\left|\log \frac{u_j^t}{\nu_j}\right|\le \frac{2|u_j^t-\nu_j|}{\nu_j}\le \frac{2|u^t-\nu|_1}{\nu_{\min}}.\] 
   Hence $|\log z_2^t-\log w^{t+1}|_\infty = O(\eta_t)$. Also, $|\pi^{t+1}-z_2^t|_1=O(\alpha_t)=O(\eta_t)$. Thus
   \[|\langle \log z_2^t-\log w^{t+1},\pi^{t+1}-z_2^t\rangle| = O(\eta_t^2).\]
   Combining the three estimates, we show that there exist an integer $K_{\operatorname{cross},2}\ge K_{\operatorname{cross},1}$ and a constant $C_{2,\operatorname{cross}}>0$ such that, for every $t\ge K_{\operatorname{cross},2}$,
   $D_h(\pi^{t+1},w^{t+1})\le C_{2,\operatorname{cross}} \eta_t^2.$ Thus  $D_{h}(\pi^{t+1}, w^t) = O(\eta_t^2)$ and $D_h(\pi^{t+1},w^{t+1}) = O(\eta_t^2)$.
\end{proof}

\paragraph{Proof for Proposition \ref{prop:cubic-main}}
\begin{proof} For notation convenience, we define the row error and first-order row moments: 
    $s_i^t := \sum_j w_{ij}^t$,
    $ e_i^t:=s_i^t-\mu_i$,
\[a_i^t:=\sum_j w^t_{ij}G_{ij}(w^t), \qquad {\text{and}} \qquad 
b_i^t:=\sum_j w^t_{ij}\exp(\alpha_t G_{ij}(w^t))G_{ij}(w^t).\] Recall that $r_i^t = \sum_j z_{1, ij}^t = \sum_j w_{ij}^t \exp(\alpha_t G_{ij}(w^t))$ and $\pi^{t+1}_{ij} = \frac{\mu_i}{r_i^t} z_{1, ij}^t$. Then for every $t$, by the definition of KL divergence, 
\begin{align*}
    A^{\mathrm{row}}_t  
     = \sum_{i,j} \left[(\pi^{t+1}_{ij} + w^{t}_{ij})\log \frac{\pi^{t+1}_{ij}}{w^t_{ij}} + 2(w^{t}_{ij} - \pi^{t+1}_{ij})\right]
     = \sum_{i,j} \left[(\pi^{t+1}_{ij} + w_{ij}^t) \left(\log \frac{\mu_i}{r_i^t} + \alpha_t G_{ij}(w^t)\right) + 2(w^t_{ij} - \pi_{ij}^{t+1})\right].
\end{align*}
Since $\sum_{j} \pi_{ij}^{t+1} = \mu_i$ and $\sum_j w_{ij}^t = s_i^t$, 
\[\sum_{j} (\pi^{t+1}_{ij} + w_{ij}^t) \log \frac{\mu_i}{r_i^t} 
= (\mu_i + s_i^t)\log \frac{\mu_i}{r_i^t},\] where $\mu_i + s_i^t = 2\mu_i + e_i^t$. 
Also by the definition above, 
\[\alpha_t \sum_j (\pi^{t+1}_{ij} + w_{ij}^t)G_{ij}(w^t)
= \alpha_t \left(\frac{\mu_i}{r_i^t}b_i^t + a_i^t\right).\] Since $\sum_j w_{ij}^t - \pi_{ij}^{t+1} = s_i^t - \mu_i = e_i^t$, summing over $i$,
\begin{equation}\label{eq:asy-row}
    A^{\mathrm{row}}_t = \sum_i
    \left[
        (2\mu_i+e_i^t)\log\frac{\mu_i}{r_i^t} +
        \alpha_t\left(\frac{\mu_i b_i^t}{r_i^t} + a_i^t
        \right)+ 2e_i^t
    \right].
\end{equation}
Define \[\zeta_i^t := (2\mu_i+e_i^t)\log\frac{\mu_i}{r_i^t} + \alpha_t\left(\frac{\mu_i b_i^t}{r_i^t} + a_i^t\right)+ 2e_i^t.\]
From \eqref{eq:asy-row}, $A_t^{\rm row} = \sum_i \zeta_i^t.$
Define $u_i^t:=\sum_j w^t_{ij}G_{ij}(w^t)^2$.
Since $\|G(w^t)\|_\infty\leq M_G$ and $\sum_{j} w_{ij}^t \leq 1$, the quantities $a_i^t$, $b_i^t$, and $u_i^t$ are uniformly bounded in
absolute value by constants depending only on $M_G$.

From Taylor expanding the exponential, \[\exp(\alpha_t G_{ij}(w^t)) = 1+ \alpha_t G_{ij}(w^t) + \frac{\alpha_t^2}{2}G_{ij}(w^t)^2 + O(\alpha_t^3).\] Then the perturbed row sum admits the expansion
\[
    r_i^t = \sum_j w^t_{ij} \exp(\alpha_t G_{ij}(w^t))
    = s_i^t+\alpha_t a_i^t+\frac{\alpha_t^2}{2}u_i^t+O(\alpha_t^3) = \mu_i + e_i^t + \alpha_t a_i^t + \frac{\alpha_t^2}{2}u_i^t+O(\alpha_t^3).
\]
Similarly, 
\[b_i^t = \sum_j w_{ij}^t\exp(\alpha_tG_{ij}(w^t))G_{ij}(w^t)= a_i^t+\alpha_t u_i^t+O(\alpha_t^2).\]
Set
    $\lambda_i^t:=r_i^t-\mu_i$.
By the boundedness of $a_i^t, u_i^t$ and diminishing $\alpha_t$,
 \[\lambda_i^t =
e_i^t+\alpha_t a_i^t+\frac{\alpha_t^2}{2}u_i^t+O(\alpha_t^3) = O(|e_i^t|+\alpha_t).\] For sufficiently large $t$, $|\lambda_i^t|$ is small enough to Taylor expand around $\mu_i$. Thus
 by Taylor expansion, 
\[
\log\frac{\mu_i}{r_i^t} = \log \frac{\mu_i}{\mu_i + \lambda_i^t }
= - \frac{\lambda_i^t }{\mu_i} + \frac{{\lambda_i^t }^2}{2\mu_i^2} + O\left({\lambda_i^t }^3\right),
\qquad \frac{\mu_i}{r_i^t} = \frac{\mu_i}{\mu_i + \lambda_i^t } = 1 - \frac{\lambda_i^t }{\mu_i} + O\left({\lambda_i^t }^2\right).\] 
Since $\lambda_i^t=O(|e_i^t|+\alpha_t)$, 
\begin{align*}
    \zeta_i^t & = (2\mu_i+e_i^t)\log\frac{\mu_i}{r_i^t} +
\alpha_t\left(\frac{\mu_i b_i^t}{r_i^t} + a_i^t
\right)+ 2e_i^t \\
& =  (2\mu_i+e_i^t)
    \left(
        -\frac{\lambda_i^t}{\mu_i}   +\frac{(\lambda_i^t)^2}{2\mu_i^2}
    \right) 
    + \alpha_t
    \left[
        \left(a_i^t+\alpha_t u_i^t\right)
        \left(  1-\frac{\lambda_i^t}{\mu_i} \right)
        +a_i^t
    \right]
    +2e_i^t +O((|e_i^t|+\alpha_t)^3)\\
    &=-2\lambda_i^t -\frac{e_i^t\lambda_i^t}{\mu_i} +\frac{(\lambda_i^t)^2}{\mu_i}
    +2\alpha_t a_i^t  +\alpha_t^2u_i^t
    -\frac{\alpha_t a_i^t\lambda_i^t}{\mu_i}
    +2e_i^t +O((|e_i^t|+\alpha_t)^3)\\
    & = \left(
        -2\lambda_i^t+2e_i^t  +2\alpha_t a_i^t +\alpha_t^2u_i^t
    \right)
    + \frac{\lambda_i^t(\lambda_i^t-e_i^t-\alpha_t a_i^t)}{\mu_i}
    +O((|e_i^t|+\alpha_t)^3),
\end{align*}
where the first bracket satisfies 
    \[-2\lambda_i^t +2e_i^t +2\alpha_t a_i^t+\alpha_t^2u_i^t
    =-2\left(
        e_i^t+\alpha_t a_i^t+\frac{\alpha_t^2}{2}u_i^t+O(\alpha_t^3)
    \right) 
    +2e_i^t +2\alpha_t a_i^t +\alpha_t^2u_i^t = 
    O(\alpha_t^3).\]
Since $\lambda_i^t-e_i^t-\alpha_t a_i^t
    = \frac{\alpha_t^2}{2}u_i^t+O(\alpha_t^3)
    = O(\alpha_t^2)$ and $\lambda_i^t=O(|e_i^t|+\alpha_t)$,
we have
\[
    \frac{\lambda_i^t(\lambda_i^t-e_i^t-\alpha_t a_i^t)}{\mu_i}
    = O((|e_i^t|+\alpha_t)\alpha_t^2)
    =  O((|e_i^t|+\alpha_t)^3).
\]
Therefore,
$\zeta_i^t =O((|e_i^t|+\alpha_t)^3)$.
By Proposition \ref{prop:automatic-marginal-feasibility},  there exist constants $C_{\mathrm{row}}>0$ and
$K_{\mathrm{row}}\geq 0$ such that, for all $t\geq K_{\mathrm{row}}$,
    $\|w^t\mathbf 1_m-\mu\|_1
    \leq
    C_{\mathrm{row}}\eta_t$.
Since $e_i^t=s_i^t-\mu_i$, this implies
    $|e_i^t|
    \leq
    C_{\mathrm{row}}\eta_t$.
Also, because $\alpha_t\leq c_+(t+1)^{-\beta}$ and $p\geq0$,
    $\alpha_t
    \leq
    c_+\eta_t$. 
Hence
    $|\zeta_i^t|
    \leq C_{\rm row}\eta_t^3$.
Summing over finitely many rows gives
$|A^{\rm row}_t| \leq C_{\rm row,asym}\eta_t^3$,
for $C_{\rm row,asym} > 0$.  

Analogously, we can repeat the preceding argument with columns, with $\nu_j$ in place of $\mu_i$, and with the column infeasibility estimate from Proposition \ref{prop:automatic-marginal-feasibility}. Thus for all sufficiently large $t$, we can show there exists some $C_{\rm col,asym} > 0$ such that 
$|A^{\mathrm{col}}_t| \leq C_{\mathrm{col,asym}}\eta_t^3$.
Summing the row and column bounds yields the 
$|A_t| \leq C_{\mathrm{asym}}\eta_t^3$,
for constant $C_{\rm asym}> 0$ and sufficiently large $t$. We thus show that $|A_t| = O\left((t+1)^{-3(\beta - p)}\right)$.
\end{proof}

\subsection{Lyapunov and Corrected Residual Estimates}
\paragraph{Proof for Proposition \ref{prop:lyap-main}}
\begin{proof}
Define row and column scaling potentials 
\[
p_i^{t+1} := -\rho_t \log \left(
\frac{\mu_i}{\sum_{\ell} w^t_{i\ell}\exp\left(G_{i\ell}(w^t)/\rho_t\right)} \right), \qquad 
q_j^{t+1} := -\rho_t \log\left(
\frac{\nu_j}{\sum_{\ell} \pi^{t+1}_{\ell j}\exp\left(G_{\ell j}(\pi^{t+1})/\rho_t\right)}
\right),
\]
with lifted matrices defined as $p_{ij}^{t+1}:=p_i^{t+1}$ and $q_{ij}^{t+1}:=q_j^{t+1}$ for $i\in [n]$ and $j \in [m]$. Then $p^{t+1}$ is row-constant and $q^{t+1}$ is column-constant. Since row-constant matrices are exactly normals to the affine row-constraint manifold and similarly for column-constant matrices, we have $p^{k+1}\in \mathcal N_{\mathcal C_1}(\pi^{t+1})$ and $q^{t+1}\in \mathcal N_{\mathcal C_2}(w^{t+1})$. Also, since 
\[
\log \pi_{ij}^{t+1} = \frac{G_{ij}(w^t)}{\rho_t} + \log \frac{\mu_i}{\sum_{l}w_{il}^t\exp(G_{il}(w^t)/\rho_t)} + \log w_{ij}^t = \frac{G_{ij}(w^t)}{\rho_t} - \frac{p_{ij}^{t+1}}{\rho_t} + \log w_{ij}^t,
\]
at iteration $t$, the exact half-step optimality conditions are:
\begin{align}\label{eq:half_step_opt}
 - G(w^t) + \rho_t(\log\pi^{t+1} - \log w^t) + p^{t+1}= \mathbf 0,\\
    -G(\pi^{t+1}) + \rho_t(\log w^{t+1} - \log \pi^{t+1}) + q^{t+1} = \mathbf 0.
\end{align}
Define $\Lambda_t = f(\pi^t, w^t) + \rho_t D_h(\pi^t, w^t)$. Since $\pi^{t+1}, \pi^t \in \mathcal C_1$ and $p^{t+1}$ is row constant, $\sum_j \pi_{ij}^{t+1} = \sum_j \pi_{ij}^t = \mu_i$ and $p^{t+1}_{ij} = p_i^{t+1}$. Thus $\langle p^{t+1}, \pi^{t+1} - \pi^t\rangle = 0$. Then 
\[
-\langle G(w^t), \pi^{t+1} - \pi^t\rangle + \rho_t \langle \log \pi^{t+1} - \log w^t, \pi^{t+1} - \pi^t\rangle = 0.
\]
Since $f(\pi, w) = -\langle \pi, G(w)\rangle$, $f(\pi^t,w^t)-f(\pi^{t+1},w^t) = \langle G(w^t),\pi^{t+1}-\pi^t\rangle$. By three-point identity for Bregman divergences, $\langle \log\pi^{t+1}-\log w^t, \pi^{t+1}-\pi^t \rangle = D_h(\pi^t,\pi^{t+1}) +D_h(\pi^{t+1},w^t) -D_h(\pi^t,w^t)$. Thus
\begin{equation}\label{eq:exc-cyc-ident-fir}
    f(\pi^{t}, w^t) - f(\pi^{t+1}, w^t) = \rho_t \left( D_h(\pi^t, \pi^{t+1}) + D_h(\pi^{t+1}, w^t) - D_h(\pi^t, w^t)\right).
\end{equation}

Similarly, since $w^t, w^{t+1} \in \mathcal C_2$ and $q^{t+1}$ is column constant, $\langle q^{t+1}, w^t - w^{t+1}\rangle = 0$ and we can obtain 
\begin{equation}\label{eq:exc-cyc-ident-sec}
    f(\pi^{t+1}, w^t) - f(\pi^{t+1}, w^{t+1}) = \rho_t \left( D_h(w^t, w^{t+1}) + D_h(w^{t+1}, \pi^{t+1}) - D_h(w^{t}, \pi^{t+1})\right).
\end{equation}

Adding \eqref{eq:exc-cyc-ident-fir} and \eqref{eq:exc-cyc-ident-sec} together, we obtain 
\begin{multline*}
f(\pi^t, w^t) - f(\pi^{t+1}, w^{t+1}) = \rho_t \big(
D_h(\pi^t, \pi^{t+1}) + D_h(w^t, w^{t+1}) \\
+ D_h(\pi^{t+1}, w^t) + D_h(w^{t+1}, \pi^{t+1}) - D_h(\pi^t, w^t) - D_h(w^{t}, \pi^{t+1})\big).
\end{multline*}

Recall that $A_t=A^{\mathrm{row}}_t+A^{\mathrm{col}}_t = D_h(\pi^{t+1},w^t)-D_h(w^t,\pi^{t+1}) + D_h(w^{t+1},\pi^{t+1})-D_h(\pi^{t+1},w^{t+1})$. Combining this, we obtain 
\begin{equation}\label{eq:exc-cyc-ident}
\Lambda_t - f(\pi^{t+1}, w^{t+1}) - \rho_t D_h(\pi^{t+1}, w^{t+1}) 
= \rho_t (D_h(\pi^t, \pi^{t+1}) + D_h(w^t, w^{t+1}) + A_t).
\end{equation}
  From \eqref{eq:exc-cyc-ident},
$\Lambda_t - \Lambda_{t+1}  = \rho_t \left(D_h(\pi^t, \pi^{t+1}) + D_h(w^t, w^{t+1}) + A_t\right) -(\rho_{t+1}-\rho_t)D_h(\pi^{t+1}, w^{t+1})$. 
By Lemma~\ref{lem:guarded-backtracking}, $\alpha_t$ is nonincreasing and $\alpha_t \geq \frac{c_-}{t+2}$ for all sufficiently large $t$. Then $\rho_t=1/\alpha_t$ is nondecreasing and $\rho_t \leq \frac{t+2}{c_-}$, i.e., $\rho_t=O(t+2)$. By Proposition \ref{prop:cubic-main},
    $\rho_t |A_t| = O\!\left((t+1)^{1-3(\beta-p)}\right)$,
which is summable when $\beta>2/3+p$. By Proposition \ref{prop:automatic-marginal-feasibility},
$D_h(\pi^{t+1},w^{t+1}) = O(\eta_t^2)=O\left((t+1)^{-2(\beta - p)}\right)$ with $r := 2(\beta-p)>4/3>1$. Since $\rho_t$ is nondecreasing and $\rho_t = O(t)$. Let $T$ be sufficiently large.
For $N>T$,
\begin{align*}
\sum_{t=T}^{N}(\rho_{t+1}-\rho_t)(t+1)^{-r}
&=\rho_{N+1}(N+1)^{-r}-\rho_T(T+1)^{-r} + \sum_{t=T+1}^{N}\rho_t\left(t^{-r}-(t+1)^{-r}\right).
\end{align*}
We know that
    $\rho_{N+1}(N+1)^{-r}=O((N+2)(N+1)^{-r})=O(N^{1-r})$,
and $r>1$. 
By the mean-value theorem, $t^{-r}-(t+1)^{-r}\leq C_r t^{-r-1}$ for a constant $C_r>0$. Then
    $\rho_t\left(t^{-r}-(t+1)^{-r}\right)
    =O(t\cdot t^{-r-1})
    =O(t^{-r})$,
and $\sum_t t^{-r}<\infty$ because $r>1$. Hence, the partial sums are uniformly bounded. Since the terms
$(\rho_{t+1}-\rho_t)(t+1)^{-r}$ are nonnegative, the series converges. Thus
\[
    \sum_{t=1}^{\infty}(\rho_{t+1}-\rho_t)D_h(\pi^{t+1},w^{t+1})<\infty.
\]
Since $\mathcal C_1\times \mathcal C_2$ is compact, $f$ is continuous and $D_h$ is nonnegative, $\Lambda_t$ is bounded below. Summing the Lyapunov identity above, we obtain
\[
\sum_{t = 1}^{\infty} \rho_t(D_h(\pi^{t}, \pi^{t+1}) + D_h(w^t, w^{t+1})) < \infty.
\]
\end{proof}
\begin{lemma} \label{lem:corr-res-inequal}
There exist an integer $K_R$ and a constant $C_R > 0$ such that for every $t \geq K_R$,
\[
    \alpha_t\|R_t\|_{M_t}^2 \leq 4 \rho_t D_h(w^t, w^{t+1}) + C_R\alpha_t\eta_t^2.
\]
\end{lemma}
\begin{proof} 
Entrywise, we have \[\min \{a, b\}(\log a - \log b)^2 \leq 2(a\log \frac{a}{b} - a + b) \Rightarrow \|\log w^t - \log w^{t+1}\|_{M_t}^2 \leq 2D_h (w^t, w^{t+1}).\]

By Pinsker's Inequality, $\|\pi^{t+1} - w^t\|_F^2 \leq \|\pi^{t+1} - w^t\|_1^2 \leq 2D_h(\pi^{t+1}, w^t)$. Then for $L_G := \|D_X\|_2 \|D_Y\|_2$,
\begin{align*}
    \|R_t\|_{M_t}^2 &\leq 2\rho_t^2 \|\nabla h(w^t) - \nabla h(w^{t+1})\|_{M_t}^2 + 2\|G(\pi^{t+1}) - G(w^t)\|_{M_t}^2 \\
    &\leq  2\rho_t^2 \|\nabla h(w^t) - \nabla h(w^{t+1})\|_{M_t}^2 + 2\|D_X\|_2^2 \|D_Y\|_2^2\|\pi^{t+1} - w^t\|_F^2\\
    & \leq 4 \rho_t^2 D_h(w^t, w^{t+1}) + 4L_G^2 D_h(\pi^{t+1}, w^t).
\end{align*}
By Proposition \ref{prop:automatic-marginal-feasibility}, there exist $K_{\mathrm{cross},1}$ and $C_{1,\mathrm{cross}}>0$ such that, for all $t\ge K_{\mathrm{cross},1}$, $D_h(\pi^{t+1},w^t) \le C_{1,\mathrm{cross}}\eta_t^2$. Hence
\[ 
4L_G^2\alpha_tD_h(\pi^{t+1},w^t) \le 4L_G^2C_{1,\mathrm{cross}}\alpha_t\eta_t^2. 
\] 
Define $C_R:=4L_G^2C_{1,\mathrm{cross}}$ and choose $K_R\ge K_{\mathrm{cross},1}$ sufficiently large. Combining all the above, for every $t\ge K_R$, 
\[ 
\alpha_t\|R_t\|_{M_t}^2 \le 4\rho_tD_h(w^t,w^{t+1}) + C_R\alpha_t\eta_t^2. 
\]
\end{proof}

\paragraph{Proof of Proposition \ref{prop:residual}}
\begin{proof}
From \eqref{eq:half_step_opt}, 
\[p^{t+1}+ q^{t+1} = G(\pi^{t+1}) + G(w^t) + \rho_t(\nabla h(w^{t}) - \nabla h(w^{t+1})),\] i.e., $R_t = -2G(w^t) + p^{t+1}+q^{t+1}$.
Since $w_{ij}^{t} \geq 0$ and $\sum_{ij} w_{ij}^t = 1$,  $w_{ij}^t \leq 1$ and thus for any matrix $U \in \mathbb R^{n \times m}$, $\|U\|^2_{M_t} = \sum_{ij}m_{ij}^tU_{ij}^2 \leq \|U\|_F^2$.
By Lemma \ref{lem:corr-res-inequal}, there exist constants $C_R>0$ and $K_R$ such that, for every $t\ge K_R$, $\alpha_t\|R_t\|_{M_t}^2 \le 4\rho_t D_h(w^t,w^{t+1}) + C_R\alpha_t\eta_t^2$. 
    Summing from $t=K_R$ to $N-1$, 
    \[ 
    \sum_{t=K_R}^{N-1} \alpha_t\|R_t\|_{M_t}^2 \le 4\sum_{t=K_R}^{N-1} \rho_tD_h(w^t,w^{t+1}) + C_R\sum_{t=K_R}^{N-1}\alpha_t\eta_t^2. 
    \] 
    By the Lyapunov summability estimate:  $\sum_{t=K_R}^{\infty} \rho_tD_h(w^t,w^{t+1}) <\infty$ from Proposition~\ref{prop:lyap-main} . 
    Since $\alpha_t=O((t+1)^{-\beta})$ and $\eta_t^2=(t+1)^{-2(\beta-p)}$, the second series is bounded by a constant times $\sum_t (t+1)^{-3\beta+2p}$, which is summable under $\beta>2/3+p$.
    Then $\sum_{t=K_R}^{N-1} \alpha_t\|R_t\|_{M_t}^2$ are uniformly bounded in $N$. Passing to the limit, 
    \[\sum_{t=K_R}^{\infty} \alpha_t\|R_t\|_{M_t}^2 <\infty.\] 
    Define $C_\star := \sum_{t=K_R}^{\infty} \alpha_t\|R_t\|_{M_t}^2 <\infty$. Then for every $N>K_R$, 
    \[ 
    \left(\min_{K_R\le i\le N-1}\|R_i\|_{M_i}^2\right) \sum_{t=K_R}^{N-1}\alpha_t 
    \le \sum_{t=K_R}^{N-1} \alpha_t\|R_t\|_{M_t}^2 \le C_\star. 
    \] 
    Thus $\min_{K_R\le i\le N-1}\|R_i\|_{M_i}^2 \le C_\star/\sum_{t=K_R}^{N-1}\alpha_t$. 
    Since 
    \[ \sum_{t=K_R}^{N-1}\alpha_t \ge c_- \sum_{t=K_R}^{N-1}\frac1{t+2}
    \ge c_- \log\!\left(\frac{N+1}{K_R+2}\right), \] we have showed \[ \min_{K_R\le i\le N-1}\|R_i\|_{M_i}^2 \le \frac{C_\star} {c_-\log\!\bigl((N+1)/(K_R+2)\bigr)}, \] which is equivalent to $\min_{0\le i\le N-1}\|R_i\|_{M_i}^2 = O\left(\frac1{\log N}\right)$.
\end{proof}

\subsection{Subsequence Convergence}
\begin{lemma}[Linear regularity of the marginal sets]
\label{lem:blr}
There exists
$\kappa>0$, depending only on $n$ and $m$, such that every
$\pi\in\mathbb R^{n\times m}$ satisfies
$\operatorname{dist}_F(\pi,\mathcal C)
\leq \kappa\max\!\left\{
\operatorname{dist}_F(\pi,\mathcal C_1),
\operatorname{dist}_F(\pi,\mathcal C_2)
\right\}$.
\end{lemma}

\begin{proof}
Let
$U=\mathcal{P}_{\mathcal C_1}(\pi)$ and
$V=\mathcal{P}_{\mathcal C_2}(\pi)$.
Denote the column sums of $U$ by
$c=U^\top\mathbf 1_n$. We first construct a coupling
$Y\in\mathcal C$ close to $U$.

For each column $j$, define
$a_j:=(c_j-\nu_j)_+$ and $b_j:=(\nu_j-c_j)_+$ where $a_j$ is the amount by which the $j$-th column sum of $U$ exceeds its prescribed marginal $\nu_j$, while $b_j$ is the
corresponding shortage. Since
$\mathbf 1_m^\top c =\mathbf 1_n^\top\mu
=\mathbf 1_m^\top\nu =1$,
the total excess equals the total shortfall, i.e.,  $d:=\sum_{j=1}^m a_j=\sum_{j=1}^m b_j$.
If $d=0$, then $U\in\mathcal C$. Otherwise, remove the excess mass
from each column proportionally:
\[
H_{ij}=
\begin{cases}
U_{ij}a_j/c_j, & c_j>0,\\
0, & c_j=0,
\end{cases}
\qquad r=H\mathbf 1_m,
\]
and redistribute the removed mass within each row among the deficient columns with
$Y=U-H+\frac{rb^\top}{d}$.
Then $Y\geq0$, $Y\mathbf 1_m=\mu$, and $Y^\top\mathbf 1_n=\nu$. Thus $Y\in\mathcal C$ and
$\|Y-U\|_1 \leq 2d =\|c-\nu\|_1$.
Because $V^\top\mathbf 1_n=\nu$,
$\|c-\nu\|_1 =\|(U-V)^\top\mathbf 1_n\|_1 \leq\|U-V\|_1 \leq\sqrt{nm}\,\|U-V\|_F$.
Therefore,
\[
\operatorname{dist}_F(\pi,\mathcal C)
\leq \|\pi-U\|_F+\|U-Y\|_F
\leq
(1+\sqrt{nm})\operatorname{dist}_F(\pi,\mathcal C_1)
+\sqrt{nm}\operatorname{dist}_F(\pi,\mathcal C_2).
\]
With $\kappa=1+2\sqrt{nm}$, $\operatorname{dist}_F(\pi,\mathcal C)
\leq \kappa\max\!\left\{
\operatorname{dist}_F(\pi,\mathcal C_1),
\operatorname{dist}_F(\pi,\mathcal C_2)
\right\}$.
\end{proof}
\paragraph{Proof for Lemma \ref{lem:aae}}
\begin{proof}
Proposition~\ref{prop:cubic-main} gives
$\bigl|D_h(\pi^{t+1},w^t)-D_h(w^t,\pi^{t+1})\bigr|
+
\bigl|D_h(w^{t+1},\pi^{t+1})
      -D_h(\pi^{t+1},w^{t+1})\bigr|
=O(\eta_t^3)$.
Because $\eta_t=(t+1)^{-(\beta-p)}$ and $\beta-p>2/3$, the resulting series is summable.
\end{proof}

\paragraph{Proof for Theorem \ref{thm:unconditional}}
\begin{proof}
    Since $k_0<\infty$, the finite prefix before $k_0$ does not affect any asymptotic statement. After relabeling the post-burn tail, the accepted steps satisfy the guarded bounds in Lemma~\ref{lem:guarded-backtracking}; all estimates above apply with constants depending on the switched point $(\pi^{k_0},w^{k_0})$.

\textbf{(1)} By Proposition \ref{prop:residual}, $\sum_{t=K_R}^{\infty}\alpha_t\|R_t\|_{M_t}^2<\infty$.
Since $\sum_{t=K_R}^{\infty}\alpha_t=\infty$, if $ \liminf_{t\to\infty}\|R_t\|_{M_t} \neq 0$, then there exists $\varepsilon>0$ and $K\ge K_R$ such that for all $t \geq K$, $\|R_t\|_{M_t}\ge \varepsilon$. Then \[\sum_{t=K}^{\infty}\alpha_t\|R_t\|_{M_t}^2 \ge  \varepsilon^2\sum_{t=K}^{\infty}\alpha_t   = \infty,\] which contradicts Proposition \ref{prop:residual}. Thus $\liminf_{t\to\infty}\|R_t\|_{M_t}=0$.

\textbf{(2)}By Proposition \ref{prop:automatic-marginal-feasibility}, $D_h(\pi^{t+1},w^{t+1})\le C_{2,\mathrm{cross}}\eta_t^2$. Since $\eta_t\to0$, $D_h(\pi^{t+1},w^{t+1})\to 0$. Since  $\pi^{t+1}$ and $w^{t+1}$ have total mass one, by Pinsker's inequality $\|\pi^{t+1}-w^{t+1}\|_1^2 \le  2D_h(\pi^{t+1},w^{t+1})$, and thus $ \varepsilon_t:= \|\pi^{t+1}-w^{t+1}\|_1\to 0$. Define $x^{t+1}:=\mathcal P_{\mathcal C}(\pi^{t+1})$. Since $\pi^{t+1}\in \mathcal C_1$ and $w^{t+1}\in \mathcal C_2$, we have $\operatorname{dist}(\pi^{t+1},\mathcal C_1)=0$ and $\operatorname{dist}(\pi^{t+1},\mathcal C_2) \le \|\pi^{t+1}-w^{t+1}\|$. Since $\mathcal C_1, \mathcal C_2$ are polyhedral, by Lemma \ref{lem:blr}, \[\operatorname{dist}(\pi^{t+1},\mathcal C) \le \kappa\max\{ \operatorname{dist}(\pi^{t+1},\mathcal C_1), \operatorname{dist}(\pi^{t+1},\mathcal C_2) \},\] and
\begin{align*}
    \|x^{t+1} - \pi^{t+1}\| + \|x^{t+1} - w^{t+1}\| & \leq 2\|x^{t+1} - \pi^{t+1}\| + \|\pi^{t+1} - w^{t+1}\| = 2 \mathrm{dist}(\pi^{t+1}, \mathcal C) + \varepsilon_t \leq (2\kappa + 1) \varepsilon_t.
\end{align*}
 As $\varepsilon_t \to 0$, $\|x^{t+1} - \pi^{t+1}\| \to 0$ and $\|x^{t+1} - w^{t+1}\| \to 0$. Then the sequences $\{x^t\}$, $\{\pi^t\}$, and $\{w^t\}$ have the same set of accumulation points.
\end{proof}
\begin{lemma}\label{lem:active-face-compactness}
    Let $x^{t_\ell+1}=\mathcal P_{\mathcal C}(\pi^{t_\ell+1})\to \bar x$ be a residual-vanishing
convergent subsequence, and let $S:=\{(i,j):\bar x_{ij}>0\}$.
The following conditions hold along the subsequence:
\begin{enumerate}[label=(\alph*)]
\item \textbf{Active-support nondegeneracy.} There exists $c_S>0$ such that for all sufficiently large $\ell$, for all $(i,j)\in S$, $\min\{w^{t_\ell}_{ij},w^{t_\ell+1}_{ij}\}\ge c_S$.
\item \textbf{Bounded normalized row/column potentials.} Let Assumption
\ref{assum:support} hold along this subsequence. Then, after fixing a reference row index $i_0$, $\{p^{t_\ell+1}\}$ and $\{q^{t_\ell+1}\}$ are bounded.
\end{enumerate}
\end{lemma}

\begin{proof}
\textbf{(a)}  By Proposition \ref{prop:automatic-marginal-feasibility}, $D_h(\pi^{t_\ell+1},w^{t_\ell+1})\to 0$. Since both $\pi^{t_\ell+1}$ and $w^{t_\ell+1}$ have total mass one, by Pinsker's inequality, 
$\|\pi^{t_\ell+1}-w^{t_\ell+1}\|_1 \to0$ .
Since $\pi^{t_\ell+1}\in \mathcal C_1$ and $w^{t_\ell+1}\in \mathcal C_2$, $\mathrm{dist}(\pi^{t_\ell+1},\mathcal C_1)=0$ and $\mathrm{dist}(\pi^{t_\ell+1},\mathcal C_2) \le \|\pi^{t_\ell+1}-w^{t_\ell+1}\|$. 
Since $\mathcal C_1$ and $\mathcal C_2$ are polyhedral, by Lemma \ref{lem:blr}, there exists a constant $\kappa>0$ such that
\[
\rm{dist}(\pi^{t_\ell+1},\mathcal C_1\cap \mathcal C_2)
\le \kappa
    \max\{\rm{dist}(\pi^{t_\ell+1},\mathcal C_1), \rm{dist}(\pi^{t_\ell+1},\mathcal C_2) \}.
\]
Since $x^{t_\ell+1}=\mathcal P_{\mathcal C}(\pi^{t_\ell+1})$ and $\mathcal C = \mathcal C_1\cap \mathcal C_2$, 
$\|x^{t_\ell+1}-\pi^{t_\ell+1}\| = \rm{dist}(\pi^{t_\ell+1},\mathcal C) \to 0$.
Since $x^{t_\ell+1}\to\bar x$, $\pi^{t_\ell+1}\to\bar x$.
By Proposition~\ref{prop:lyap-main},
$\sum_{k=1}^{\infty}\rho_k D_h(w^k,w^{k+1})<\infty$.
Since $\rho_k$ is bounded below by a positive constant, $D_h(w^k,w^{k+1})\to0$. 
Again by Pinsker's inequality $\|w^{k+1}-w^k\|_1\to0$ .
Therefore $w^{t_\ell}\to\bar x$ as well. 
If $(i,j)\in S$, then
$\bar x_{ij}>0$. Hence, for all sufficiently large $\ell$,
    \[w^{t_\ell}_{ij}\ge \frac{\bar x_{ij}}{2} \qquad\text{and} \qquad
    w^{t_\ell+1}_{ij}\ge \frac{\bar x_{ij}}{2}.\]
Since $S$ is finite, for all $(i,j)\in S$ and for sufficiently large $\ell$,
$\min\{w^{t_\ell}_{ij},w^{t_\ell+1}_{ij}\}\ge c_S$ where \[c_S:=\frac12\min_{(i,j)\in S}\bar x_{ij}>0.\]

\textbf{(b)} By \textbf{(a)}, for every $(i,j)\in S$,
$m^{t_\ell}_{ij} = \min\{w^{t_\ell}_{ij},w^{t_\ell+1}_{ij}\} \ge c_S$ for sufficiently large $\ell$. 
Since $\|R_{t_\ell}\|_{M_{t_\ell}}\to 0$,
for $(i,j)\in S$, 
\[
c_S R_{t_\ell, ij}^2 \leq m_{ij}^{t_\ell}R_{t_\ell, ij}^2 
\leq \sum_{a,b} m^{t_\ell}_{ab} R_{t_\ell, ab}^2
= \|R_{t_\ell}\|_{M_{t_\ell}}^2 \to 0,
\]
so $R_{t_\ell, ij} \to 0$.
Since $R_k=-2G(w^k)+p^{k+1}+q^{k+1}$, for every $(i,j)\in S$,
$p_i^{t_\ell+1}+q_j^{t_\ell+1}
= 2G(w^{t_\ell})_{ij}+R_{t_\ell,ij}$.
Since $w^{t_\ell}\to\bar x$ and $G$ is continuous, the right-hand side is bounded. Hence the active-edge sums
$p_i^{t_\ell+1}+q_j^{t_\ell+1}$ are bounded for all $(i,j)\in S$.

Define $\mathcal G_S$ as a bipartite graph with row vertices $\{1,\ldots,n\}$ and column vertices $\{1,\ldots,m\}$, and edge set $S$. Since $\bar x\in \mathcal C$ and the $\mu,\nu$ have strictly positive entries, every row and column is incident to at least one edge of $S$. 
Let $\mathcal C_1,\ldots, C_r$ be the connected components of $\mathcal G_S$. For each component, write $ C_a=(I_a,J_a)$, where $I_a$ and $J_a$ denote its row and column index sets. Choose one reference row index $i_a\in I_a$ and define
\[
    \widehat p_i^{t_\ell+1} :=
    p_i^{t_\ell+1}-p_{i_a}^{t_\ell+1},
    \qquad
    \widehat q_j^{t_\ell+1} :=
    q_j^{t_\ell+1}+p_{i_a}^{t_\ell+1},
\]
for every $i,j\in C_a$.
Then $\widehat p_{i_a}^{t_\ell+1}=0$ since $p_i^{t_\ell+1}$ is row-constant and on every active edge $(i,j)\in S$,
\[\widehat p_i^{t_\ell+1} + \widehat q_j^{t_\ell+1} = p_i^{t_\ell+1}+ q_j^{t_\ell+1}.\]
Hence, these normalized sums remain bounded on all active entries.
Because $ C_a$ is connected, as shown in Figure \ref{fig:lemma_10_b}, every row or column index in $ C_a$ can be reached from $i_a$ through a finite alternating sequence of row and column indices with corresponding entries belonging to $S$. 
Starting from $\widehat p_{i_a}^{t_\ell+1}=0$, boundedness of the active-entry sums
successively implies boundedness of the adjacent normalized column multipliers and then of the adjacent normalized row multipliers. 
Proceeding along such a sequence, we show that $\{\widehat p_i^{t_\ell+1}\}_\ell$ and $\{\widehat q_j^{t_\ell+1}\}_\ell$ bounded for every $i\in I_a$ and $j\in J_a$.
\begin{figure}[ht]
  \centering
  \includegraphics[width=0.42\textwidth]{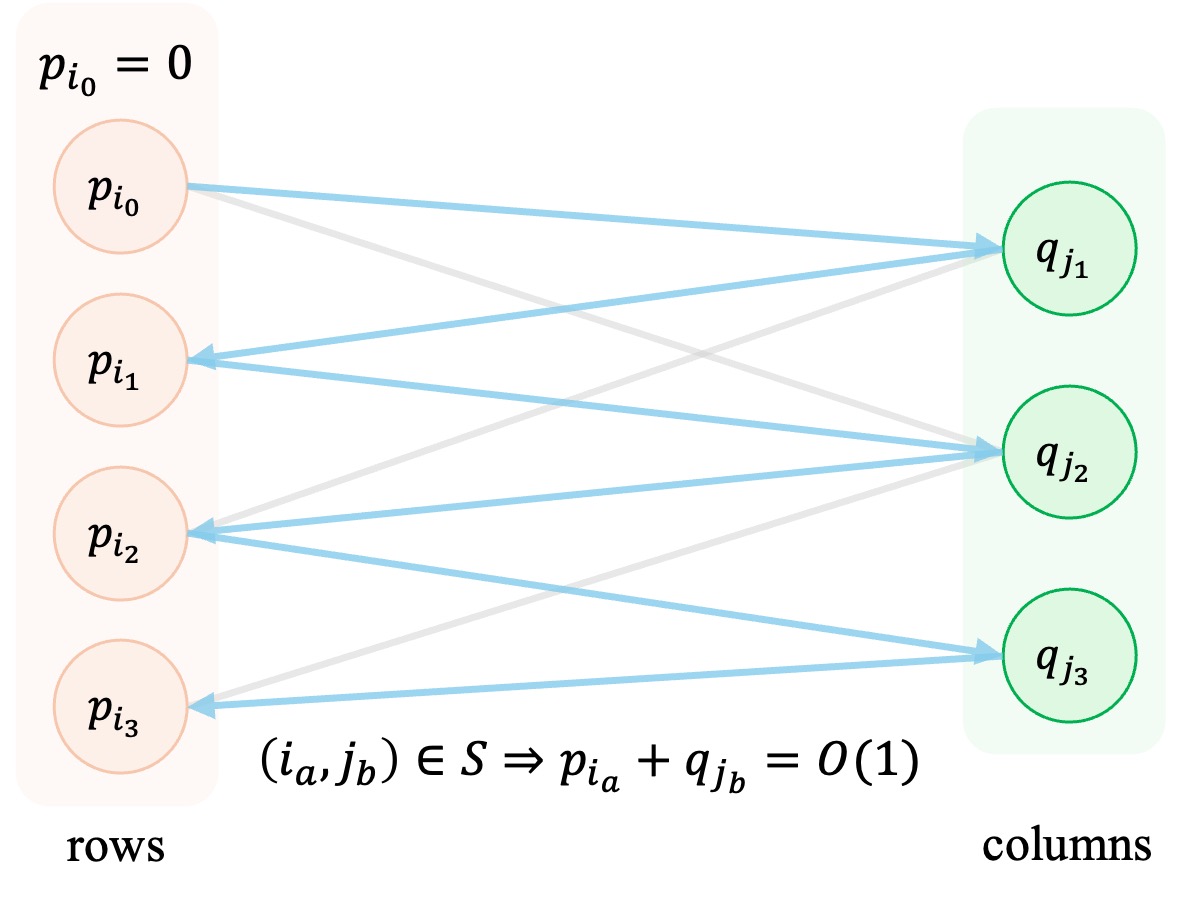}
  \caption{Bipartite Active-Support Graph. Each edge represents
  $(i_a,j_b)\in S$. Fixing
  $p_{i_0}=0$, the bounded edge sums propagate boundedness along
  the blue alternating path; gray lines denote the remaining
  active-support edges.}
  \label{fig:lemma_10_b}
\end{figure}

Define the component offset
$c_a^\ell := p_{i_a}^{t_\ell+1}$. We normalize the offsets by setting $c_1^\ell=0$, and prove that the relative offsets $c_a^\ell-c_b^\ell$ are bounded.
For two distinct components $ C_a$ and $ C_b$, choose a row vertex $i_a\in C_a$ and a column vertex $j_b\in C_b$. Then $(i_a,j_b)\notin S$. 
By Assumption~\ref{assum:support},
$\liminf_{\ell\to\infty} R_{t_\ell,i_a j_b}\ge0$. By construction,
\[
R_{t_\ell,i_a j_b}
= -2G(w^{t_\ell})_{i_a j_b} + \widehat p_{i_a}^{t_\ell+1} + \widehat q_{j_b}^{t_\ell+1}
 + c_a^\ell-c_b^\ell .
\]
Since the first three terms on the right-hand side are bounded, by Assumption~\ref{assum:support}, $c_a^\ell-c_b^\ell$ is bounded from below.
Similarly, reverse and choose a row vertex in $\mathcal C_b$ and a column vertex in $\mathcal C_a$. The same argument shows that $c_b^\ell-c_a^\ell$ is bounded from below. Therefore, $c_a^\ell-c_b^\ell$ is bounded above and below.
Since there are finitely many connected components and $c_1^\ell=0$, all normalized component offsets $\{c_a^\ell\}_\ell$ are bounded. 
Combining this with the boundedness of $\{\widehat p^{t_\ell + 1}\}$ and $\{\widehat q^{t_\ell + 1}\}$, we obtain the boundedness of
$\{p^{t_\ell+1}\}_\ell$ and $\{q^{t_\ell+1}\}_\ell$.
\end{proof}

\paragraph{Proof for Theorem \ref{thm:stationarity}}
\begin{proof}
Let $x^{t_\ell+1}\to\bar x$ be any residual-vanishing convergent shadow subsequence. By Theorem~\ref{thm:unconditional},
$\pi^{t_\ell+1}\to\bar x$ and $w^{t_\ell+1}\to\bar x$. Lemma~\ref{lem:active-face-compactness} further gives $w^{t_\ell}\to\bar x$ and, after normalizing the row and column potentials and passing to a further
subsequence, $p^{t_\ell+1}\to\bar p$ and  $q^{t_\ell+1}\to\bar q$.
Thus the identity $R_t=-2G(w^t)+p^{t+1}+q^{t+1}$ passes to the limit. Writing the limit as $\bar R$, Lemma~\ref{lem:active-face-compactness} gives $\bar R_{ij}=0$ on $\operatorname{supp}(\bar x)$, while
Assumption~\ref{assum:support} gives $\bar R_{ij}\geq0$ on its complement. Therefore,
$\mathbf 0=-2G(\bar x)+\bar p+\bar q-\bar R$ and
$\bar p+\bar q-\bar R\in\mathcal N_{\mathcal C}(\bar x)$.
Since $\nabla F(\bar x)=-2G(\bar x)$,
$\mathbf 0\in\nabla F(\bar x)+\mathcal N_{\mathcal C}(\bar x)$. Since Theorem~\ref{thm:unconditional} provides a residual-vanishing subsequence, and compactness of $\mathcal C$ provides a convergent further subsequence, the shadow sequence has at least one stationary accumulation point.
\end{proof}

\subsection{Proof of Corollary \ref{coro:1/k}}\label{app:1/k}
\begin{proof}
    For any $(i,i^\prime, j, j^\prime)$,
    \[
    \log \frac{(z_1^k)_{ij}(z_1^k)_{i^\prime j^\prime}}{(z_1^k)_{i^\prime j}(z_1^k)_{ij^\prime}}
    = \log \frac{w^k_{ij}w^k_{i^\prime j^\prime}}{w^k_{i^\prime j}w^k_{ij^\prime}} 
    + \alpha_k^{\prime} (G(w^k)_{ij} + G(w^k)_{i^\prime j^\prime} - G(w^k)_{i^\prime j} - G(w^k)_{ij^\prime}).
    \]
    Since $\|G(w^k)\|_{\infty} \leq M_G$, $G(w^k)_{ij} + G(w^k)_{i^\prime j^\prime} - G(w^k)_{i^\prime j} - G(w^k)_{ij^\prime} \leq 4M_G$. Taking maxima on both sides, we obtain $\Delta (z_1^k) \leq \Delta (w^k) + 4M_G \alpha_k^{\prime}$. Since the row projection onto $\mathcal C_1$ is a diagonal row scaling, $\Delta (\pi^{k+1}) = \Delta(z_1^k)$. 
    Similarly, $\Delta(z_2^k) \leq \Delta(\pi^{k+1}) +4M_G \alpha_k^{\prime}$ with $\Delta(w^{k+1}) = \Delta (z_2^k)$. Thus \[\Delta(w^{k+1}) \leq \Delta (\pi^{k+1}) + 4M_G \alpha_k^{\prime} \leq \Delta (w^k) + 8M_G\alpha_k^{\prime} = \Delta (w^k) + B_\Delta\alpha_k^{\prime}.\] Summing the changes up until iteration $k$, 
    \[
    \Delta(w^{k}) \leq \Delta (w^0) + B_\Delta^{\prime} \sum_{t = 0}^{t-1} \frac{1}{t+1} \leq \Delta(w^0) + B_{\Delta}^{\prime}(1+\log k).
    \]
    Similar to the proof in Lemma \ref{lem:polynomial-row-error},
    \[
    1 - \tanh^2\left(\frac{\Delta(w^k)}{4}\right) 
    \geq \exp(-\frac{\Delta(w^k)}{2})
    \geq \exp\left(-\frac{\Delta(w^0)+B_\Delta^\prime}{2}\right)(k+1)^{-4M_G} =: c_\Delta^\prime (k+1)^{-4M_G},
    \]
    with $c_\Delta^\prime := \exp\left(-\frac{1}{2}(\Delta(w^0)+B_\Delta^\prime)\right)$. Hence \[\tanh^2\left(\frac{\Delta(w^k)}{4}\right) \leq 1 - c_\Delta^{\prime}(k+1)^{-\gamma},\] with $\gamma := 4M_G$. Thus, the harmonic trajectory produces the same polynomial contraction-loss form used in Lemma \ref{lem:polynomial-row-error} with $p$ replaced by $\gamma$. 
    
    By Lemma \ref{lem:full-cycle-comparison}, 
    $d_H\left(r(w^{k+1}, r(\widehat w^k)\right)\leq 12M_G \alpha_k^\prime = O((k+1)^{-1})$.
    Combining this with Lemma \ref{lem:sinkhorn-projective-contraction}, we obtain 
    \[
    d_H(r(w^{k+1}), \mathbf 1_n) \leq (1- c_\Delta^\prime(k+1)^{-\gamma}) d_H(r(w^{k}), \mathbf 1_n) + C^\prime (k+1)^{-1},
    \]
    which is exactly the recursion we see in Lemma \ref{lem:polynomial-row-error} with $\beta = 1, p = \gamma$, and $\beta - p = 1 - \gamma$. When $B_\Delta^\prime < \frac{2}{3}$, $\gamma < \frac{1}{3} < \beta$. Thus the comparison condition in Lemma \ref{lem:polynomial-row-error} applies and gives
    $d_H(r(w^{k}), \mathbf 1_n) = O((k+1)^{-(1-\gamma)}) = O(\eta_k^{\prime})$
    for $\eta_k^{\prime}:= (k+1)^{-(1-\gamma)}$. 
    Similar to Proposition \ref{prop:automatic-marginal-feasibility}, $\|w^k \mathbf 1_m - \mu\|_1 = O(\eta_k^\prime)$ and $\|z_1^k - w^k\| = O(\alpha_k^\prime)$. 
    Since $\gamma \geq 0$, $\alpha_k^\prime = (k+1)^{-1} \leq (k+1)^{-(1-\gamma)} = \eta_k^{\prime}$. Then 
    \[
    \|z_1^k\mathbf 1_m - \mu\|_1 \leq \|z^{k}_1 - w^k\|_1 + \|w^k\mathbf{1}_m -\mu\|_1 = O(\eta_k^\prime) + O(\alpha_k^\prime) = O(\eta_k^\prime).
    \]
    Similarly, we obtain $\|(\pi^k)^\top\mathbf 1_n - \nu\|_1 = O(\eta_k^\prime)$, thus yielding corresponding row- and column-marginal feasibility estimates of order $O(\eta_k^\prime)$.
    
    Since the proof for Proposition \ref{prop:automatic-marginal-feasibility} and \ref{prop:cubic-main} requires only the marginal errors and the step size to be both $O(\eta_k^\prime)$ with $\gamma < 1$. The same proof structure gives $D_h(\pi^{k+1}, w^k) = O({\eta_k^\prime}^2)$, $D_h(\pi^{k+1}, w^{k+1}) = O({\eta_k^\prime}^2)$, and $|A_t| = O({\eta_k^\prime}^3)$.

    For the harmonic schedule, $\rho_{k+1}^\prime - \rho_k^\prime = 1$ and $\rho_k^\prime |A_k| = O\left((k+1)^{1-3(1-\gamma)}\right) = O\left((k+1)^{-2 + 3\gamma}\right)$. When $\gamma < \frac{1}{3}$, 
    \[\sum_{k = 1}^{\infty} \rho_k^{\prime}|A_t| < \infty.\]
    When $\gamma < \frac{1}{2}$, the series of penalty-variation error satisfies
    \[
   \sum_{k = 1}^{\infty} (\rho_{k+1}^\prime - \rho_k^\prime)D(\pi^{k+1}, w^{k+1}) = \sum_{k = 1}^{\infty} O({\eta^\prime_k}^2) = \sum_{k = 1}^{\infty} O\left((k+1)^{-2(1-\gamma)}\right) < \infty.
    \]
    By the Lyapunov identity and lower-boundedness similar to Proposition \ref{prop:lyap-main}, we obtain
    \[
    \sum_{k = 1}^{\infty} \rho_k^\prime \left(D_h(\pi^k, \pi^{k+1}) + D_h(w^k, w^{k+1})\right) < \infty.
    \]
    Lemma \ref{lem:corr-res-inequal} only uses the half-step optimality conditions and the cross-divergence bound. Thus, Lemma \ref{lem:corr-res-inequal} remains valid for $(\alpha_k^\prime, \rho_k^\prime, \eta_k^\prime)$ where $\alpha_k^\prime\|R_k\|_{M_k}^2  = O\left(\rho_k^\prime D_h(w^k, w^{k+1})\right) + O(\alpha_k^\prime {\eta_k^\prime}^2)$. When $\gamma < \frac{1}{3}$, $\alpha_k^\prime {\eta_k^\prime}^2 = (k+1)^{-1 - 2 (1-\gamma)} = (k+1)^{-3+2\gamma}$ is summable. Thus 
    \[
    \sum_{k = 1}^\infty \alpha_k^\prime\|R_k\|_{M_k}^2  = \sum_{k = 1}^\infty O\left(\rho_k^\prime D_h(w^k, w^{k+1})\right) + \sum_{k = 1}^\infty O(\alpha_k^\prime {\eta_k^\prime}^2) < \infty.
    \]
    Since $\sum_{k = 1}^{N-1} \alpha_k^\prime = \Theta(\log N)$, similar to the proof in Proposition \ref{prop:residual}, we obtain 
    \[
    \min_{0 \leq i \leq N-1} \|R_i\|_{M_i}^2 = O\left(\frac{1}{\log N}\right).
    \]
    The proof of Theorem \ref{thm:stationarity} uses the divergence of $\sum_k \alpha_k^\prime$, compactness, and Assumption \ref{assum:support}, residual summability, and vanishing cross-divergence, all of which were established above. We thus show that under the restraint $\gamma < \frac{1}{3}$, i.e., $M_G < \frac{1}{12}$, the same subsequence-convergence conclusions follow.
\end{proof}

\subsection{Proof of Proposition \ref{cor:mono-suff}}
\begin{proof}
Since $R_t = \rho_t\bigl(\log w^t-\log w^{t+1}\bigr)
+G(\pi^{t+1})-G(w^t)$, by Proposition~\ref{prop:automatic-marginal-feasibility} and Pinsker's inequality,
\[\|\pi^{t+1}-w^t\|_F \leq \|\pi^{t+1}-w^t\|_1 \leq \sqrt{2D_h(\pi^{t+1},w^t)} = O(\eta_t).\] 
Since $G$ is linear,
\[\|G(\pi^{t+1})-G(w^t)\|_\infty \leq \|D_X(\pi^{t+1}-w^t)D_Y\|_F \leq \|D_X\|_2\|D_Y\|_2 \|\pi^{t+1}-w^t\|_F =O(\eta_t).\]
Since $D_X$ and $D_Y$ are fixed, along $\{t_\ell\}$,
\[R_{t_\ell,ij} = \rho_{t_\ell} \log\frac{w^{t_\ell}_{ij}}{w^{t_\ell+1}_{ij}} +O(\eta_{t_\ell}).\]  Fix $(i,j)\notin S(\bar x)$. By hypothesis, there exists a sequence $\varepsilon_\ell=o(\alpha_{t_\ell})$ such that
$\left(w^{t_\ell+1}_{ij}/w^{t_\ell}_{ij}\right) \leq 1+\varepsilon^{ij}_\ell$.
Consequently, since $\varepsilon^{ij}_\ell=o(\alpha_{t_\ell})$ and
$\log(1+\varepsilon^{ij}_\ell)=\varepsilon^{ij}_\ell+o(\varepsilon^{ij}_\ell)$,
\[
\rho_{t_\ell} \log\frac{w^{t_\ell}_{ij}}{w^{t_\ell+1}_{ij}}
\geq -\rho_{t_\ell}\log(1+\varepsilon^{ij}_\ell) 
=-\frac{\log(1+\varepsilon^{ij}_\ell)}{\alpha_{t_\ell}}
=o(1).
\]
Since $\eta_{t_\ell}\to0$, the residual decomposition therefore yields $\liminf_{\ell\to\infty}R_{t_\ell,ij}\geq0$, which holds for every $(i,j)\notin S(\bar x)$. Therefore, Assumption~\ref{assum:support} holds along the subsequence.

If $w^{t_\ell+1}_{ij}\leq w^{t_\ell}_{ij}$ eventually, then \[\rho_{t_\ell}\log\left(\frac{w^{t_\ell}_{ij}}{w^{t_\ell+1}_{ij}}\right)\geq0,\] so the same conclusion follows immediately.
The first condition then implies $\liminf_{\ell\to\infty}R^{t_\ell}_{ij}\geq0$. 
Moreover, since $\varepsilon_\ell^{ij}=o(\alpha_{t_\ell})$ and
$\rho_{t_\ell}=1/\alpha_{t_\ell}$, $-\rho_{t_\ell}\log(1+\varepsilon_\ell^{ij})=o(1)$, which proves the second claim.
\end{proof}

\subsection{Proof of Proposition \ref{prop:no-universal-power-lower-bound}}\label{app:proof_lower_bound}

\begin{proof}
    Let
    \[
        \mu=\nu=\left(\frac12,\frac12\right)^\top,
        \qquad
        D_X=D_Y=
        \begin{pmatrix}
        0&1\\
        1&0
        \end{pmatrix}
        =:J .
    \]
    Consider the one-parameter family
    \[
        W(u)=
        \begin{pmatrix}
        u&\frac12-u\\
        \frac12-u&u
        \end{pmatrix},
        \qquad u\in\left(\frac14,\frac12\right).
    \]
    Then $W(u)\in \mathcal C_1\cap \mathcal C_2$ and $G(W(u))=JW(u)J=W(u)$. Its projective diameter is
    \[
        \Delta(W(u))
        = \log\frac{W_{11}(u)W_{22}(u)}{W_{12}(u)W_{21}(u)}
        = 2\log\frac{u}{\frac12-u}>0.
    \]
    During the first exponential perturbation, the log cross-ratio changes by
    \[
        \alpha\bigl(G_{11}(W(u))+G_{22}(W(u))-G_{12}(W(u))-G_{21}(W(u))\bigr) = \alpha(4u-1).
    \]
    The row normalization is a positive diagonal scaling and therefore preserves cross-ratios. Moreover, the normalized matrix remains in the same family, say $W(u_1)$, with $u_1>u$. The second exponential perturbation changes the log cross-ratio by $\alpha(4u_1-1)\geq \alpha(4u-1)$, and the column normalization again preserves cross-ratios. Hence, one full alternating update cycle satisfies
    \[
        \Delta(w^{t+1})-\Delta(w^t) \geq 2\alpha_t(4u_t-1).
    \]
    Since $u_t$ is nondecreasing and $u_0>1/4$, there exists $c_0>0$ such that $2(4u_t-1)\geq c_0$ for all $t$. Therefore
    \[
        \Delta(w^T)-\Delta(w^0)
        \geq
        c_0\sum_{s=0}^{T-1}\alpha_s.
    \]
    If the projective guard holds, then
       $ \Delta(w^T)-\Delta(w^0) \leq a\log(T+1)$.
    Thus
    \[
        \sum_{s=0}^{T-1}\alpha_s
        \leq
        \frac{a}{c_0}\log(T+1).
    \]
    Now suppose, for contradiction, that there exist $c>0$, $\beta<1$, and $T_0$
    such that
        $\alpha_t\geq c(t+1)^{-\beta}$, for every $t\geq T_0$.
    Then for every $T>T_0$,
    \[
        \sum_{s=0}^{T-1}\alpha_s
        \geq
        c\sum_{s=T_0}^{T-1}(s+1)^{-\beta}
        \geq
        c'\bigl(T^{1-\beta}-T_0^{1-\beta}\bigr)
    \]
    for some constant $c'>0$. Since $\beta<1$, this lower bound grows faster than
    $\log(T+1)$, contradicting
    \[
        \sum_{s=0}^{T-1}\alpha_s
        \leq
        \frac{a}{c_0}\log(T+1).
    \]
\end{proof}
This proposition explains why Lemma~\ref{lem:guarded-backtracking} is stated in
the asymmetric form
    $\frac{c_-}{t+2}\leq \alpha_t\leq c_+(t+1)^{-\beta}$.
The logarithmic guard gives only the harmonic accepted-step floor $\alpha_t\geq c_-/(t+2)$, while the delayed-power proposal gives the upper control $\alpha_t\leq c_+(t+1)^{-\beta}$.

\section{Experiments}

\subsection{Synthetic Graph Alignment Experiment Details}\label{app:syn_ga}
\paragraph{Graph Construction}
The source graph contains $n=250$ nodes and is sampled from a Gaussian random partition model with parameters $s=50$, $v=5$, within-community edge probability $p_{\mathrm{in}}=0.2$, and between-community edge probability $p_{\mathrm{out}}=0.02$. Here, $s$ is the nominal mean community size, and $v$ controls its variability. Community sizes are sampled until 250 nodes are assigned, with the remaining nodes assigned to the final community. Given the source graph, we construct the target graph using the noisy-node procedure of \citet{li2023convergent}. For noise level $q=0.3$, we first add $\lfloor q|E|\rfloor$ randomly sampled edges that are absent from the source graph and append $\lfloor qn\rfloor=75$ noisy nodes, resulting in a target graph with $325$ nodes. Each added node is connected randomly to the existing graph. We use the unweighted adjacency matrices as the structural matrices $D_X$ and $D_Y$, together with uniform node distributions $\mu=\mathbf{1}_{250}/250$ and $\nu=\mathbf{1}_{325}/325$. All results are averaged over 30 independent random seeds. Runtime denotes the wall-clock time from initialization to the coupling's return.

\paragraph{FGW stationarity residual.} For synthetic and real domain-adaptation experiments, we evaluate the stationarity residual using the FGW objective $F_{\mathrm{FGW}}(X)=2\lambda F_{\mathrm{GW}}(X) +(1-\lambda)\langle M,X\rangle$. Thus, $\nabla F_{\mathrm{FGW}}(X) =2\lambda\nabla F_{\mathrm{GW}}(X)+(1-\lambda)M$. The linear term does not affect the gradient Lipschitz bound, so we use $\eta =\max\{1,4\lambda|D_X|_2|D_Y|_2\}^{-1}$.

\paragraph{KL-BAPG + Sinkhorn Scaling} We compare fixed-penalty KL-BAPG with a postprocessed variant that applies alternating row and column scalings to its final coupling, without further GW optimization. Scaling stops when the sum of the marginal $L_1$ errors is at most $10^{-12}$ or after 1,000 cycles. We apply no subsequent exact marginal correction. We report accuracy, GW objective, and marginal infeasibility for the returned coupling, and evaluate projected stationarity on its feasible shadow using the same diagnostic procedure as for other infeasible outputs. This comparison assesses the effect of post hoc scaling on feasibility and projected stationarity.
\paragraph{Top-1 Accuracy} Let $W\in\mathbb{R}+^{n\times m}$ denote the estimated coupling and let $\tau(i)$ be the ground-truth target node corresponding to source node $i$. The predicted correspondence is the target node receiving the largest coupling weight, $\widehat{\tau}(i)=\arg\max_{1\leq j\leq m} W_{ij}$, and top-1 accuracy is
\[\operatorname{Accuracy}
= \frac{1}{n} \sum_{i=1}^{n} \mathbbm{1}\left\{\widehat{\tau}(i)=\tau(i)\right\}.\]

\subsection{Real graph-alignment details}\label{app:real_ga}
In PROTEINS and ENZYMES, nodes represent secondary protein-structure elements such as helices, sheets, and turns, while edges encode sequential or spatial proximity. ENZYMES contains proteins from six enzyme classes, and PROTEINS distinguishes enzymes from non-enzymes; these graph-level labels are not used.
We retain graphs with at most 250 nodes, yielding 598 ENZYMES graphs and 1106 PROTEINS graphs. For each source graph $G_X=(V_X,E_X)$ and corruption level $q=0.1$, following \citet{li2023convergent}, we add $\lfloor q|E_X|\rfloor$ randomly sampled absent edges and $\lfloor q|V_X|\rfloor$ noisy nodes and randomly connect them to the augmented graph. This produces the target graph $G_Y=(V_Y,E_Y)$. We set $D_X=A_X$ and $D_Y=A_Y$, where $A_X$ and $A_Y$ are the corresponding unweighted symmetric adjacency matrices, and use uniform node measures.

Given a feasible coupling $\widehat W$, Hungarian decoding produces an injective assignment $\widehat\tau$. We measure Hungarian one-to-one matching accuracy as
\[
\mathrm{Accuracy}
=
\frac{1}{|V_X|}
\sum_{i\in V_X}
\mathbf 1\{\widehat\tau(i)=i\}.
\]

\subsection{Synthetic Domain Adaptation Details} 
\label{app:syn_da}
\paragraph{Graph Construction}
For source sample $i$ and target sample $j$, let $x_i^s\in\mathbb{R}$ denote the common source feature and $x_j^t=(x_{j,c}^t,x_{j,a}^t)\in\mathbb{R}^2$ denote the target observation, where $x_{j,c}^t$ is the common feature and $x_{j,a}^t$ is the target-only feature.  Source and target class proportions are both $1/2$, and we use uniform marginals $\mu_i=1/n_s$ and $\nu_j=1/n_t$.
Target labels are hidden from all methods and accessed only after training to compute accuracy. Following \citet{aritake2026domain}, the cross-domain feature cost is \[M_{ij} = \frac{|x_i^s-x_{j,c}^t|^2} {\max_{i',j'}|x_{i'}^s-x_{j',c}^t|^2}.\]
Let $\widetilde D_X$ be the max-normalized squared-Euclidean source distance.
We form the label-aware source matrix
\[
(D_X)_{ii'} = 0.9(\widetilde D_X)_{ii'} + 0.1\mathbf{1}\{y_i^s\ne y_{i'}^s\}.
\]
The target matrix $D_Y$ is the all-pairs shortest-path distance on the symmetrized 4-NN graph of the full target features with squared-Euclidean edge weights and clipping distances at 10. This differs from \citet{aritake2026domain}, where they assign disconnected pairs a penalty equal to ten times their largest finite shortest-path distance, which assigns a penalty of 10 times the largest finite shortest-path distance to disconnected pairs.
All BAPG variants minimize $F_{\mathrm{FGW}}(X)$ with $\lambda=10^{-2}$. Here we use the squared-loss FGW for the objective function while \citet{aritake2026domain} uses an absolute distance discrepancy. 

Given the returned coupling $X$, we form a feasible coupling $\widehat X$ by alternating the row scaling
$X\leftarrow\operatorname{Diag}(\mu/(X\mathbf 1))X$ and column scaling
$X\leftarrow X\operatorname{Diag}(\nu/(X^\top\mathbf 1))$.  We stop when $\operatorname{feas}(X)\leq10^{-12}$ or after 500 scaling cycles, and use the resulting $\widehat X$ to transport source labels:
  \[\widehat y_j^t=\arg\max_c
  \frac{1}{\nu_j}\sum_i\widehat X_{ij}\mathbf 1\{y_i^s=c\}.\]
The accuracy is then defined as \[ \operatorname{Accuracy}=\frac{1}{n_t}\sum_j
  \mathbf 1\{\widehat y_j^t=y_j^t\}.\]

\paragraph{Baseline Information}
 Baselines are POT implementations of exact OT, GW, FGW, LA--FGW, GB--FGW, LAGB--FGW, and naive OT clustering \citep{aritake2026domain}. Naive OT clustering averages the transported class scores within each target spectral cluster and assigns the cluster-wise maximum to the cluster. The BAPG variants start from $X^0=\mu\nu^\top$ and run for 4,000 iterations. The POT GW- and FGW-based baselines use a maximum of 200 outer iterations with a tolerance of $10^{-7}$.
\begin{figure}[htbp]
    \centering
    \includegraphics[width=\linewidth]{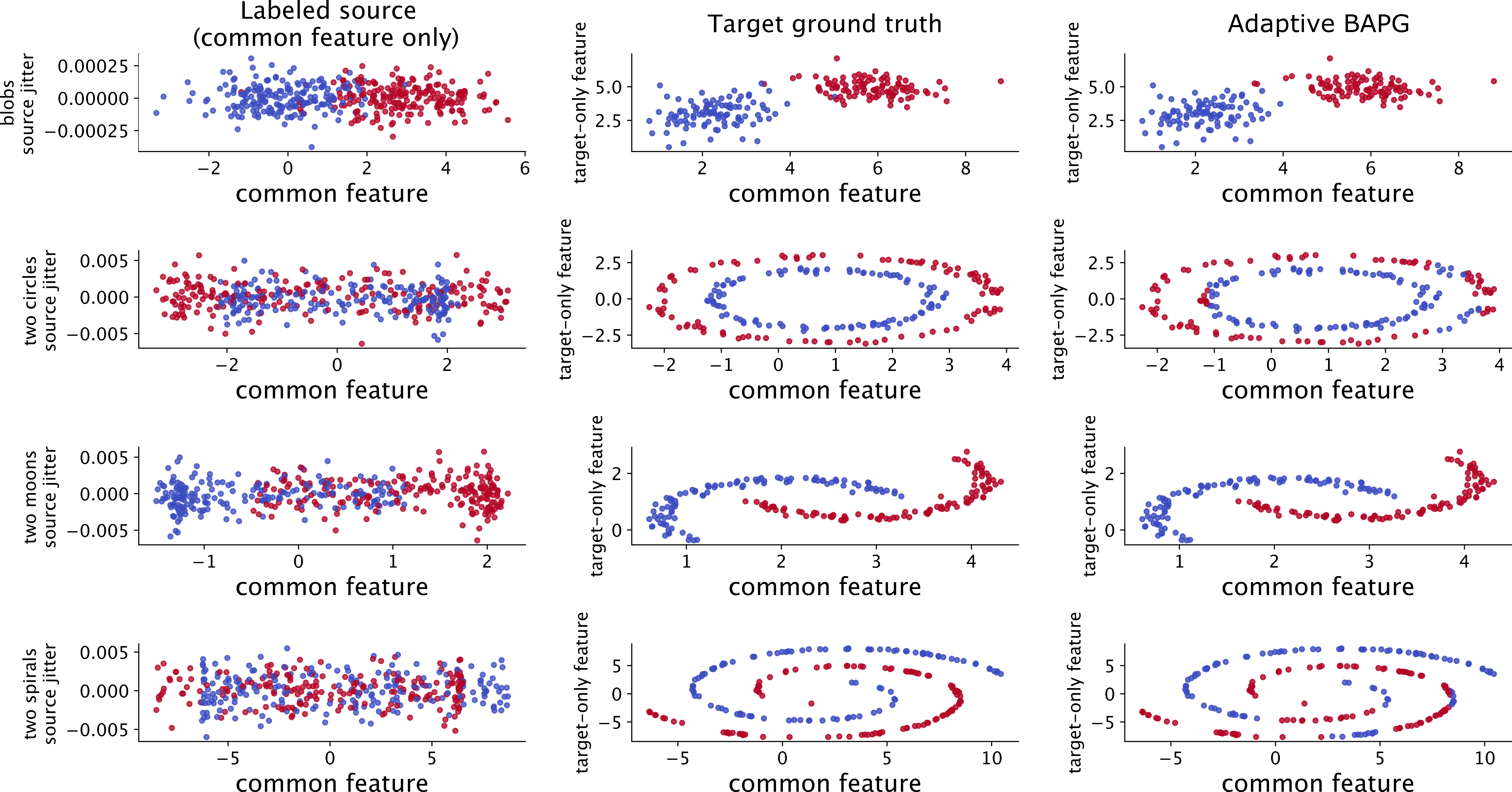}
    \caption{Visualization of heterogeneous domain adaptation for seed 0. Left: labeled source samples plotted by their single observed common feature; the vertical displacement of source points is jitter for visualization only and is not used as method input. Middle: target samples in the common and target-only features, colored by ground-truth labels hidden during training. Right: the same target samples colored by A-KL-BAPG predictions. Quantitative comparisons with other baselines are reported in Table \ref{tab:da-results}(a).}
    \label{fig:da_syn_vis}
\end{figure}
\paragraph{Results Visualization} Figure \ref{fig:da_syn_vis} provides a qualitative visualization of the four synthetic domain-adaptation settings: blobs, two circles, two moons, and two spirals. The first column shows the labeled source samples, where only the common feature is observed, while the second column shows the target samples colored by their ground-truth labels. The target samples contain an additional target-specific feature that reveals their two-dimensional geometry. The third column shows the labels predicted by A-KL-BAPG. Despite the shift between the source and target distributions, A-KL-BAPG separates the two classes almost perfectly for the blobs and two moons. For the two-circles setting, most points are classified correctly, with errors concentrated in portions of the rings that are difficult to distinguish using the common feature alone. The two-spirals setting is particularly challenging because the two classes become intertwined and overlap substantially when projected onto the common feature. Nevertheless, the predicted labels largely follow the ground truth. These examples demonstrate A-KL-BAPG's ability to produce coherent target predictions across a range of heterogeneous domain-adaptation settings.

\subsection{Real Domain Adaptation Details}\label{app:real_da}
\paragraph{Graph Construction} Each recording is divided into 5-second windows and represented by the mean and standard deviation of each axis, giving six back-sensor features shared by both domains and six thigh-sensor features available only in the target domain.
For each subject, we use a stratified 80/20 split where the labeled training windows that form the source domain use only the six shared features, and the test
Windows in the unlabeled target domain use all 12 features.
Following \citet{aritake2026domain}, we introduce a controlled shift in the target data. For each split, we fit PCA to the unshifted target samples and perturb the first three principal directions sequentially. Along each direction, we sort the sample projections and perform 50 sweeps, updating samples in random order and drawing each updated projection uniformly between its current neighbors. Boundary positions are padded by the original projection range. We use the split seed for these draws and leave the source samples unchanged.
After shifting the target observations, we construct \(D_X\), \(D_Y\), and \(M\) as in Appendix \ref{app:syn_da}. In particular, \(D_Y\) is the all-pairs shortest-path distance on a symmetrized 4-NN graph built from all twelve target features, with squared-Euclidean edge weights and distances capped at \(10\); \(D_X\) uses the six shared source features and source labels, while \(M\) uses only the six features shared across domains. $\lambda$ is also set to $0.01$, following \citet{aritake2026domain}.

\begin{table}[htbp]
\centering
\small
\caption{Complete HAR70+ domain adaptation results.
Accuracy values are retained from the previous experiments. Runtimes are updated using 54 jobs and reported as mean $\pm$ standard deviation. All values are rounded to three decimal places. A dash indicates an unavailable updated runtime.}
\label{tab:har70-all-methods}
\begin{tabular}{lcc}
\toprule
Method & Accuracy & Runtime (s) \\
\midrule
KL--BAPG & $0.880\pm0.074$ & $11.739\pm7.283$ \\
KL--BAPG + Sinkhorn & $0.880\pm0.074$ & $11.967\pm 7.531$ \\
$1/k$ & $0.662\pm0.068$ & $11.635\pm7.216$ \\
Pure power & $0.741\pm0.068$ & $11.627\pm7.215$ \\
A-KL-BAPG & $0.880\pm0.074$ & $67.665\pm58.035$ \\
\midrule
POT OT & $0.827\pm0.079$ & $0.008\pm0.006$ \\
POT GW & $0.803\pm0.211$ & $0.102\pm0.127$ \\
POT FGW & $0.827\pm0.079$ & $0.025\pm0.017$ \\
POT LA-FGW & $0.829\pm0.078$ & $0.028\pm0.020$ \\
POT GB-FGW & $0.854\pm0.079$ & $0.047\pm0.036$ \\
LAGB-FGW & $0.879\pm0.077$ & $0.056\pm0.048$ \\
Naive OT cluster & $0.767\pm0.085$ & $0.011\pm0.008$ \\
JDOT & $0.805\pm0.095$ & $0.112\pm0.079$ \\
DAEVS & $0.895\pm0.091$ & $0.255\pm0.181$ \\
\bottomrule
\end{tabular}
\end{table}

\end{document}